\documentclass[11pt,oneside,a4paper]{article}
\usepackage[top=1 in, bottom=1 in, left=0.85 in, right=0.85 in]{geometry}
\usepackage{hyperref}
\usepackage{float}
\usepackage{amsmath,amssymb,graphicx,url}
\usepackage{ntheorem}
\usepackage{thmtools,thm-restate,wrapfig,enumitem,mathabx}
\usepackage{algorithm,algorithmicx,algpseudocode}
\newtheorem{assumption}{Assumption}
\newtheorem{theorem}{Theorem}
\newtheorem{remark}{Remark}
\newtheorem{lemma}{Lemma}

\newtheorem*{proof}{Proof}
\newtheorem{definition}{Definition}
\usepackage{color}

\usepackage[utf8]{inputenc}
\usepackage[T1]{fontenc}
\usepackage{booktabs}
\usepackage{amsfonts}
\usepackage{nicefrac}
\usepackage{microtype}
\usepackage{xcolor}
\usepackage{subfigure}
\usepackage{multirow}
\usepackage{bbm}
\usepackage{mathtools}

\newcommand{\blue}[1]{\textcolor{blue}{#1}}

\newcommand{\expt}{\mathbb{E}}
\newcommand{\prob}{\mathbb{P}}

\title{Exploration, Exploitation, and Engagement\\in Multi-Armed Bandits with Abandonment}

\author{%
  Zixian Yang\\
  University of Michigan\\
  Ann Arbor, MI, USA\\
  \texttt{zixian@umich.edu}\\
  \and
  Xin Liu\\
  ShanghaiTech University\\
  Shanghai, China\\
  \texttt{liuxin7@shanghaitech.edu.cn} \\
  \and
  Lei Ying\\
  University of Michigan\\
  Ann Arbor, MI, USA\\
  \texttt{leiying@umich.edu}\\
}
\date{}

\begin{document}

\maketitle

\begin{abstract}
Multi-armed bandit (MAB) is a classic model for understanding the exploration-exploitation trade-off. The traditional MAB model for recommendation systems assumes the user stays in the system for the entire learning horizon. In new online education platforms such as ALEKS or new video recommendation systems such as TikTok and YouTube Shorts, the amount of time a user spends on the app depends on how engaging the recommended contents are. Users may temporarily leave the system if the recommended items cannot engage the users. To understand the exploration, exploitation, and engagement in these systems, we propose a new model, called MAB-A where ``A'' stands for abandonment and the abandonment probability depends on  the current recommended item and the user's past experience (called state). We propose two algorithms, ULCB and KL-ULCB, both of which do more exploration (being optimistic) when the user likes the previous recommended item and less exploration (being pessimistic) when the user does not like the previous item. We prove that both ULCB and KL-ULCB achieve logarithmic regret, $O(\log K)$, where $K$ is the number of visits (or episodes). Furthermore, the regret bound under KL-ULCB is asymptotically sharp. We also extend the proposed algorithms to the general-state setting. Simulation results confirm our theoretical analysis and show that the proposed algorithms have significantly lower regrets than the traditional UCB and KL-UCB, and Q-learning-based algorithms.
\end{abstract}

\section{Introduction}
\label{sec:intro}

Recommendation algorithms have become increasingly important in many online platforms such as online education, TikTok, YouTube Shorts, advertising platforms, etc. The system interacts with the users to learn their preferences and recommends personalized contents (learning subjects, videos, songs, products etc.) to each user. Multi-armed bandit (MAB)~\cite{lattimore2020bandit} is a classic problem which can model these recommendation systems. Each arm in MAB corresponds to a specific type of item in the recommendation system. The recommendation of an item of the $i$th type is regarded as a pull of arm $a_i$.  Taking recommending short videos as an example, each arm $a_i$ represents a class of similar videos (e.g. videos from the same dancer). For simplicity, we assume the reward is $1$ if the user likes the recommended item and is $0$ otherwise. In a traditional MAB problem, the learner can continue to play the arms with the goal of maximizing the average reward, which either assumes a single user stays in the system for a long period of time or assumes the learner is recommending a single item to each user with a large number of users. While this traditional MAB formulation models recommendation systems such as online advertising well, there are new recommendation systems that are significantly different from these traditional models. In these new recommendation systems, such as TikTok or ALEKS, the learner continuously recommends videos/contents to a user, and the user, other than like or dislike the item, may abandon the system if the recommended items cannot engage the user, and come back later. For example, a user watches TikTok or YouTube Shorts for some period of time, where the duration depends on  how interesting/engaging the videos, then leaves the systems, and comes back later, as shown in Figure~\ref{fig:application}.

\begin{figure}[htbp]
    \centering
    \subfigure[A new recommendation system.]
    {
        \includegraphics[width=0.45\textwidth]{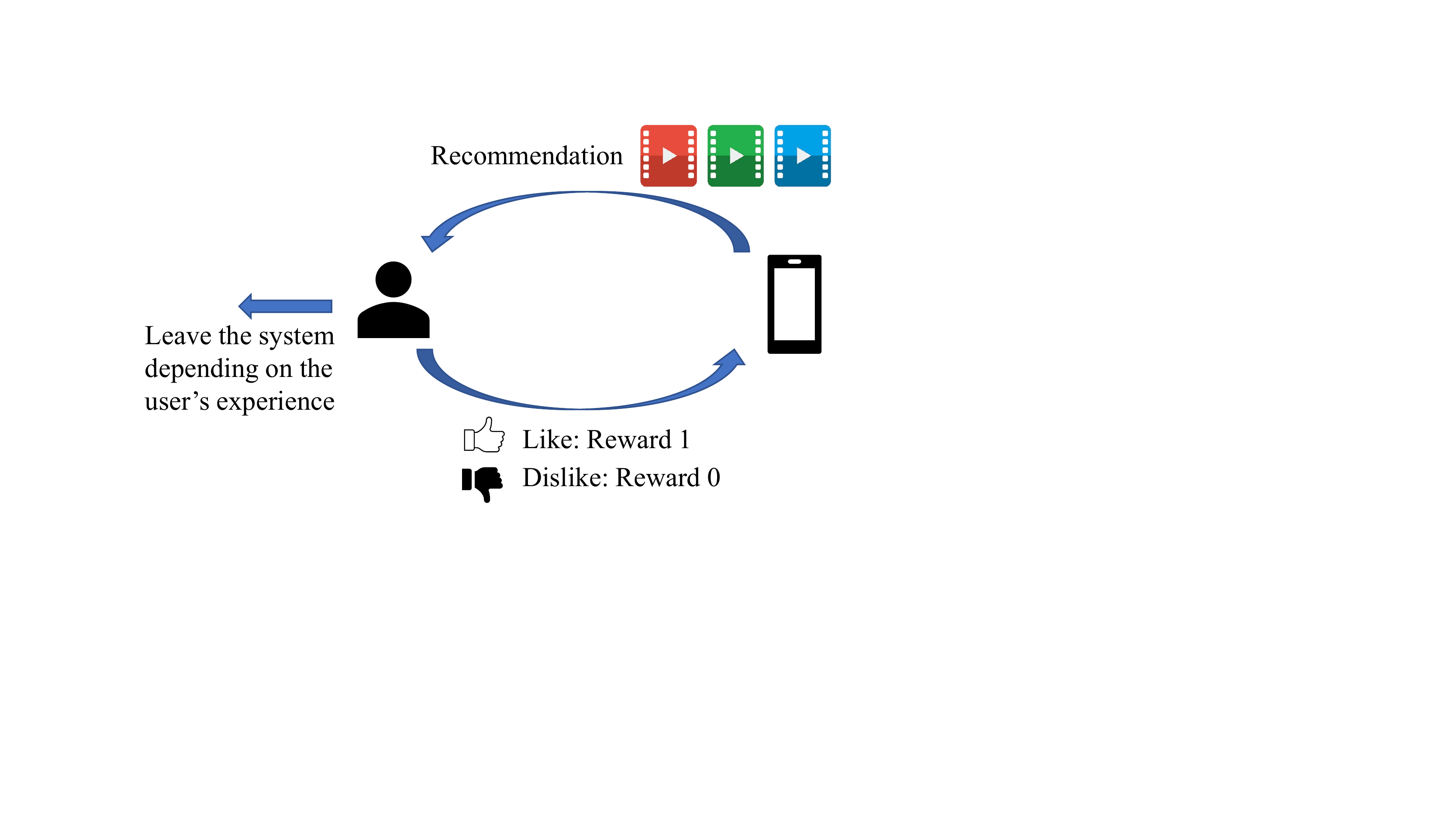}
        \label{fig:application}
    }
    \hfill
    \subfigure[Total regret over $K$ episodes.]
    {
        \includegraphics[width=0.45\textwidth]{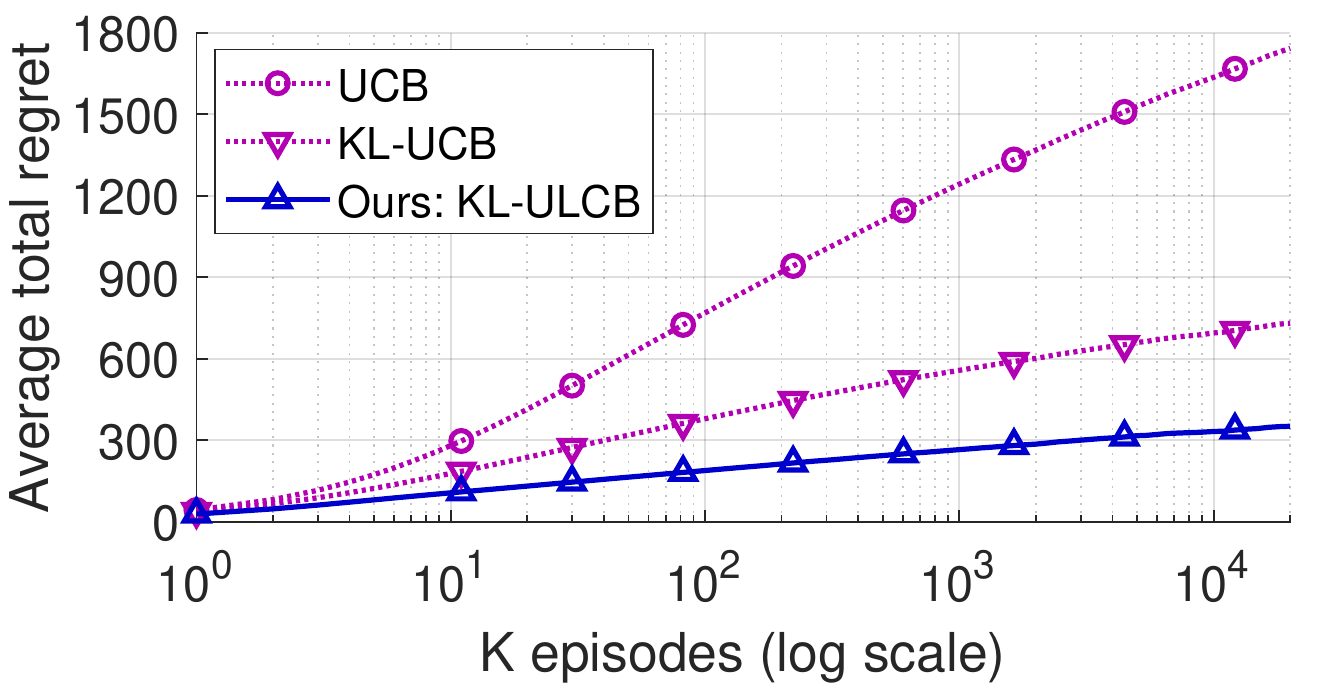}
        \label{fig:results-intro}
    }
    \caption{A new recommendation system and comparison among algorithms.}
\end{figure}

This makes the problem different from traditional MAB because the objective now is to maximize the total reward per episode (visit) instead of the average reward per pull. Therefore, in addition to finding the most rewarding arm, the learner also needs to continue to engage the user to maximize the number of plays of each episode. Because of the abandonment, the exploration needs to be carefully designed so that the learner should explore (recommend new types of items) when the user is less likely to abandon the system. In other words, we need to consider an exploration-exploitation-engagement tradeoff in this problem.  As we can see from Figure~\ref{fig:results-intro}, a well designed algorithm can significantly outperform the traditional MAB algorithms such as upper confidence bound (UCB)~\cite{auer2002finite} and Kullback-Leibler UCB (KL-UCB)~\cite{garivier2011kl}. 

We study this new MAB problem with abandonment, denoted by MAB-A. Consider a recommendation system where the system recommends one item at a time to the user. For example, for mobile phone users, since the screen is small, the system such as a mobile app can only recommend one item at a time instead of recommending multiple items simultaneously. The user may or may not like the item, and they may abandon the system with a certain probability (called abandonment probability in this paper) based on current and previous experience. The objective is to maximize the total reward per episode, where an episode ends when the user abandons (leaves) the system temporarily. 

For traditional MAB problems, classic index policy algorithms such as UCB and KL-UCB work well and KL-UCB  achieves the instance-dependent lower bound for traditional MAB with Bernoulli rewards~\cite{garivier2011kl,lai1985asymptotically}. However, these traditional algorithms are not suitable for the new MAB-A problem, since they do not consider the abandonment and do not utilize the state information (the user's experience). Hence, they may not be optimal. We propose to use a state-dependent exploration-exploitation mechanism, which does more exploration (being optimistic) when the user is less likely to abandon the system and less exploration (being pessimistic) when the user is more likely to abandon the system. Our algorithms are based on both an upper confidence bound and a lower confidence bound. Our main contributions are as follows:
\begin{itemize}[leftmargin=*]
    \item {\bf Baseline:} First, we characterize the baseline by showing that a genie-aided optimal policy for MAB-A problem is always pulling the optimal arm ({\bf Lemma \ref{lemma:1}}).
    \item {\bf Sharp bounds:} We propose two algorithms based on upper and lower confidence bounds, named as Upper and Lower Confidence Bounds (ULCB) and Kullback-Leibler Upper and Lower Confidence Bounds (KL-ULCB) algorithms. We prove that both algorithms achieve $O(\log K)$ regret bound ({\bf Theorem \ref{theorem:1} and Theorem \ref{theorem:2}}). We further establish a lower bound for MAB-A problem and show that KL-ULCB attains the bound ({\bf Theorem \ref{theorem:3}}), so the regret under KL-ULCB is asymptotically sharp.
    \item {\bf Extension to a general-state model:} We extend the proposed algorithms to MAB-A problems with a general continuous state space. In particular, we propose four algorithms, DISC-ULCB, DISC-KL-ULCB, CONT-ULCB, and CONT-KL-ULCB. We establish $O(\log K)$ upper bounds for DISC-ULCB and DISC-KL-ULCB ({\bf Theorem \ref{theorem:cont-state}}) and show that the bound for DISC-KL-ULCB is nearly sharp for large $n$, where $n$ is the number of discretized bins in the algorithm.
    \item {\bf Numerical evaluation:} Simulation results in Section~\ref{sec:simu} and Appendix~\ref{app:add-simulations} confirm our theoretical results and show that our algorithms have significantly lower regrets than the traditional UCB and KL-UCB algorithms, and have order-wise lower regrets than generic reinforcement learning (RL) algorithms like Q-learning.
    
    \item {\bf Technical novelty:} In MAB-A, the episode length follows different distributions under different policies, so the regret analysis based on step-by-step coupling, like in the traditional MAB analysis, does not work. We overcome this difficulty by exploiting the performance difference lemma~\cite{Kakade02approximatelyoptimal,yang2021q} to couple the rewards by the sum of gap functions along {\em the sample path that follows our algorithm}. On the other hand, MAB-A can be regarded as a special class of Markov decision process (MDP) problems with a terminal state, and is mostly related to the stochastic shortest path (SSP) problem \cite{bertsekas1991analysis}. MAB-A, however, is a stochastic {\em longest} path problem so the existing algorithms and analysis for SSP do not apply. We utilize the special properties of bandits and abandonment to establish a sharp $O(\log K)$ regret bound, while most regret bounds on SSP are $O(\sqrt{K}).$ 
\end{itemize}

\subsection{Related work} We are not aware of any work in the literature with the same setting as the MAB-A problem, but there are a few related works. The work in~\cite{schmit2018learning} studies a setting with abandonment, where the mean reward is an increasing function of the action.
The abandonment occurs when the action is larger than the user's threshold, and thus the algorithms should consider the trade-off between getting high rewards and avoiding losing users. In contrast, in our MAB-A setting, the reward is unknown and a higher reward makes the user less likely to abandon the system. The concept of abandonment also appears in the sequential choice bandit problem in~\cite{cao2019dynamic} and the departing bandit problem in~\cite{ben2022modeling}. However, the abandonment probabilities in their models do not depend on the past experience of the user.
Another work in~\cite{wu2018adaptive} studies the exploration-exploitation tradeoff in an opportunistic bandit setting, where the regret of pulling a suboptimal arm varies under different environmental conditions. Our proposed algorithms and proof ideas are partly inspired by the exploration-exploitation intuition in their work~\cite{wu2018adaptive}. However, the key difference between the opportunistic bandit setting and our MAB-A setting is that there is no abandonment in the opportunistic bandit setting. Also, the state in the MAB-A setting depends on previous rewards while the load (state) in the opportunistic bandit setting does not. The above differences lead to different algorithms and theoretical results.

Note that the MAB-A problem can be modeled as a special case of stochastic shortest path (SSP) problems~\cite{bertsekas1991analysis} with non-positive costs. RL algorithms like Q-learning~\cite{watkins1989learning} and Q-learning with UCB~\cite{yang2021q} might be used for the MAB-A problem but these general algorithms do not utilize the special structures in MAB-A and therefore are too complex and not regret optimal, which is verified in the simulation results in Section~\ref{sec:simu}. Other algorithms~\cite{cohen2021minimax,chen2021implicit,vial2021regret, tarbouriech2021stochastic} are designed for SSP problems with non-negative costs, which are fundamentally different from MAB-A since MAB-A tries to maximize the episode length but SSP problems with non-negative costs may not. Hence, these algorithms cannot be applied to the MAB-A problem.
Besides, only $O(\sqrt{K})$ instead of $O(\log K)$ regret bounds are proved in these papers.

\section{Model and preliminaries}
\label{sec:model}

The MAB-A problem is defined as follows. Let $M$ ($M\ge 2$) be the number of arms and denote the set of arms by $\{a_1, a_2, \cdots, a_M\}$. Assume that the rewards of pulling the arms are i.i.d. Bernoulli random variables with unknown mean $\mu(a_i)$, $i\in\{1,2,...,M\}$. Consider $K$ episodes in total, where each episode represents a single visit of an user and an episode ends when the user abandons the system temporarily. The process of the $k^{\mathrm{th}}$ episode goes as follows. At step $h=1$, an initial state $S_{k,1}\in\{0,1\}$ is sampled from an arbitrary distribution. At step $h=2,3,...$, the state is defined by $S_{k,h}\coloneqq  R_{k,h-1}$, where $R_{k,h-1}$ is the reward obtained at the previous step $h-1$. Then an arm $A_{k,h}\in\{a_1,...,a_M\}$ is pulled and a Bernoulli random reward $R_{k,h}\in\{0,1\}$ is obtained with mean $\mu(A_{k,h})$. Given $(S_{k,h}, R_{k,h})$, abandonment occurs with probability $q(S_{k,h}, R_{k,h})$. If the abandonment occurs, the terminal state $g$ is reached, i.e., $S_{k,h+1} = g$, which terminates the current episode $k$. Otherwise, the process goes to the next step.

Therefore, the process of one episode is an MDP
with state space $\mathcal{S} = \{0,1,g\}$, action space $\mathcal{A} = \{a_1,...,a_M\}$, and Bernoulli random rewards. The state can be interpreted as the experience of the user. At the first step ($h=1$) in each episode, the initial state $S_{k,1}$ can be interpreted as the user's first impression. 
Note that given $A_{k,h}$, the reward $R_{k,h}$ is independent of $S_{k,h}$. We write $R_{k,h}(A_{k,h})$ when necessary in order to explicitly show the dependency between $R_{k,h}$ and $A_{k,h}$.
We will consider a general-state model in Section~\ref{sec:extension}, where the state depends on the rewards received in all previous steps of the current episode. We remark that we first consider the current model, for which we can establish sharp bounds. However, the intuition and exploration strategy obtained from the current model will be applied to the general-state model and nearly sharp bounds can be established based on discretization. 
We make the following assumption on the problem. 
\begin{assumption}\label{assum:1}
Assume $q(i,j)\ge q(i',j')$ if $i+j< i'+j',$  $q(0,0)>0$ and $\mu(a_M) \le \mu(a_{M-1}) \le ... \le \mu(a_2) \le \mu(a_1) < 1$.
\end{assumption}
The assumption on $q(\cdot,\cdot)$ implies the abandonment probability becomes larger when the user's experience becomes worse. The assumptions $q(0,0)>0$ and $\mu(a_i) < 1, \forall i$ ensures that all policies are proper. That is, all policies lead to the terminal state $g$ with probability one, regardless of the initial state~\cite{bertsekas1991analysis}. Without loss of generality, we let $\mu(a_M) \le \mu(a_{M-1}) \le ... \le \mu(a_2) \le \mu(a_1)$.

To understand the exploration-exploitation-engagement trade-off of MAB-A defined above, we next define the baseline, i.e. the reward under a genie-aided (model-based) optimal policy, which knows the model perfectly. The result is summarized in Lemma~\ref{lemma:1},  which states that the optimal policy is to always pull arm $a_1.$ The proof can be found in Appendix~\ref{app:lemma1}.
\begin{lemma}\label{lemma:1}
    Let Assumption~\ref{assum:1} hold. The genie-aided optimal policy $\pi^*$ is always pulling arm $a_1$.
\end{lemma}

Let $\pi: \mathcal{S} \times \Phi \rightarrow \mathcal{A}$ denote a deterministic policy such that $A_{k,h}=\pi(S_{k,h}, \phi_{k,h})$, where $\phi_{k,h}\in \Phi$ is the historical samples till step $h$ of episode $k$ (not including the current step), i.e.,
\begin{align}
    \phi_{k,h} = (S_{1,1},A_{1,1},R_{1,1},..., S_{k,1},A_{k,1},R_{k,1},...,S_{k,h-1},A_{k,h-1},R_{k,h-1}).
\end{align}
Let $\Pi\coloneqq \{\pi: \mathcal{S} \times \Phi \rightarrow \mathcal{A}\}$ denote the set of all such policies. Let $I_k(\pi,s,\varphi)$ denote the number of steps taken to reach the terminal state $g$ given the current state $s$ and the historical samples $\varphi\in\Phi$ under the policy $\pi\in\Pi$ in episode $k$.
Mathematically, let $D$ be a random set such that $D(\pi,s,\varphi)\coloneqq \{i:S_{k,h+i}=g, S_{k,h}=s, \phi_{k,h}=\varphi, A_{k,h+j}=\pi(S_{k,h+j},\phi_{k,h+j}), \forall j=0,1,...,i-1\}$. Then
\begin{equation}\label{equ:def-I}
    I_k(\pi,s,\varphi)\coloneqq \begin{cases}
    \min D(\pi,s,\varphi), & \mbox{if } D(\pi,s,\varphi)\neq \emptyset;\\
    \infty, & \mbox{if }D(\pi,s,\varphi)=\emptyset.
    \end{cases}
\end{equation}
Similarly, let $I_k(\pi^*, s)$ denote the number of steps taken to reach the terminal state $g$ given the current state $s$ under $\pi^*$ in episode $k$, i.e.,
\begin{equation}\label{equ:def-I-star}
    I_k(\pi^*, s)\coloneqq \begin{cases}
    \min D^*(s) , & \mbox{if } D^*(s)\neq \emptyset;\\
    \infty, & \mbox{if }D^*(s)=\emptyset,
    \end{cases}
\end{equation}
where $D^*(s)\coloneqq \{i:S_{k,h+i}=g, S_{k,h}=s, A_{k,h+j}=a_1, \forall j=0,1,...,i-1\}$.

The objective is to find a policy $\pi\in\Pi$ to minimize the expected regret (over $K$ episodes) defined by
\begin{equation}\label{equ:def-regret-b}
\begin{aligned}
    \expt[\mathrm{Reg}_{\pi}(K)] =
    \expt\left[\sum_{k=1}^{K}\sum_{h=1}^{I_k(\pi^*, S_{k,1})}R_{k,h}(a_1)\right] - \expt\left[\sum_{k=1}^{K}\sum_{h=1}^{I_k(\pi, S_{k,1}, \phi_{k,1})}R_{k,h}(\pi(S_{k,h},\phi_{k,h}))\right].
\end{aligned}
\end{equation}

\section{Main results and the proof roadmap}
\label{sec:main}
In this section, we first present two algorithms for the MAB-A problem. One is ULCB, which uses an upper or lower confidence bound depending on the state for exploration and exploitation. The other one is KL-ULCB algorithm, which uses KL divergence for the confidence bounds.

\subsection{Algorithms}

\begin{algorithm}[ht]
\caption{ULCB Algorithm}\label{alg:1}
\begin{algorithmic}[1]
    \State {\textbf{Input:}} $N_1(a)\gets 0$, $\bar{\mu}_1(a) \gets 0$ for all $a\in\mathcal{A}$, $t\gets 1$, $c_0$, $c_1$, $c$.
    \State {$h\gets 1$, $S_{k,1} \gets$ initial state of episode $k$, $S_{k,1}\in\{0,1\}$}
    \For{episode $k=1,...,K$}
        \While{$S_{k,h} \neq g$}
            \If{there exists Arm $a'$ such that $N_{t}(a')=0$}
            \State {play Arm $A_{k,h} =a'$ and observe $R_{k,h}$}~~~\blue{// play each arm once}
            \Else
            \If{$S_{k,h}=0$}
            \State {Let $\tilde{\mu}_t^0(a)=\bar{\mu}_{t}(a)+c_0\sqrt{\frac{\log t + c \log(\log t)}{2 N_t(a)}}$ for all $a\in\mathcal{A}$}\label{alg:1:mu0}~~~\blue{// indices for state 0}
            \State {Take the action $A_{k,h}\in \operatorname{argmax}_{a}\tilde{\mu}^0_t(a)$}
            \Else
            \State {Let $\tilde{\mu}_t^1(a)=\bar{\mu}_{t}(a)+c_1\sqrt{\frac{\log t + c \log(\log t)}{2 N_t(a)}}$ for all $a\in\mathcal{A}$}\label{line:alg1-state1-mu-tilde}\label{alg:1:mu1}~~~\blue{// indices for state 1}
            \State {Take the action $A_{k,h}\in \operatorname{argmax}_{a}\tilde{\mu}^1_t(a)$} \label{line:alg1-state1-action}
            \EndIf
            \State {Observe $R_{k,h}$}
            \EndIf
            \If{abandonment occurs}
                \State {$S_{k,h+1} = g$}
            \Else
                \State {$S_{k,h+1} = R_{k,h}$}
            \EndIf
            \State{Define $(S_{t}, A_{t}, S'_{t}, R_{t}) \coloneqq   (S_{k,h}, A_{k,h}, S_{k,h+1}, R_{k,h})$}
            \State {Update: $N_{t+1}(A_t) = N_{t}(A_t) + 1$ and $N_{t+1}(a) = N_{t}(a)~\forall a\neq A_t$}~~~\blue{// update $N_{t+1}(a)$}
            \State {Update: $\bar{\mu}_{t+1}(A_t) = \frac{\bar{\mu}_{t}(A_t)N_{t}(A_t)+R_t}{N_{t+1}(A_t)}$ and $\bar{\mu}_{t+1}(a) = \bar{\mu}_{t}(a)~\forall a\neq A_t$}~~~\blue{// update $\bar{\mu}_{t+1}(a)$}
            \State {$t\gets t+1$, $h\gets h+1$}
        \EndWhile
    \EndFor
 \end{algorithmic}
\end{algorithm}

We propose the ULCB algorithm, which is an index policy like UCB algorithm but the difference is that ULCB uses state-dependent indices, as shown in Algorithm~\ref{alg:1}. Firstly, the ULCB algorithm plays each arm once by Round-Robin. After that, at step $h$ of episode $k$, if the state $S_{k,h}=0$, we let
\begin{align}\label{equ:alg:1:mu0}
    \tilde{\mu}_t^0(a)=\bar{\mu}_{t}(a)+c_0\sqrt{\frac{\log t + c \log(\log t)}{2 N_t(a)}}
\end{align}
for all $a\in\mathcal{A}$, where $c$ and $c_0$ are constants, $t$ is the time step counting from the first episode, $N_{t}(a)\coloneqq \sum_{s=1}^{t-1} \mathbbm{1}\{A_s=a\}$ denotes the number of times arm $a$ has been pulled before time step $t$, and $\bar{\mu}_{t}(a)\coloneqq \left(\sum_{s=1}^{t-1} \mathbbm{1}\{A_s=a\} R_s \right)/N_{t}(a)$ denotes the average of rewards of pulling arm $a$ before time step $t$. Note that we also denote the state, the action, and the reward at time step $t$ by $S_t$, $A_t$, and $R_t$, respectively. Then we take an action $A_{k,h}\in \operatorname{argmax}_{a}\tilde{\mu}^0_t(a)$. If the state $S_{k,h}=1$, we let
\begin{align}\label{equ:alg:1:mu1}
    \tilde{\mu}_t^1(a)=\bar{\mu}_{t}(a)+c_1\sqrt{\frac{\log t + c \log(\log t)}{2 N_t(a)}}
\end{align}
for all $a\in\mathcal{A}$, where $c_1$ is a constant. Then we take an action $A_{k,h}\in \operatorname{argmax}_{a}\tilde{\mu}^1_t(a)$. The algorithm then updates $S_{t+1}$, $N_{t+1}(a)$, and $\bar{\mu}_{t+1}(a)$. The process goes to the next step or the next episode depending on whether the abandonment occurs or not. 

In fact, the indices $\tilde{\mu}_t^0(a)$ and $\tilde{\mu}_t^1(a)$ are the (upper or lower) confidence bounds of the expected reward of arm $a$. Note that $c_0$ and $c_1$ in~\eqref{equ:alg:1:mu0} and~\eqref{equ:alg:1:mu1} are not necessarily positive. Our theoretical results actually indicate that we should use positive coefficient in state $1$ and negative coefficient in state $0$, which means optimism (upper confidence bound) in state $1$ and pessimism (lower confidence bound) in state $0$. This leads to more exploration in state $1$ than in state $0$.

We also propose the KL-ULCB algorithm, which replaces the indices $\tilde{\mu}_t^0(a)$ and $\tilde{\mu}_t^1(a)$ in~\eqref{equ:alg:1:mu0} and~\eqref{equ:alg:1:mu1} with
\begin{align}
    \tilde{\mu}_t^0(a)=&\min\left\{p:\mathrm{kl}(\bar{\mu}_t(a),p)N_t(a)\le c_0 \log t + c \log(\log t) \right\}\\
    \tilde{\mu}_t^1(a)=&\max\left\{p:\mathrm{kl}(\bar{\mu}_t(a),p)N_t(a)\le c_1 \log t + c \log(\log t) \right\}
\end{align}
where $\mathrm{kl}(p_1, p_2)$ is the KL divergence between two Bernoulli random variables with parameters $p_1$ and $p_2$.
KL-ULCB is similar to ULCB except that KL-ULCB uses KL divergence for the confidence bound instead of directly adding the bonus term. This idea is borrowed from KL-UCB~\cite{garivier2011kl}.

\subsection{Main results}
We next present three theorems, including the regret upper bound on ULCB ({\bf Theorem \ref{theorem:1}}), the regret upper bound on KL-ULCB ({\bf Theorem \ref{theorem:2}}), and a regret lower bound ({\bf Theorem \ref{theorem:3}}) that matches the upper bound of KL-ULCB.
We also present the proof idea and roadmap in the next sbusection and present the results for the general-state setting in Section \ref{sec:extension}.

Let $V^*(s)$ and $Q^*(s,a)$ denote the optimal value function and optimal Q-function defined by
\begin{align}
    V^{*}(s) \coloneqq  & \expt \left[\sum_{h=1}^{I_k(\pi^*,s)} R_{k,h}(a_1)\left\vert\right.S_{k,1}=s\right]\label{equ:v*-def-v2},\\
    Q^{*}(s, a) \coloneqq  &
    \mu(a) + \expt \left[\sum_{h=2}^{I_k(\pi^*,S_{k,2})+1} R_{k,h}(a_1)\left\vert\right. S_{k,1}=s, A_{k,1}=a\right],\label{equ:q*-def-v2}
\end{align}
for $s\neq g$, and $V^{*}(g)\coloneqq Q^{*}(g, a)\coloneqq 0$, for any $a\in\mathcal{A}$.

\begin{theorem}[Upper bound for ULCB]\label{theorem:1}
Let Assumption~\ref{assum:1} hold. Assume that $\mu(a_1)>\mu(a_2)$, $\mu(a_M) > 0$, $q(0,1)<1$, $q(1,1)<1$, and for any $a\neq a_1$,
\begin{align}\label{assum:gap}
    V^*(0) - Q^*(0, a) \ge V^*(1) - Q^*(1, a).
\end{align}
Then under ULCB with $c_0=-1$, $c_1=1$ and $c=4$,  we have
\begin{align}\label{equ:theorem:1}
    \limsup_{K\rightarrow \infty} \frac{\expt[\mathrm{Reg}_{\pi}(K)]}{\log K} \le \sum_{i\neq 1} \frac{V^*(1) - Q^*(1,a_i)}{2(\mu(a_1)-\mu(a_i))^2}.
\end{align}
\end{theorem}

The assumption $\mu(a_1)>\mu(a_2)$ in Theorem~\ref{theorem:1} ensures that the suboptimal gap $\mu(a_1)-\mu(a_i)$ is positive, which is a general assumption when deriving instance-dependent regret bounds for MAB problems. $\mu(a_M)>0$, $q(0,1)<1$, and $q(1,1)<1$ ensure that there is always a positive proportion of time during which the process is in state $1$.

The condition~\eqref{assum:gap} means that a suboptimal pull induces more regret (loss) in state $0$ than in state $1$. This motivates us to do more exploration in state $1$ and to be conservative in state $0$, i.e., $c_1>c_0$. With $c_1=1$ and $c_0=-1$, Theorem~\ref{theorem:1} provides an asymptotic logarithmic upper bound with an instance-dependent constant $\sum_{i\neq 1} \frac{V^*(1) - Q^*(1,a_i)}{2(\mu(a_1)-\mu(a_i))^2}$. We will show later in Section~\ref{sec:decom-regret} that the constant term in the upper bound for the traditional UCB algorithm ($c_0=c_1=1$) could be $\sum_{i\neq 1} \frac{V^*(0) - Q^*(0,a_i)}{2(\mu(a_1)-\mu(a_i))^2}$, which is greater than the one obtained by the ULCB algorithm. In fact, $V^*(1) - Q^*(1,a_i)$ could be significantly smaller than $V^*(0) - Q^*(0,a_i)$ in some cases. Consider a simple example $q(0,0)=1$ and $q(0,1)=q(1,0)=q(1,1)=0$. Then we have $V^*(1) - Q^*(1,a_i) = \frac{\mu(a_1)-\mu(a_i)}{1-\mu(a_1)}$ and $V^*(0) - Q^*(0,a_i) = \frac{\mu(a_1)-\mu(a_i)}{(1-\mu(a_1))^2}$, and thus the upper bound obtained by ULCB algorithm will be significantly better especially when $\mu(a_1)$ is close to $1$.

Condition \eqref{assum:gap} is in terms of the value function and Q-function, which may not be straightforward to verify. Lemma~\ref{lemma:sufficient-condition} provides a sufficient condition for~\eqref{assum:gap} in terms of $q(i,j)$. See Appendix~\ref{app:lemma:sufficient-condition} for the proof.

\begin{lemma}\label{lemma:sufficient-condition}
    Let Assumption~\ref{assum:1} hold. Assume that $\mu(a_1)>0$, $q(1,1)<1$, $q(1,0) \neq q(0,0)$, and
    \begin{align}\label{equ:theorem:1-condition}
    \frac{q(0,1)-q(1,1)}{q(0,0)-q(1,0)}\le \min\left\{\frac{1-q(0,1)}{1-q(1,1)}, \frac{1-q(0,0)}{1-q(1,0)}\right\}
    \end{align}
    Then for any $a\neq a_1$, we have
    \begin{align}
        V^*(0) - Q^*(0, a) \ge V^*(1) - Q^*(1, a).
    \end{align}
\end{lemma}

One example of the condition~\eqref{equ:theorem:1-condition} is that the difference between $q(0,1)$ and $q(1,1)$ is small but the difference between $q(0,0)$ and $q(1,0)$ is relatively large. This means that when the user obtains a reward $1$, they are more likely to forget their previous reward compared with obtaining a reward $0$ when they make the abandonment decision. Hence, intuitively, we should be more conservative in state $0$ so that we are more likely to obtain a reward $1$ in order to encourage the user to stay in the system. That is the intuition of using optimistic estimate in state $1$ and pessimistic estimate in state $0$ in the ULCB algorithm.

When the condition~\eqref{assum:gap} does not hold and a suboptimal pull induces more regret in state 1 than in state 0, the learner needs to be optimistic in state 0 and pessimistic in state 1. Modified ULCB and KL-ULCB can guarantee the same regret bounds (see Theorems~\ref{theorem:5},~\ref{theorem:6}, and~\ref{theorem:7} in Appendix~\ref{app:add-cases}).

\begin{theorem}[Upper bound for KL-ULCB]\label{theorem:2}
Let all the assumptions in Theorem~\ref{theorem:1} hold. Then using the KL-ULCB algorithm with $c_0=c_1=1$, and $c=4$,  we have
\begin{align}
    \limsup_{K\rightarrow \infty} \frac{\expt[\mathrm{Reg}_{\pi}(K)]}{\log K} \le \sum_{i\neq 1} \frac{V^*(1) - Q^*(1,a_i)}{\mathrm{kl}(\mu(a_i),\mu(a_1))}.
\end{align}
\end{theorem}
Theorem~\ref{theorem:2} gives a regret upper bound for KL-ULCB. Compared with the result in Theorem~\ref{theorem:1}, the bound in Theorem~\ref{theorem:2} is better since $\mathrm{kl}(\mu(a_i),\mu(a_1))\ge 2(\mu(a_1)-\mu(a_i))^2$ by Pinsker's inequality. This bound is also better than the one obtained by KL-UCB, $\sum_{i\neq 1} \frac{V^*(0) - Q^*(0,a_i)}{\mathrm{kl}(\mu(a_i),\mu(a_1))}$, which will be illustrated later in Section~\ref{sec:decom-regret}.

In order to analyze instance-dependent lower bound for MAB-A, similar to~\cite{lai1985asymptotically,lattimore2020bandit}, we define the set of all consistent policies by $\Pi_{\mathrm{cons}}$:
\begin{definition}\label{def:consistent-policy}
A policy $\pi\in\Pi$ is consistent, i.e., $\pi\in\Pi_{\mathrm{cons}}$, if for any $\mu(a_1),...,\mu(a_M)$, $q(0,0)$, $q(0,1)$, $q(1,0)$, $q(1,1)$, and any $\alpha>0$, $\lim_{K\rightarrow \infty} \expt[\mathrm{Reg}_{\pi}(K)]/K^{\alpha} = 0.$
\end{definition}
Theorem~\ref{theorem:3} gives an asymptotic lower bound among policies in $\Pi_{\mathrm{cons}}$ for the MAB-A problem.
\begin{theorem}[Lower bound]\label{theorem:3}
Let all the assumptions in Theorem~\ref{theorem:1} hold. For any $\pi\in\Pi_{\mathrm{cons}}$ and any $\mu(a_1)$,...,$\mu(a_M)$, $q(0,0)$, $q(0,1)$, $q(1,0)$, $q(1,1)$ satisfying the assumptions, we have
\begin{align}
    \liminf_{K\rightarrow \infty} \frac{\expt[\mathrm{Reg}_{\pi}(K)]}{\log K} \ge \sum_{i\neq 1} \frac{V^*(1) - Q^*(1,a_i)}{\mathrm{kl}(\mu(a_i),\mu(a_1))}.
\end{align}

\end{theorem}
By Theorem~\ref{theorem:2} and Theorem~\ref{theorem:3}, the regret upper bound obtained by the KL-ULCB algorithm attains the lower bound asymptotically.

\subsection{Proof roadmap}
\label{sec:proof-roadmap}

In this section, we present the proof roadmap of Theorems~\ref{theorem:1},~\ref{theorem:2}, and~\ref{theorem:3}. Before that, we first illustrate the main intuition behind the proofs.

The first challenge in the proofs is how to couple the rewards from two policies, i.e., the two terms in the regret definition~\eqref{equ:def-regret-b}. Note that we cannot subtract the rewards step by step like the traditional proof for MAB, since the episode lengths $I_k$ for the two different policies are two different random variables. We utilize the performance difference lemma~\cite{Kakade02approximatelyoptimal,yang2021q} in the RL literature to couple the rewards by the sum of gap functions ($V^*(s)-Q^*(s,a)$) along the sample path that follows the policy of the algorithm. The gap function represents the regret induced by pulling a suboptimal arm $a$ in state $s$ assuming all the future actions follow the optimal policy. Then we can deal with the regret step by step and further decompose the regret of each step into state $1$ and state $0$.

In the proof for the upper bound for ULCB (Theorem~\ref{theorem:1}), we managed to bound the regret induced in state $0$ by a constant. First, since there is always a positive proportion of time during which the process is in state $1$ and optimistic estimates are used in state $1$, we can show that the number of optimal pulls in state $1$ scales linearly with $t$ with high probability (Lemma~\ref{lemma:sd-cb-1}). The intuition is that optimistic estimates encourage exploration, which induces only logarithmic number of suboptimal pulls. Hence, the confidence intervals around the optimal arm $a_1$ should be tight enough. In state $0$, it can be proved that {\em under the pessimistic estimation,} the estimate of a suboptimal arm $\tilde{\mu}_t^0(a)$ is less than the true mean $\mu(a)$ with high probability. Since $\mu(a)\le \mu(a_1)$, $\tilde{\mu}_t^0(a)$ will also be less than $\tilde{\mu}_t^0(a_1)$ with high probability by the tightness of $\tilde{\mu}_t^0(a_1)$ (a result from the analysis in state 1). Hence, $a_1$ will be pulled with high probability in state $0$, which implies a constant upper bound of suboptimal pulls in state $0$ (Lemma~\ref{lemma:sd-cb-2}). Then it remains to bound the number of suboptimal pulls in state $1$, which can be bounded using the techniques in~\cite{garivier2011kl} and an upper bound of episode length (Lemma~\ref{lemma:sd-cb-3}). The proof for the upper bound for KL-ULCB (Theorem~\ref{theorem:2}) is similar to that for ULCB.  The difference is that we use concentration inequalities for the KL divergence.  

For the lower bound (Theorem~\ref{theorem:3}), we first bound the regret below by the number of suboptimal pulls multiplied by the gap function in state $1$ since the gap function in state $1$ is smaller than that in state $0$. Then it remains to bound the number of suboptimal pulls below. The idea is similar to the proof of the instance-dependent lower bound for the MAB problem~\cite{lai1985asymptotically}, but the difference is that the horizon (total number of pulls) in MAB-A is not a constant but a random variable. Our idea is to use a simple lower bound, i.e., the horizon is greater than the number of episodes.

Our upper bound matches the lower bound thanks to the regret decomposition, the constant regret in state $0$, and the sharpness of the bound in state $1$ using KL divergence. Next, we will present the regret decomposition and the proof for Theorem~\ref{theorem:1} in more details. See Appendix~\ref{app:theorem2} and Appendix~\ref{app:theorem3} for the proofs of Theorem~\ref{theorem:2} and Theorem~\ref{theorem:3}, respectively.

\subsubsection{Regret decomposition}
\label{sec:decom-regret}

Our results and the proofs start from the regret decomposition. 

We first define value function and Q function to facilitate the analysis. Define
\begin{align}
    V^{\pi}(s,\varphi) \coloneqq  & \expt\left[\sum_{h=1}^{I_k(\pi,s,\varphi)}R_{k,h}(\pi(S_{k,h},\phi_{k,h}))\vert S_{k,1}=s,\phi_{k,1}=\varphi\right],\label{equ:vpi-qpi-def}\\
    Q^{\pi}(s,\varphi, a) \coloneqq  &
    \mu(a) +  \expt\biggl[\sum_{h=2}^{I_k(\pi,S_{k,2},\phi_{k,2})+1} R_{k,h}(\pi(S_{k,h},\phi_{k,h}))|S_{k,1}=s, \phi_{k,1}=\varphi, A_{k,1}=a\biggr],\nonumber
\end{align}
for any $s\neq g$, and $V^{\pi}(g,\varphi)\coloneqq Q^{\pi}(g,\varphi, a)\coloneqq 0$, for any $\varphi\in\Phi$, $a\in\mathcal{A}$, and $\pi\in\Pi$.

By the definitions of $V^*$ in~(\ref{equ:v*-def-v2}) and $V^{\pi}$ in~(\ref{equ:vpi-qpi-def}), the expected regret defined in~(\ref{equ:def-regret-b}) is
\begin{equation}\label{equ:regret-2}
    \begin{aligned}
    \expt[\mathrm{Reg}_{\pi}(K)] = \sum_{k=1}^K \biggl[ \expt\left[V^*(S_{k,1})\right] - \expt\left[V^{\pi}(S_{k,1},\phi_{k,1})\right]\biggr].
    \end{aligned}
\end{equation}
By the performance difference formula in the RL literature, we can decompose the regret into the summation of the gaps between value function and Q function in different states shown as~\eqref{equ:regret-decomp}.
\begin{align}\label{equ:regret-decomp}
     \expt[\mathrm{Reg}_{\pi}(K)]  =  \expt\Biggl[ \sum_{t=1}^{T(K,\pi)} \sum_{i=2}^{M} & \mathbbm{1}\{S_t = 0, A_t = a_i\} \left[ V^*(0) - Q^*(0, a_i)\right]\Biggr. \nonumber\\
    & \Biggl. + \mathbbm{1}\{S_t = 1, A_t = a_i\} \left[ V^*(1) - Q^*(1, a_i)\right] \Biggr],
\end{align}
where $S_t$ and $A_t$ are the state and action at time step $t$ following the policy $\pi$, and $T(K,\pi)\coloneqq \sum_{k=1}^{K} I_k(\pi, S_{k,1}, \phi_{k,1})$ is the number of pulls over $K$ episodes under policy $\pi$. The proof details for~\eqref{equ:regret-decomp} can be found in Appendix~\ref{app:regret-decomp}.
From the regret decomposition~\eqref{equ:regret-decomp}, the terms $V^*(0) - Q^*(0, a_i)$ and $V^*(1) - Q^*(1, a_i)$ can be interpreted as the regrets induced by pulling a suboptimal arm $a_i$ in state $0$ and $1$, respectively. Thus, the key idea of obtaining a lower regret is first determining which of the two terms is smaller and then putting more exploration in that state.

Suppose that we use traditional UCB or KL-UCB algorithm. Both use the same exploration strategy for state $1$ and state $0$. Thus, from~\eqref{equ:regret-decomp}, we obtain an upper bound:
\begin{align}
    \expt[\mathrm{Reg}_{\pi}(K)]\le \expt\left[ \sum_{t=1}^{T(K,\pi)} \sum_{i=2}^{M} \mathbbm{1}\{A_t = a_i\}\right] \left[V^*(0) - Q^*(0,a_i)\right],
\end{align}
where the constant term $V^*(0) - Q^*(0,a_i)$ is worse than $V^*(1) - Q^*(1,a_i)$ in Theorem~\ref{theorem:1} and Theorem~\ref{theorem:2}.
Therefore, we use state-dependent exploration-exploitation to utilize the difference between states so that we can obtain a better upper bound.

\subsubsection{Proof of Theorem~\ref{theorem:1}}
The proof mainly includes three steps, which correspond to the following three lemmas as we explained in Section \ref{sec:proof-roadmap}.

\begin{lemma}\label{lemma:sd-cb-1}
    Let all the assumptions in Theorem~\ref{theorem:1} hold. Consider the ULCB algorithm with $c_0=-1$, $c_1=1$, and $c=4$. Let $p_{\min}\coloneqq \mu(a_M) \min\left\{1-q(0,1), 1-q(1,1)\right\}$. Let $\eta\in(0, p_{\min})$ and $\gamma \in (0, \mu(a_1)-\mu(a_2))$ be two constants. 
    There exists a constant $T_1$ such that for any $t\ge T_1$,
    \begin{align}
        & \prob\left(N_{t}(a_1)\le \frac{(p_{\min}-\eta)(t-1)}{2}\right)\nonumber\\
        \le & \frac{M-1}{2\gamma^2\exp\left(2\gamma^2 c_2(t-1) - 4\gamma^2 \right)} + \frac{c_3}{c_2(t-1)\left[\log \left(c_2(t-1)\right)\right]^2} + \exp\left(-\frac{\eta^2(t-1)}{2}\right),
    \end{align}
    where $c_2$, $c_3$, and $T_1$ are constants depending only on $p_{\min}$, $\eta$, $M$, $\gamma$, $\mu(a_1)$, and $\mu(a_2)$.
\end{lemma}
Lemma~\ref{lemma:sd-cb-1} shows that when $t$ is large enough, the number of optimal pulls scales linearly with $t$ with high probability. The key idea of the proof of Lemma~\ref{lemma:sd-cb-1} is that when $N_t(a_1)$ is small, $a_1$ will be pulled with high probability. Lemma~\ref{lemma:sd-cb-1} looks similar to Lemma~2 in~\cite{wu2018adaptive} but we have a tighter bound which requires extra efforts in the proof. Note that the proof of this result only utilizes the optimistic exploration in state $1$. See Appendix~\ref{app:lemma:sd-cb-1} for a complete proof. Lemma~\ref{lemma:sd-cb-1} is important since we can show that the confidence bound around the optimal arm is tight enough for large $t$ based on this result. Then we can bound the regret induced by pulling suboptimal arms in state $0$ by a constant using pessimistic estimate (lower confidence bound), which is shown by Lemma~\ref{lemma:sd-cb-2}:
\begin{lemma}\label{lemma:sd-cb-2}
    Let all the assumptions in Theorem~\ref{theorem:1} hold. Consider the ULCB algorithm with $c_0=-1$, $c_1=1$, and $c=4$. The regret induced in state $0$ is bounded by
    \begin{align}\label{equ:regret-induced-in-state-0-ub}
        \expt\left[ \sum_{t=1}^{T(K,\pi)} \sum_{i=2}^{M} \mathbbm{1}\{S_t = 0, A_t = a_i\} \left[ V^*(0) - Q^*(0, a_i)\right]\right]  \le  c_4 \sum_{i=2}^{M} \left[ V^*(0) - Q^*(0, a_i)\right],
    \end{align}
    where $c_4$ is a constant which depends only on $M$, $\mu(a_1)$, $\mu(a_2)$, $p_{\mathrm{min}}$, $\eta$, and $\gamma$.
\end{lemma}
The proof idea of Lemma~\ref{lemma:sd-cb-2} is as follows. We first show that  $\mu(a_i)\ge \tilde{\mu}_t^0(a_i)$ with high probability. And based on Lemma~\ref{lemma:sd-cb-1} we can show that $\tilde{\mu}_t^0(a_1)$ and $\mu(a_1)$ are close enough for large $t$. Hence, for large $t$, we have $\tilde{\mu}_t^0(a_1) \approx \mu(a_1) \ge \mu(a_i)\ge \tilde{\mu}_t^0(a_i)$ with high probability, which implies that $a_1$ will be pulled in state $0$ with high probability. Hence $\prob(S_t=0, A_t=a_i)$ is small enough so that we can bound~\eqref{equ:regret-induced-in-state-0-ub}. See Appendix~\ref{app:lemma:sd-cb-2} for a complete proof.

We then bound the regret induced in state $1$ by a term of order $\log K$ shown by Lemma~\ref{lemma:sd-cb-3}:
\begin{lemma}\label{lemma:sd-cb-3}
    Let all the assumptions in Theorem~\ref{theorem:1} hold. Consider the ULCB algorithm with $c_0=-1$, $c_1=1$, and $c=4$. For any $\epsilon>0$, the regret induced in state $1$ is bounded by
    \begin{align}\label{equ:regret-induced-in-state-1-ub}
        & \expt\left[ \sum_{t=1}^{T(K,\pi)} \sum_{i=2}^{M} \mathbbm{1}\{S_t = 1, A_t = a_i\} \left[ V^*(1) - Q^*(1, a_i)\right]\right] \nonumber\\
        \le & \sum_{i\neq 1} \frac{(1+\epsilon)\left[V^*(1) - Q^*(1,a_i)\right]}{2(\mu(a_1)-\mu(a_i))^2} \log K + o(\log K)
    \end{align}
\end{lemma}
In the proof of Lemma~\ref{lemma:sd-cb-3}, we first use techniques from~\cite{garivier2011kl} since ULCB uses upper confidence bound in state $1$, and then we bound the term $\expt[\log (T(K,\pi))]$ to get a bound of order $\log K$. See Appendix~\ref{app:lemma:sd-cb-3} for a complete proof.

Combining~\eqref{equ:regret-decomp} with Lemma~\ref{lemma:sd-cb-2} and Lemma~\ref{lemma:sd-cb-3}, Theorem~\ref{theorem:1} is proved.

\section{Extension to a general-state setting}
\label{sec:extension}

In this section, we extend our results to the general-state setting. We first present an MAB-A model with continuous state space, and then verify that the optimal policy is still always pulling the optimal arm. Next, we propose two types of algorithms and analyze the regret. We obtain the same form of regret lower bound for the general-state setting. We also obtain regret upper bounds for DISC-ULCB and DISC-KL-ULCB algorithms.

\subsection{Model and the optimal policy}
\label{sec:cont-ext-model}
Define the continuous state space by $\mathcal{S}=[0,1]\cup\{g\}$.
Define the state $S_{k,h}$ at step $h$ of episode $k$ as an exponential moving average of previous rewards in episode $k$, i.e.,
\begin{align}
    S_{k,h} \coloneqq  (1-\theta) S_{k,h-1} + \theta R_{k,h-1} 
\end{align}
for any $k\ge 1$ and $h\ge 2$, where $\theta\in(0,1)$ is a constant forgetting factor, which means how much the user forgets their previous experience. $S_{k,1}\in[0,1]$ is sampled from an arbitrary distribution. The abandonment probability at step $h$ of episode $k$ is a function of the next state $S_{k, h+1}$, denoted by $q(S_{k,h+1})$. 
\begin{assumption}\label{assum:2}
Assume that $0 < q(s_1)\le q(s_2)$ if $s_1\ge s_2$ for any $s_1,s_2\in[0,1]$, and $\mu(a_M) \le \mu(a_{M-1}) \le ... \le \mu(a_2) \le \mu(a_1)$.
\end{assumption}
The assumptions on $q(\cdot)$ is reasonable since the abandonment probability becomes larger when the user's experience becomes worse. The positivity assumption on $q$ ensures that all policies are proper. Without loss of generality, we let $\mu(a_M) \le \mu(a_{M-1}) \le ... \le \mu(a_2) \le \mu(a_1)$.

Define the genie-aided (model-based) optimal policy $\pi^*$ the same way as in the original setting. Then we have Lemma~\ref{lemma:cont-1}. The proof can be found in Appendix~\ref{app:proof-lemma-optimal-policy-general}.
\begin{lemma}
\label{lemma:cont-1}
Let Assumption~\ref{assum:2} hold. Then the genie-aided optimal policy $\pi^*$ is always pulling arm $a_1$.
\end{lemma}

\subsection{Algorithms and regret analysis}
\label{sec:cont-ext-alg}

We propose DISC-ULCB and DISC-KL-ULCB algorithms, which first discretize the state space $[0,1]$ into $n$ bins, $[0,\frac{1}{n})$, $[\frac{1}{n},\frac{2}{n})$,...,$[\frac{n-1}{n}, 1]$, and then use the ULCB or KL-ULCB algorithm, where we view any state in $[\frac{n-1}{n}, 1]$ as state $1$ and any state in the other bins as state $0$. 

We next analyze the regret of these two algorithms. Following the same way as the regret decomposition in~\eqref{equ:regret-decomp}, we have
\begin{align}
    \expt[\mathrm{Reg}_{\pi}(K)] = & \expt\left[\sum_{t=1}^{T(K,\pi)} \sum_{i=2}^{M} 
    \sum_{m=1}^{n-1} \mathbbm{1}\left\{S_t \in \left[\frac{m-1}{n}, \frac{m}{n}\right), A_t = a_i\right\}
    \left[ V^*(S_t) - Q^*(S_t, a_i)\right]\right. \nonumber\\
    &~~~~~~~~~~~~~~~~~~~~~~\left. + \mathbbm{1}\left\{S_t \in \left[\frac{n-1}{n}, 1\right], A_t = a_i\right\}
    \left[ V^*(S_t) - Q^*(S_t, a_i)\right] \right]
\end{align}
for any integer $n\ge 2$, where $V^*(S_t) - Q^*(S_t, a_i)$ can be interpreted as the regret induced by pulling arm $a_i$ in state $S_t$. We consider the case where $V^*(s_1) - Q^*(s_1, a) \le V^*(s_2) - Q^*(s_2, a)$ for any $a\in\mathcal{A}, s_1,s_2\in[0,1], s_1\ge s_2$. Some examples can be found in Section~\ref{app:cont-state-gap-function}. In this case, we can obtain an upper bound
\begin{align}\label{equ:cont-state-upper-bound}
    \expt[\mathrm{Reg}_{\pi}(K)] \le & \expt\left[\sum_{t=1}^{T(K,\pi)} \sum_{i=2}^{M} 
    \sum_{m=1}^{n-1} \mathbbm{1}\left\{S_t \in \left[\frac{m-1}{n}, \frac{m}{n}\right), A_t = a_i\right\}
    \left[ V^*\left(\frac{m-1}{n}\right) - Q^*\left(\frac{m-1}{n}, a_i\right)\right]\right. \nonumber\\
    &~~~~~~~~~~~~~~~~~~~~~~\Biggl. + \mathbbm{1}\left\{S_t \in \left[\frac{n-1}{n}, 1\right], A_t = a_i\right\}
    \left[ V^*\left(\frac{n-1}{n}\right) - Q^*\left(\frac{n-1}{n}, a_i\right)\right] \Biggr],
\end{align}
and a lower bound
\begin{align}\label{equ:cont-state-lower-bound}
    \expt[\mathrm{Reg}_{\pi}(K)] \ge & \expt\left[ \sum_{t=1}^{T(K,\pi)} \sum_{i=2}^{M} \mathbbm{1}\{A_t = a_i\} \right] \left[ V^*(1) - Q^*(1, a_i)\right].
\end{align}
From~\eqref{equ:cont-state-lower-bound} we can obtain the same regret lower bound as Theorem~\ref{theorem:3} by following the same proof. For the upper bounds for DISC-ULCB and DISC-KL-ULCB algorithms, we have the following theorem.
\begin{theorem}\label{theorem:cont-state}
Let Assumption~\ref{assum:2} hold. Let $n\ge 2$ denote the number of bins for DISC-ULCB or DISC-KL-ULCB algorithms. Assume that $\mu(a_1)>\mu(a_2)$, $\mu(a_M)>0$, $q(s)<1~\forall s\in[\frac{n-1}{n}, 1]$, and 
\begin{align}
    V^*(s_1) - Q^*(s_1, a) \le V^*(s_2) - Q^*(s_2, a)
\end{align}
for any $a\in\mathcal{A}, s_1,s_2\in[0,1], s_1\ge s_2$. Then using DISC-ULCB algorithm with $c_0=-1$, $c_1=1$, and $c=4$, we have
\begin{align}
    \limsup_{K\rightarrow \infty} \frac{\expt[\mathrm{Reg}_{\pi}(K)]}{\log K} \le \sum_{i\neq 1} \frac{V^*\left(\frac{n-1}{n}\right) - Q^*\left(\frac{n-1}{n},a_i\right)}{2(\mu(a_1)-\mu(a_i))^2}.
\end{align}
Using DISC-KL-ULCB algorithm with $c_0=c_1=1$, and $c=4$, we have
\begin{align}\label{equ:disc-kl-ulcb-upper-bound}
    \limsup_{K\rightarrow \infty} \frac{\expt[\mathrm{Reg}_{\pi}(K)]}{\log K} \le \sum_{i\neq 1} \frac{V^*\left(\frac{n-1}{n}\right) - Q^*\left(\frac{n-1}{n},a_i\right)}{\mathrm{kl}(\mu(a_i),\mu(a_1))}.
\end{align}
\end{theorem}
Theorem~\ref{theorem:cont-state} shows that we have $O(\log K)$ upper bounds for DISC-ULCB and DISC-KL-ULCB. The proof is by~\eqref{equ:cont-state-upper-bound} and the same method as the proofs of Theorem~\ref{theorem:1} and Theorem~\ref{theorem:2}, and therefore is omitted.
For DISC-KL-ULCB, the asymptotic upper bound is nearly tight for large $n$. However, if $n$ is large, there is a very small fraction of time when the state of the system is in $[\frac{n-1}{n}, 1]$. It results in a very slow exploration since most of the exploration happens in $[\frac{n-1}{n}, 1]$. Hence, the regret might be large initially despite the fact that the asymptotic regret upper bound is near optimal.
To overcome this, we propose a second type of algorithms, CONT-ULCB and CONT-KL-ULCB.
For CONT-ULCB, we use indices $\tilde{\mu}_t^s(a)$ as follows:
\begin{align}
    \tilde{\mu}_t^{s}(a)=\bar{\mu}_{t}(a)+(2s - 1)\sqrt{\frac{\log t + c \log(\log t)}{2 N_t(a)}}.
\end{align}
Similarly, CONT-KL-ULCB uses KL divergence in the indices.
$\tilde{\mu}_t^{s}(a)$ changes gradually from lower confidence bound to upper confidence bound when $s$ changes from $0$ to $1$, which means that the algorithm changes from exploitation to exploration continuously, which therefore leads to more exploration at the beginning compared to DISC-ULCB and DISC-KL-ULCB. 

More details, proofs, and simulation results about this extension can be found in Appendix~\ref{app:extension-cont}.

\section{Simulation results}
\label{sec:simu}

In this section, we present simulation results for the performance of the proposed algorithms. In the simulation, we assume $S_{k,1}=1$ for simplicity. This is to say that the user assumes a class of items are good if the user has not yet seen the items. Let $M=2$, $\mu(a_1)=0.9$, and $\mu(a_2)=0.8$. Note that for all the algorithms in the simulation, we do not include the $\log \log $ terms (i.e., $c=0$) in the indices which are also omitted in~\cite{garivier2011kl}.
We simulated $2\times 10^4$ episodes with $10^7$ independent runs.
We set $c_1=1$, $c_0=-1$ for ULCB and $c_1=c_0=1$ for KL-ULCB.
For Figure~\ref{fig:simu-0-1-comp}, the $95\%$ confidence bounds are at most $\pm 8.73$. For Figure~\ref{fig:simu-0-1-comp-2}, the $95\%$ confidence bounds are at most $\pm 0.63$.
In terms of average cumulative regret, Figure~\ref{fig:simu-0-1-comp} and~\ref{fig:simu-0-1-comp-2} show that ULCB outperforms traditional UCB and that KL-ULCB outperforms traditional KL-UCB. Our algorithms have order-wise lower regrets than Q-learning~\cite{watkins1989learning} with $\epsilon$-greedy and Q-learning with UCB~\cite{yang2021q}. Note that the asymptotic upper bound (UB) and lower bound (LB) in the figures only consider the $\log K$ term in the regret and ignore the other lower order terms, so only the slopes matter.
Figure~\ref{fig:simu-0-1-ub-lb} and~\ref{fig:simu-0-1-ub-lb-2} plot the average cumulative regret over $K$ episodes divided by $\log K$. It can be seen that the curves go towards the asymptotic regret upper bound (UB) and the asymptotic lower bound (LB). These results confirm our theoretical results. See Appendix~\ref{app:add-simulations} for simulation parameters and additional simulation results. Simulations for the general-state setting can be found in Appendix~\ref{app:cont-state-simu}.

\begin{figure}[ht]
    \centering
    \subfigure[Comparison]
    {
        \includegraphics[width=0.48\textwidth]{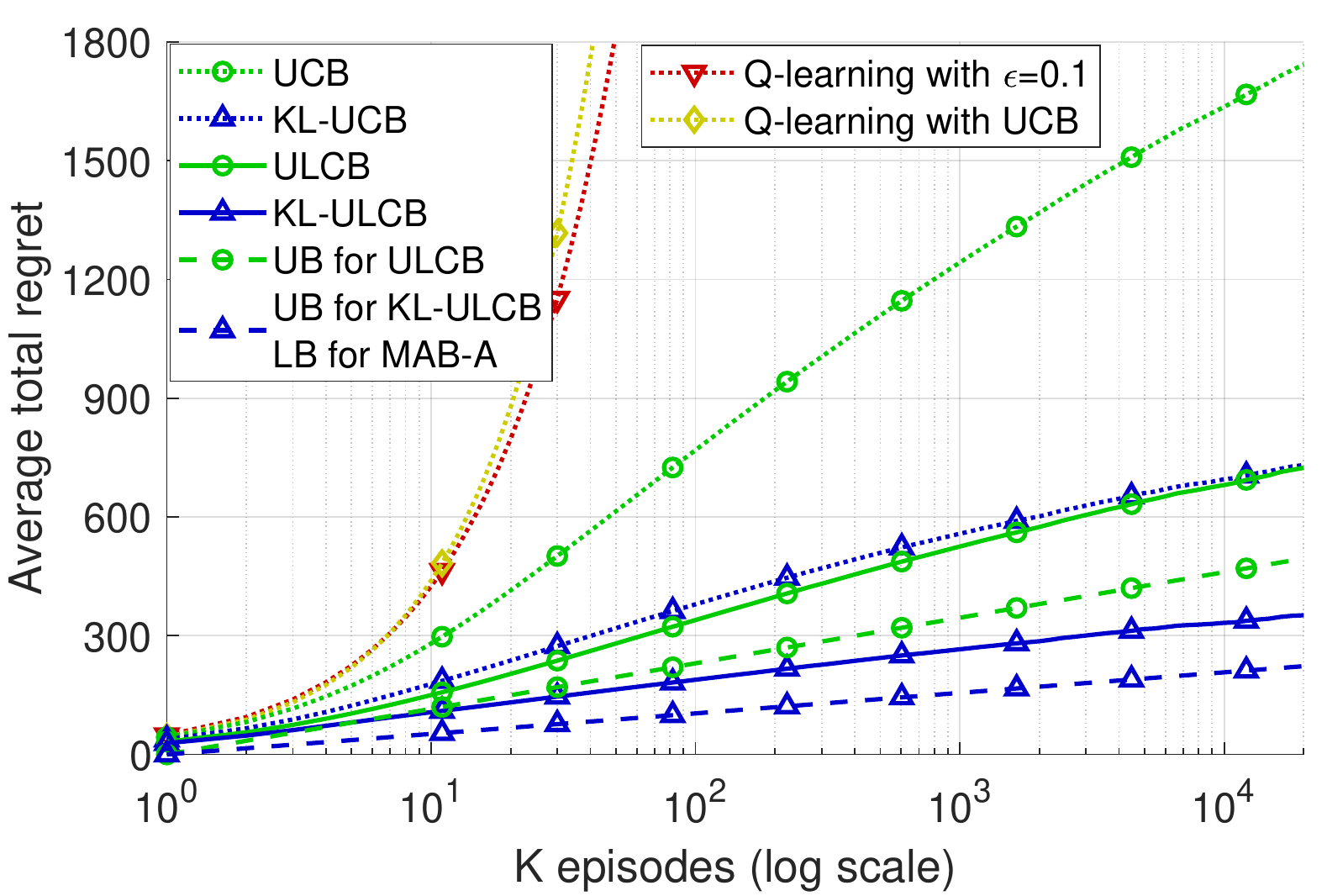}
        \label{fig:simu-0-1-comp}
    }
    \hfill
    \subfigure[UB and LB]
    {
        \includegraphics[width=0.48\textwidth]{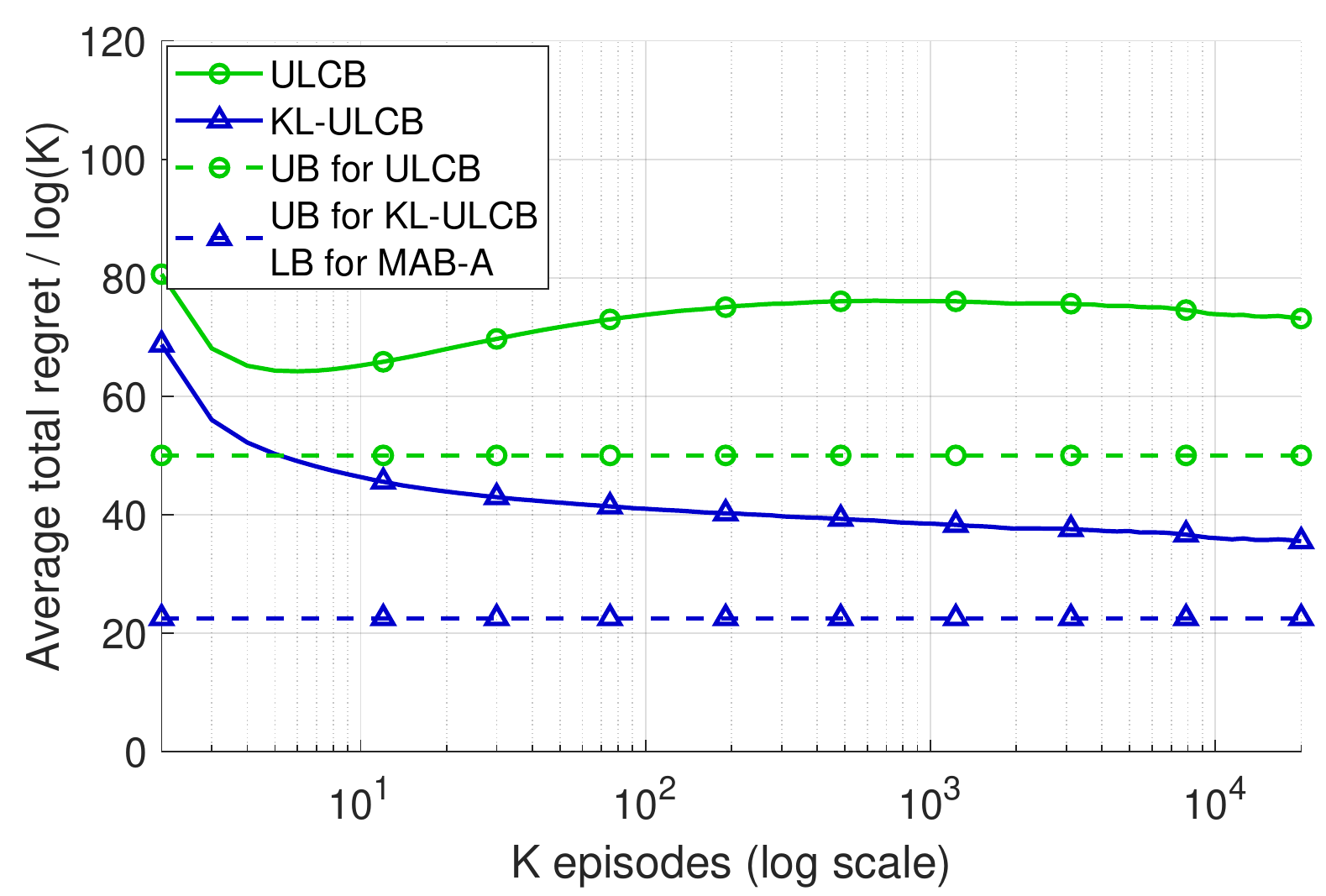}
        \label{fig:simu-0-1-ub-lb}
    }
    \hfill
    \subfigure[Comparison]
    {
        \includegraphics[width=0.48\textwidth]{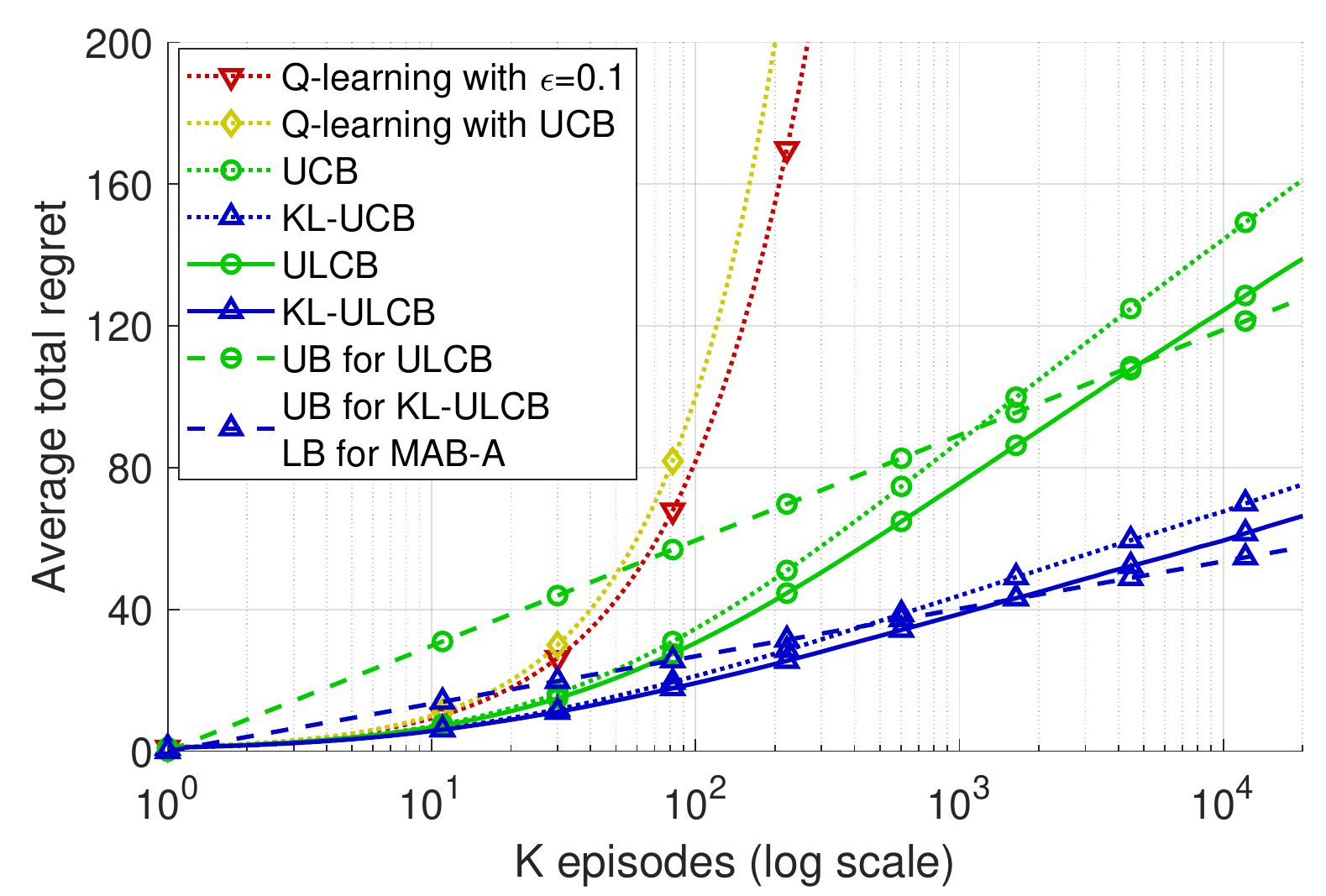}
        \label{fig:simu-0-1-comp-2}
    }
    \hfill
    \subfigure[UB and LB]
    {
        \includegraphics[width=0.48\textwidth]{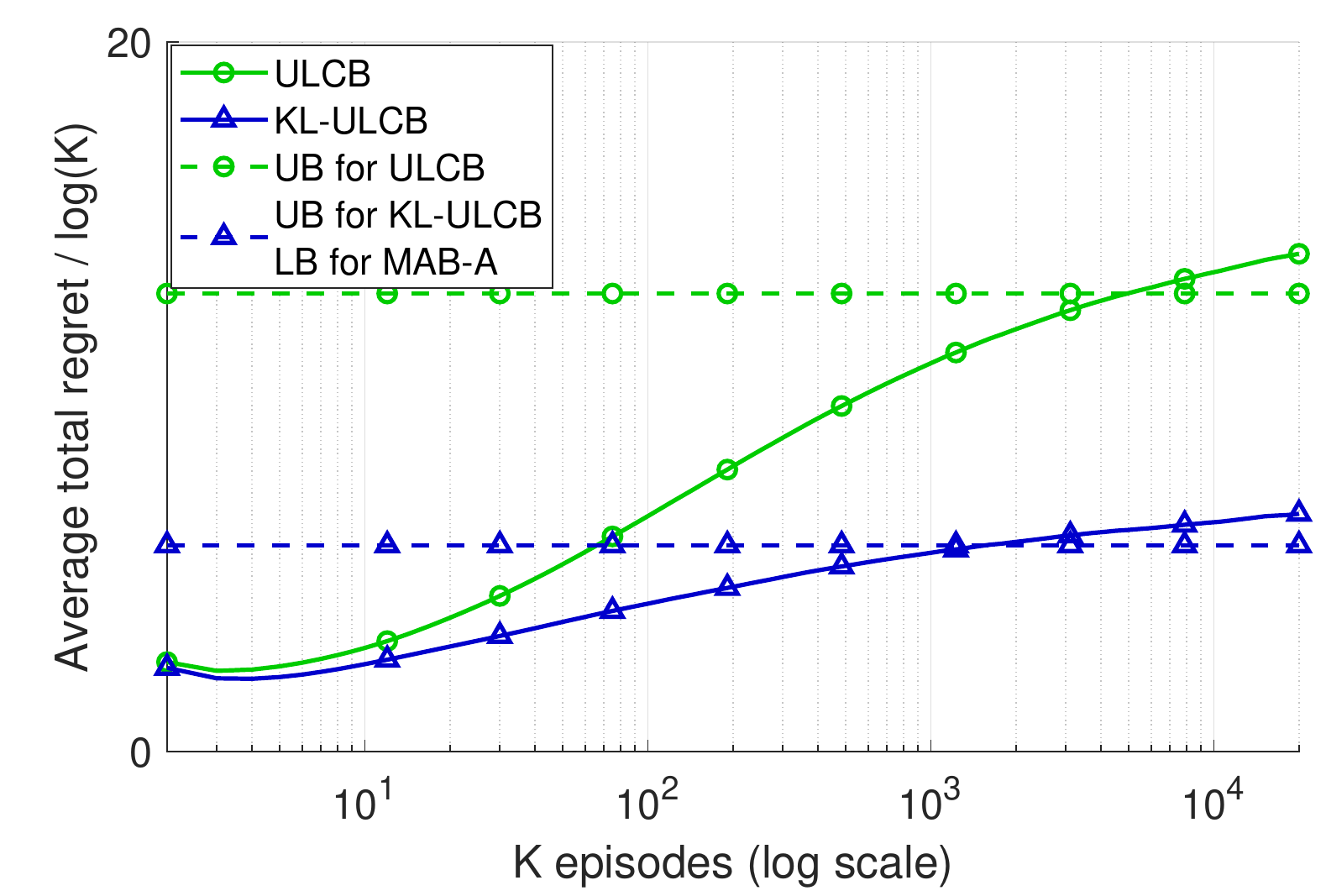}
        \label{fig:simu-0-1-ub-lb-2}
    }
    \hfill
    \caption{Simulation results: For (a) and (b), $q(0,0)=1$, $q(1,1)=q(1,0)=q(0,1)=0$. For (c) and (d), $q(0,0)=0.8$, $q(1,0)=q(0,1)=0.2$, $q(1,1)=0.1$.}
    \label{fig:simu-0-1}
\end{figure}

\section{Conclusion}
\label{sec:conclusion}

We studied a new MAB problem with abandonment. The proposed ULCB and KL-ULCB achieve $O(\log K)$ regret, and KL-ULCB is asymptotically sharp. We also extended our algorithms to the general-state setting. Simulation results show that our algorithms outperform UCB, KL-UCB, and Q-learning-based algorithms and confirm our theoretical results about the state-dependent exploration-exploitation mechanism.

\newpage
\bibliographystyle{plain}
\bibliography{MAB-A-arxiv}

\newpage
\appendix
\allowdisplaybreaks[0]

\section{Organization of the appendices}

In the appendices, we present additional details of the model, the details of the proofs of the theorems and lemmas, additional details of the extension to the general-state setting, and additional simulations. The organization of the content is as follows:
\begin{itemize}[leftmargin=*]
    \item Section~\ref{app:transition} contains the transition graph and probabilities of the MDP.
    \item Section~\ref{app:missing-proofs} contains the proofs of the theorems and lemmas in the paper.
    \begin{itemize}
        \item Section~\ref{app:lemma1}: Proof of Lemma~\ref{lemma:1}.
        \item Section~\ref{app:regret-decomp}: Proof of the regret decomposition~\eqref{equ:regret-decomp} in Section~\ref{sec:decom-regret}.
        \item Section~\ref{app:lemma:sufficient-condition}: Proof of Lemma~\ref{lemma:sufficient-condition}.
        \item Section~\ref{app:lemma:sd-cb-1}: Proof of Lemma~\ref{lemma:sd-cb-1}.
        \item Section~\ref{app:lemma:sd-cb-2}: Proof of Lemma~\ref{lemma:sd-cb-2}.
        \item Section~\ref{app:lemma:sd-cb-3}: Proof of Lemma~\ref{lemma:sd-cb-3}.
        \item Section~\ref{app:theorem2}: Proof of Theorem~\ref{theorem:2}.
        \item Section~\ref{app:theorem3}: Proof of Theorem~\ref{theorem:3}.
    \end{itemize}
    \item Section~\ref{app:extension-cont} contains additional details of the extension to the general-state setting.
    \begin{itemize}
        \item Section~\ref{app:cont-ext-model}: Transition probabilities of the MDP.
        \item Section~\ref{app:proof-lemma-optimal-policy-general}: Proof of Lemma~\ref{lemma:cont-1}.
        \item Section~\ref{app:cont-state-gap-function}: Some examples of abandonment probability functions that satisfy $V^*(s_1) - Q^*(s_1, a) \le V^*(s_2) - Q^*(s_2, a)$.
        \item Section~\ref{app:cont-ext-alg}: Details of CONT-ULCB and CONT-KL-ULCB algorithms.
        \item Section~\ref{app:cont-state-simu}: Simulation results.
        
    \end{itemize}
    
    \item Section~\ref{app:add-simulations} contains simulation parameters and additional simulation results.
    
    \item Section~\ref{app:add-cases} contains the algorithms and the results for the case where $V^*(1)-Q^*(1,a_i)\ge V^*(0)-Q^*(0,a_i)$.
    \begin{itemize}
        \item Section~\ref{app:theorem1-case2}: Theorem~\ref{theorem:5}
        \item Section~\ref{app:theorem2-case2}: Theorem~\ref{theorem:6}
        \item Section~\ref{app:theorem3-case2}: Theorem~\ref{theorem:7}
    \end{itemize}
    
\end{itemize}

\newpage
\section{State transition}
\label{app:transition}
The transition graph of the MDP is shown in Fig.~\ref{fig:mdp} with state space $\mathcal{S} = \{0,1,g\}$, action space $\mathcal{A} = \{a_1,...,a_M\}$, and Bernoulli random rewards. 

\begin{figure}[ht]
    \centering
    \includegraphics[width=0.9\linewidth]{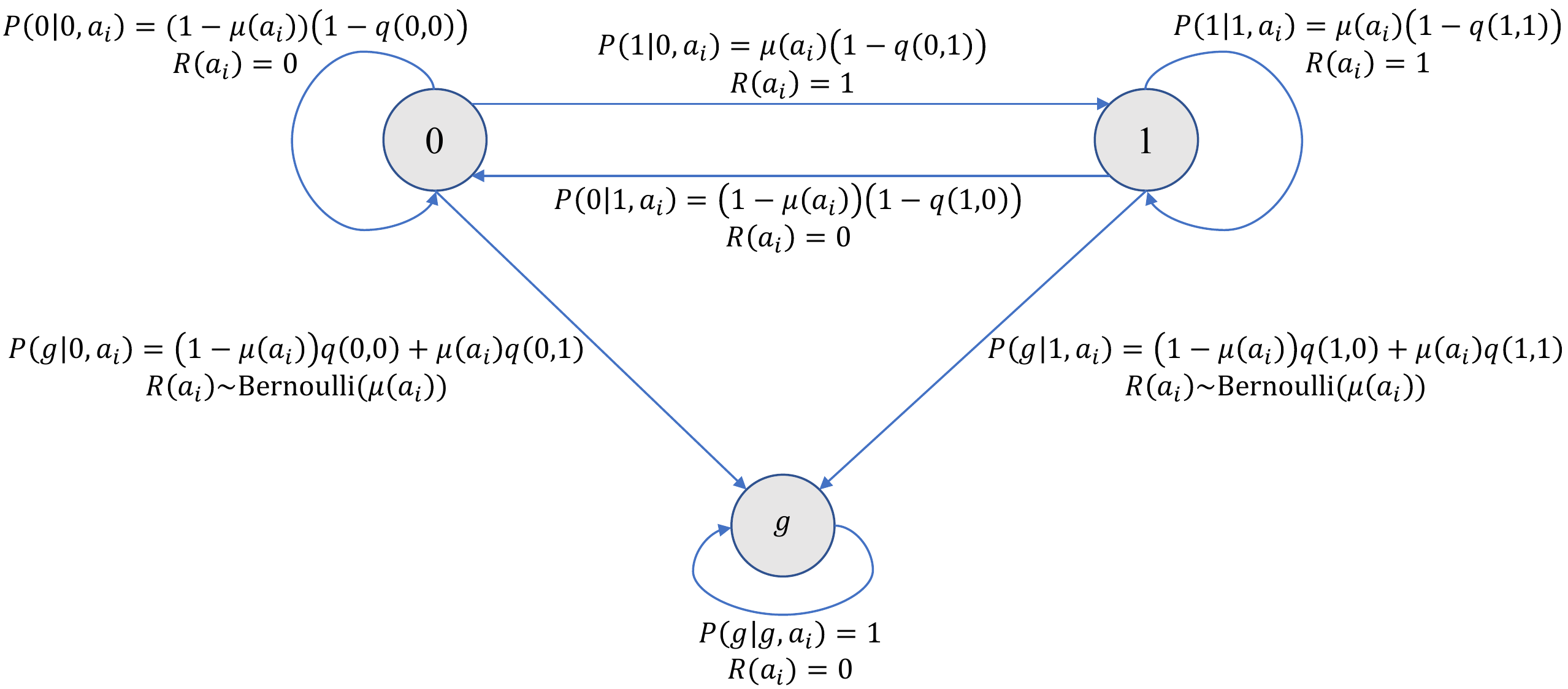}
    \caption{Transition graph for action $a_i$, $i\in\{1,2,...,M\}$.}
    \label{fig:mdp}
\end{figure}

The transition probabilities $P(s'|s,a)$ while pulling arm $a$ are shown in Table~\ref{tab:tran-prob}. The model can also be extended to the case where users never abandon the system at the first step by defining one more state in which the abandonment probability is 0.

\begin{table}[!htbp]
\centering
\small
\caption{Transition probabilities $P(s'|s,a)$}
\label{tab:tran-prob}
\begin{tabular}{|c|c|c|c|c|}
\hline
\multicolumn{2}{|c|}{ \multirow{2}*{$P(s'|s,a)$} }& \multicolumn{3}{c|}{Next state $s'$}\\
\cline{3-5}
\multicolumn{2}{|c|}{}&0&1&g\\
\hline
\multirow{3}*{Current state $s$}&0&$(1-\mu(a))(1-q(0,0))$&$\mu(a)(1-q(0,1))$&$(1-\mu(a))q(0,0)+\mu(a)q(0,1)$\\
\cline{2-5}
&1&$(1-\mu(a))(1-q(1,0))$&$\mu(a)(1-q(1,1))$&$(1-\mu(a))q(1,0)+\mu(a)q(1,1)$\\
\cline{2-5}
&g&0&0&1\\
\hline
\end{tabular}
\end{table}

\section{Missing proofs}
\label{app:missing-proofs}

\subsection{Proof of Lemma~\ref{lemma:1}: optimal policy}
\label{app:lemma1}

If the model is known, this problem can be viewed as a SSP problem~\cite{bertsekas1991analysis}. Since $\mu(a_i) \le \mu(a_1) < 1, \forall i=2,...,M$ and $q(0,0)>0$, all policies are proper. Hence, by the results in~\cite{bertsekas1991analysis}, there exists a stationary optimal policy. Therefore, it is enough to consider only stationary policies for $\pi^*$.
Define for any state $s\in\mathcal{S}$ and $a\in\mathcal{A}$,
\begin{equation}\label{equ:v*-def}
    V^{*}(s) \coloneqq  \expt_{\pi^*} \left[\sum_{h=1}^{\infty} R_{k,h}(A_{k,h})\left|\right.S_{k,1}=s\right]
\end{equation}
where $A_{k,h}$ follows the policy $\pi^*$, and
\begin{equation}\label{equ:q*-def}
    Q^{*}(s, a) \coloneqq  \begin{cases} 
    \mu(a) + \expt_{\pi^*} \left[\sum_{h=2}^{\infty} R_{k,h}(A_{k,h})\left|\right.S_{k,1}=s, A_{k,1}=a\right], & s\neq g\\
    0, & s=g
    \end{cases}
\end{equation}
where $A_{k,h}, h\ge 2$ follows the policy $\pi^*$.
Note that $V^{*}(s)$ and $Q^{*}(s, a)$ do not depend on $k$ since the statistics of the MDPs remain the same among episodes and these MDPs are independent. For $s\neq g$, we have the Bellman equation as follows:
\begin{equation}\label{equ:bellman-v*q*}
\begin{aligned}
    V^{*}(s) =& \max_{a} Q^{*}(s, a)\\
    Q^*(s,a) = & \mu(a) + \expt\left[\expt_{\pi^*}\left[\sum_{h=2}^{\infty}R_{k,h}(A_{k,h})\left|S_{k,2}\right.\right]\left|\right.S_{k,1}=s,A_{k,1}=a\right]\\
    = & \mu(a)+\expt \left[V^*(S_{k,2})|S_{k,1} = s, A_{k,1} = a\right]
\end{aligned}
\end{equation}
Thus, we have
\begin{equation}\label{equ:q-function}
    \begin{aligned}
    Q^*(s,a) = &\mu(a) + P(0\vert s,a)V^*(0) +P(1\vert s,a)V^*(1)\\
    = & \mu(a) + (1-\mu(a))(1-q(s,0))V^*(0) + \mu(a)(1-q(s,1))V^*(1)
    \end{aligned}
\end{equation}
for any $s\in\{0,1\}$ and $a\in\mathcal{A}$. Then we have
\begin{equation}\label{equ:q1-q0}
    \begin{aligned}
        Q^*(1,a)-Q^*(0,a) = & \biggl[(1-\mu(a))(1-q(1,0))V^*(0) + \mu(a)(1-q(1,1))V^*(1) \biggr]\\
        & - \biggl[ (1-\mu(a))(1-q(0,0))V^*(0) + \mu(a)(1-q(0,1))V^*(1)\biggr]\\
        = & (1-\mu(a))V^*(0)(q(0,0)-q(1,0)) + \mu(a)V^*(1)(q(0,1)-q(1,1))
    \end{aligned}
\end{equation}
Since by definition we know $V^*(s)\ge 0$ for any $s\in\{0,1\}$, and we know $q(0,0)-q(1,0)\ge 0$ and $q(0,1)-q(1,1)\ge 0$, from the result of (\ref{equ:q1-q0}), we have $Q^*(1,a)-Q^*(0,a)\ge 0$ for any $a\in\{a_1,...,a_M\}$. Therefore, we have
\begin{equation}
    \begin{aligned}
    V^*(1)- V^*(0) = & \max_{a} Q^*(1,a) - \max_{a} Q^*(0,a) \\
                   = & \max_{a} Q^*(1,a) -  Q^*(0,a')\\
                   \ge & Q^*(1,a') -  Q^*(0,a')\\
                   \ge & 0
    \end{aligned}
\end{equation}
where $a'\coloneqq \mathop{\mathrm{argmax}}_a Q^*(0,a)$. Then by (\ref{equ:q-function}), for any $i=2,...,M$, we have
\begin{align}\label{equ:qa1-qa2}
    Q^*(s,a_1) - Q^*(s,a_i) = & (\mu(a_1)-\mu(a_i)) + (\mu(a_i)-\mu(a_1))(1-q(s,0))V^*(0) \nonumber\\
    & + (\mu(a_1)-\mu(a_i))(1-q(s,1))V^*(1)\nonumber\\
    = & (\mu(a_1)-\mu(a_i))+(\mu(a_1)-\mu(a_i))\biggl[(1-q(s,1))V^*(1)-(1-q(s,0))V^*(0)\biggr]
\end{align}
where
\begin{equation}\label{equ:temp-1}
    (1-q(s,1))V^*(1)-(1-q(s,0))V^*(0) \ge 0
\end{equation}
due to the fact that $V^*(1)\ge V^*(0)\ge 0$ and $q(s,0)\ge q(s,1)$ for any $s\in\{0,1\}$. Therefore, by (\ref{equ:qa1-qa2}), (\ref{equ:temp-1}), and $\mu(a_1)\ge \mu(a_i), \forall i=2,...,M$, we have $Q^*(s,a_1)\ge Q^*(s,a_i)$ for any $i=2,...,M$ and $s\in\{0,1\}$. Therefore,  always pulling Arm $a_1$ is an optimal policy.

\subsection{Proof of the regret decomposition}
\label{app:regret-decomp}

From the definition of $V^*$ and $Q^*$ in~\eqref{equ:v*-def-v2} and~\eqref{equ:q*-def-v2} and by Lemma~\ref{lemma:1}, we have the following Bellman equation:
\begin{equation}\label{equ:bellman-v*q*-v2}
\begin{aligned}
    V^{*}(s) =& \max_{a} Q^{*}(s, a) = Q^{*}(s, a_1)\\
    Q^*(s,a)
    = & \mu(a)+\expt \left[V^*(S_{k,2})|S_{k,1} = s, A_{k,1} = a\right]
\end{aligned}
\end{equation}
for $s\neq g$. Similarly, we have the Bellman equation for $V^{\pi}$ and $Q^{\pi}$ as follows:
\begin{equation}\label{equ:bellman-vpi-qpi}
    \begin{aligned}
    V^{\pi}(s,\varphi) =& Q^{\pi}(s,\varphi, \pi(s,\varphi))\\
    Q^{\pi}(s,\varphi,a) =& \mu(a) + \expt\Biggl[\expt\biggl[\sum_{h=2}^{I_k(\pi,S_{k,2},\phi_{k,2})+1}R_{k,h}(\pi(S_{k,h},\phi_{k,h}))\left|S_{k,2},\phi_{k,2}\right.\biggr]\Biggr.\nonumber\\
    &~~~~~~~~~~~~~~~~~~\biggl|\biggr.\Biggl.S_{k,1}=s,\phi_{k,1}=\varphi,A_{k,1}=a\Biggr]\\
    = & \mu(a) + \expt\left[V^{\pi}(S_{k,2},\phi_{k,2})\left|\right.S_{k,1}=s,\phi_{k,1}=\varphi,A_{k,1}=a\right]
    \end{aligned}
\end{equation}
for $s\neq g$.

From~\eqref{equ:regret-2}, the regret induced in episode $k$ is $\expt\left[V^*(S_{k,1})\right] - \expt\left[V^{\pi}(S_{k,1},\phi_{k,1})\right]$, which can be decomposed as follows:
\begin{align}
    & \expt\left[V^*(S_{k,1})\right] - \expt\left[V^{\pi}(S_{k,1},\phi_{k,1})\right] \nonumber\\
    =& \expt\left[ V^*(S_{k,1}) - Q^*(S_{k,1}, A_{k,1}) \right] + \expt\left[Q^*(S_{k,1}, A_{k,1}) - V^{\pi}(S_{k,1},\phi_{k,1})\right] \nonumber\\
    =& \expt\left[ V^*(S_{k,1}) - Q^*(S_{k,1}, A_{k,1}) \right] + \expt\left[Q^*(S_{k,1}, A_{k,1}) - Q^{\pi}(S_{k,1},\phi_{k,1}, A_{k,1}) \right]\nonumber\\
    =& \expt\left[ V^*(S_{k,1}) - Q^*(S_{k,1}, A_{k,1}) \right]\nonumber\\
    & + \expt\left[\expt\left[V^*(S_{k,2}) \vert S_{k,1}, A_{k,1} \right] - \expt\left[V^{\pi}(S_{k,2}, \phi_{k,2})\vert S_{k,1},\phi_{k,1}, A_{k,1} \right] \right]\nonumber\\
    =& \expt\left[ V^*(S_{k,1}) - Q^*(S_{k,1}, A_{k,1}) \right] + \expt\left[V^*(S_{k,2}) - V^{\pi}(S_{k,2},\phi_{k,2})\right] = ...\nonumber\\
    =& \sum_{h=1}^{\infty} \expt\left[V^*(S_{k,h}) - Q^*(S_{k,h}, A_{k,h})\right]
\end{align}
where $S_{k,h}$, $\phi_{k,h}$, and $A_{k,h}$ are the states, historical samples, and actions following the policy $\pi$, respectively. The second equality is due to the fact that $A_{k,h}=\pi(S_{k,h},\phi_{k,h})$, the third equality follows from the Bellman equations~\eqref{equ:bellman-v*q*-v2} and~\eqref{equ:bellman-vpi-qpi}, and the fourth equality is by the tower law. The limit in the result is well-defined since $V^*(S_{k,h}) - Q^*(S_{k,h}, A_{k,h})\ge 0$. In fact, this regret decomposition borrows from~\cite{yang2021q}, and it can also be viewed as the performance difference formula~\cite{Kakade02approximatelyoptimal} in the RL literature. Then the regret can be further decomposed into the summation of the gaps between value function and Q function in different states shown as follows:
\begin{align}\label{equ:regret-decomp-app}
    & \expt[\mathrm{Reg}_{\pi}(K)] = \sum_{k=1}^{K} \sum_{h=1}^{\infty} \expt\left[V^*(S_{k,h}) - Q^*(S_{k,h}, A_{k,h})\right]\nonumber\\
    = & \expt\left[ \sum_{k=1}^{K} \sum_{h=1}^{\infty} V^*(S_{k,h}) - Q^*(S_{k,h}, A_{k,h})\right]\nonumber\\
    = & \expt\left[ \sum_{k=1}^{K} \sum_{h=1}^{I_k(\pi, S_{k,1}, \phi_{k,1})} V^*(S_{k,h}) - Q^*(S_{k,h}, A_{k,h})\right]
    = \expt\left[ \sum_{t=1}^{T(K,\pi)} V^*(S_{t}) - Q^*(S_{t}, A_{t})\right]\nonumber\\
    = & \expt\left[ \sum_{t=1}^{T(K,\pi)} \sum_{i=2}^{M} \mathbbm{1}\{S_t = 0, A_t = a_i\} \left[ V^*(0) - Q^*(0, a_i)\right]+ \mathbbm{1}\{S_t = 1, A_t = a_i\} \left[ V^*(1) - Q^*(1, a_i)\right] \right]
\end{align}
where the second equality is by monotone convergence theorem, the third equality is by the definition of $I_k(\pi, S_{k,1}, \phi_{k,1})$, $T(K,\pi)\coloneqq \sum_{k=1}^{K} I_k(\pi, S_{k,1}, \phi_{k,1})$ is the number of pulls over $K$ episodes following the policy $\pi$, and the last equality follows from the fact that $V^*(s)-Q^*(s,a_1)=0$ for any $s$.

\subsection{Proof of Lemma~\ref{lemma:sufficient-condition}}
\label{app:lemma:sufficient-condition}

Lemma~\ref{lemma:sufficient-condition} can be proved by obtaining a lower bound for the ratio $V^*(0)/V^*(1)$. 

For $i\in\{2,3,...,M\}$, we have
\begin{align}\label{equ:vqdiff}
    & \left[V^*(0) - Q^*(0, a_i)\right] - \left[V^*(1) - Q^*(1, a_i)\right]\nonumber\\
    = & \left[Q^*(0, a_1) - Q^*(0, a_i)\right] - \left[Q^*(1, a_1) - Q^*(1, a_i)\right]\nonumber\\
    = & \left[\mu(a_1)-\mu(a_i)\right]\left[1 + (1-q(0,1))V^*(1) - (1-q(0,0))V^*(0)\right]\nonumber\\
    & - \left[\mu(a_1)-\mu(a_i)\right]\left[1 + (1-q(1,1))V^*(1) - (1-q(1,0))V^*(0)\right] \nonumber\\
    = & \left[\mu(a_1)-\mu(a_i)\right] \left[(q(0,0)-q(1,0))V^*(0) - (q(0,1)-q(1,1))V^*(1)\right]
\end{align}
where the first and second equalities follow from~\eqref{equ:bellman-v*q*-v2} and Lemma~\ref{lemma:1}. Since $\mu(a_1)>0$, we have $V^*(1) > 0$. Then by the Bellman equation~\eqref{equ:bellman-v*q*-v2} and Lemma~\ref{lemma:1}, we have
\begin{align}\label{equ:lower-bound-ratio}
    &\frac{V^*(0)}{V^*(1)} \nonumber\\
    = & \frac{\mu(a_1) + \mu(a_1)(1-q(0,1))V^*(1) + (1-\mu(a_1))(1-q(0,0))V^*(0)}{\mu(a_1) + \mu(a_1)(1-q(1,1))V^*(1) + (1-\mu(a_1))(1-q(1,0))V^*(0)}\nonumber\\
    = & \frac{\mu(a_1) + \frac{(1-q(0,1))}{(1-q(1,1))}\mu(a_1)(1-q(1,1))V^*(1) + \frac{(1-q(0,0))}{(1-q(1,0))}(1-\mu(a_1))(1-q(1,0))V^*(0)}{\mu(a_1) + \mu(a_1)(1-q(1,1))V^*(1) + (1-\mu(a_1))(1-q(1,0))V^*(0)}\nonumber\\
    \ge & \frac{\min\left\{\frac{1-q(0,1)}{1-q(1,1)}, \frac{1-q(0,0)}{1-q(1,0)}\right\}\left[\mu(a_1) + \mu(a_1)(1-q(1,1))V^*(1) + (1-\mu(a_1))(1-q(1,0))V^*(0)\right]}{\mu(a_1) + \mu(a_1)(1-q(1,1))V^*(1) + (1-\mu(a_1))(1-q(1,0))V^*(0)}\nonumber\\
    = & \min\left\{\frac{1-q(0,1)}{1-q(1,1)}, \frac{1-q(0,0)}{1-q(1,0)}\right\}
\end{align}
where the inequality is due to the fact that $\min\left\{\frac{1-q(0,1)}{1-q(1,1)}, \frac{1-q(0,0)}{1-q(1,0)}\right\}\le 1$. It follows from~\eqref{equ:theorem:1-condition} and~\eqref{equ:lower-bound-ratio} that $\frac{V^*(0)}{V^*(1)} \ge \frac{q(0,1)-q(1,1)}{q(0,0)-q(1,0)}$, which implies 
$$(q(0,0)-q(1,0))V^*(0) - (q(0,1)-q(1,1))V^*(1) \ge 0.$$
Hence, it follows from~\eqref{equ:vqdiff} that $V^*(0) - Q^*(0, a_i) \ge V^*(1) - Q^*(1, a_i)$.
 
\subsection{Proof of Lemma~\ref{lemma:sd-cb-1}}
\label{app:lemma:sd-cb-1}

Choose a $T_1$ such that for any $t\ge T_1$, 
    $$\frac{(p_{\min}-\eta)(t-1)}{2(M-1)}\ge 2,~\mathrm{and}~ \sqrt{\frac{\log t + 4 \log \left(\log t \right)}{\frac{(p_{\min} - \eta)(t-1)}{(M-1)} - 2}} \le (\mu(a_1) - \mu(a_2)) - \gamma.$$

Let $N_{t}^{1}(a)$ be the number of times arm $a\in\mathcal{A}$ was pulled in state $1$ before time step $t$. Then
\begin{align}\label{equ:prob-nta}
    &\prob\left(N_{t}(a_1)\le \frac{(p_{\min}-\eta)(t-1)}{2}\right) \nonumber \\
    \le & \prob\left(N_{t}^{1}(a_1)\le \frac{(p_{\min}-\eta)(t-1)}{2}\right) \nonumber \\
    \le & \prob\left(N_{t}^{1}(a_1)\le \frac{(p_{\min} - \eta)(t-1)}{2}, \sum_{i=1}^{t-1}\mathbbm{1}\{S_i=1\}>(p_{\min}-\eta)(t-1)\right) \nonumber \\
    & + \prob\left(\sum_{i=1}^{t-1}\mathbbm{1}\{S_i=1\}\le(p_{\min}-\eta)(t-1)\right)
\end{align}
where the first inequality follows from the fact that $N_{t}^{1}(a)\le N_{t}(a)$. Next we show that 
\[
    \prob\left(\sum_{i=1}^{t-1}\mathbbm{1}\{S_i=1\}\le(p_{\min}-\eta)(t-1)\right)
\]
is small. Let $\mathcal{F}_0\coloneqq \{\emptyset, \Omega\}$ be the minimum $\sigma$-algebra, and $\mathcal{F}_i\coloneqq \sigma(S_1,A_1,...,S_i,A_i)$ be the $\sigma$-algebra generated by the random variables up to time $i$. Since for any $a\in\mathcal{A}$,
\begin{align}
    \prob\left(S_i=1\vert S_{i-1}=0, A_{i-1}=a\right) \ge & \mu(a)(1-q(0,1))\nonumber\\
    \prob\left(S_i=1\vert S_{i-1}=1, A_{i-1}=a\right) \ge & \mu(a)(1-q(1,1)),
\end{align}
we have
\begin{align}
   \expt\left[\mathbbm{1}\{S_i=1\}\vert \mathcal{F}_{i-1}\right] \ge \mu(a_M)\min\left\{1-q(0,1), 1-q(1,1)\right\} = p_{\min} > 0 \nonumber\\
\end{align}
Hence we have
\begin{align}\label{equ:prob-sum-indicator}
    &\prob\left(\sum_{i=1}^{t-1}\mathbbm{1}\{S_i=1\}\le(p_{\min}-\eta)(t-1)\right)\nonumber\\
    = & \prob\left(\sum_{i=1}^{t-1}p_{\min} - \sum_{i=1}^{t-1}\mathbbm{1}\{S_i=1\}\ge \eta(t-1)\right)\nonumber\\
    \le & \prob\left(\sum_{i=1}^{t-1}\left(\expt\left[\mathbbm{1}\{S_i=1\}\vert \mathcal{F}_{i-1}\right] - \mathbbm{1}\{S_i=1\}\right)\ge \eta(t-1)\right)
\end{align}
Let $\Delta_i\coloneqq \expt\left[\mathbbm{1}\{S_i=1\}\vert \mathcal{F}_{i-1}\right] - \mathbbm{1}\{S_i=1\}$. Note that $\Delta_i$ is measurable with respect to $\mathcal{F}_{i}$, $\expt[\Delta_i\vert \mathcal{F}_{i-1}]=0$, and $\lvert \Delta_i \rvert \le 1$. Hence by Azuma-Hoeffding inequality~\cite{van2016probability}, we have
\begin{align}\label{equ:prob-conc-indicator-state}
\prob\left(\sum_{i=1}^{t-1}\left(\expt\left[\mathbbm{1}\{S_i=1\}\vert \mathcal{F}_{i-1}\right] - \mathbbm{1}\{S_i=1\}\right)\ge \eta(t-1)\right)\le \exp\left(-\frac{\eta^2(t-1)^2}{2(t-1)}\right) = \exp\left(-\frac{\eta^2(t-1)}{2}\right)
\end{align}
Therefore, from~\eqref{equ:prob-nta},~\eqref{equ:prob-sum-indicator} and~\eqref{equ:prob-conc-indicator-state}, it follows that
\begin{align}\label{equ:nta1-first-ineq}
    &\prob\left(N_{t}(a_1)\le \frac{(p_{\min}-\eta)(t-1)}{2}\right) \nonumber \\
    \le & \prob\left(N_{t}^{1}(a_1)\le \frac{(p_{\min} - \eta)(t-1)}{2}, \sum_{i=1}^{t-1}\mathbbm{1}\{S_i=1\}>(p_{\min}-\eta)(t-1)\right) + \exp\left(-\frac{\eta^2(t-1)}{2}\right)\nonumber\\
    \le & \prob\left(\sum_{i=2}^{M} N_{t}^{1}(a_i) > \frac{(p_{\min} - \eta)(t-1)}{2}\right) + \exp\left(-\frac{\eta^2(t-1)}{2}\right)\nonumber\\
    \le & \prob\left(N_{t}^{1}(a_j) > \frac{(p_{\min} - \eta)(t-1)}{2(M-1)}\right) + \exp\left(-\frac{\eta^2(t-1)}{2}\right)\nonumber\\
\end{align}
where the second inequality is due to the fact that $\sum_{i=1}^{M} N_t^{1}(a_i) = \sum_{i=1}^{t-1}\mathbbm{1}\{S_i=1\}$ and in the last inequality $j\in\operatorname{\mathrm{argmax}}_{i\in\{2,...,M\}} N_{t}^{1}(a_i)$. Consider the event $\{N_{t}^{1}(a_j) > \frac{(p_{\min} - \eta)(t-1)}{2(M-1)}\}$. Let $\tau_t<t$ be the time step when $a_j$ is pulled in state $1$ for the $\left\lceil{\frac{(p_{\min} - \eta)(t-1)}{2(M-1)}}\right\rceil$-th time. Then we have
\begin{align}\label{equ:tau-t}
    \tau_t \ge & \left\lceil{\frac{(p_{\min} - \eta)(t-1)}{2(M-1)}}\right\rceil + (M-1) 
    \ge \frac{(p_{\min} - \eta)(t-1)}{2(M-1)} + (M-1)\\
    \label{equ:N-tau-t}
    N_{\tau_t}^{1}(a_j) = & \left\lceil{\frac{(p_{\min} - \eta)(t-1)}{2(M-1)}}\right\rceil - 1
    \ge \frac{(p_{\min} - \eta)(t-1)}{2(M-1)} - 1
\end{align}
where the first inequality is due to the fact that the ULCB algorithm pulls the other $(M-1)$ arms at the beginning. Let $L_t\coloneqq \frac{(p_{\min} - \eta)(t-1)}{2(M-1)} + (M-1)$. Then we have
\begin{align}\label{equ:prob-nt1aj}
    & \prob\left(N_{t}^{1}(a_j) > \frac{(p_{\min} - \eta)(t-1)}{2(M-1)}\right) \nonumber\\
    \le & \prob\left(\tau_t \ge L_t, N_{\tau_t}^{1}(a_j) = \left\lceil{\frac{(p_{\min} - \eta)(t-1)}{2(M-1)}}\right\rceil - 1, S_{\tau_t} = 1, A_{\tau_t}=a_j\right)\nonumber\\
    \le & \prob\left(\tilde{\mu}_{\tau_t}^{1}(a_j)\ge \tilde{\mu}_{\tau_t}^{1}(a_1), \tau_t \ge L_t, N_{\tau_t}^{1}(a_j) = \left\lceil{\frac{(p_{\min} - \eta)(t-1)}{2(M-1)}}\right\rceil - 1 \right)\nonumber\\
    \le & \prob\left(\tilde{\mu}_{\tau_t}^{1}(a_j)\ge \tilde{\mu}_{\tau_t}^{1}(a_1), \tau_t \ge L_t, N_{\tau_t}(a_j) \ge \left\lceil{\frac{(p_{\min} - \eta)(t-1)}{2(M-1)}}\right\rceil - 1 \right)\nonumber\\
    \le & \prob\left(\tilde{\mu}_{\tau_t}^{1}(a_j)\ge \tilde{\mu}_{\tau_t}^{1}(a_1), \tilde{\mu}_{\tau_t}^{1}(a_1) \ge \mu(a_1), \tau_t \ge L_t, N_{\tau_t}(a_j) \ge \left\lceil{\frac{(p_{\min} - \eta)(t-1)}{2(M-1)}}\right\rceil - 1 \right) \nonumber\\
    & + \prob\left(\tilde{\mu}_{\tau_t}^{1}(a_1) < \mu(a_1), \tau_t \ge L_t\right)
\end{align}
where the first inequality follows from~\eqref{equ:tau-t},~\eqref{equ:N-tau-t}, and the definition of $\tau_t$, the second inequality follows from $S_{\tau_t}=1$, $A_{\tau_t}=a_j$, and Line~\ref{line:alg1-state1-action} of Algorithm~\ref{alg:1}, the third inequality is due to the fact that $ N_{\tau_t}(a_j) \ge  N_{\tau_t}^{1}(a_j)$, and the last inequality is by law of total probability. 

By Lemma~\ref{lemma:sd-cb-prob-tilde-mu} (which is presented after this proof) and $c_1=1, c=4$, for any $t\ge T_1$, we can bound the second term in~\eqref{equ:prob-nt1aj} as follows:
\begin{align}\label{equ:nta1-second-ineq}
    & \prob\left(\tilde{\mu}_{\tau_t}^{1}(a_1) < \mu(a_1), \tau_t \ge L_t\right) \nonumber\\
    \le & \frac{e\log L_t \log(t-2) + 4e\log (\log L_t) \log(t-2) + e}{L_t(\log L_t)^4}\nonumber\\
    \le & \frac{e\log(t-2) + 4\log(t-2) + e}{L_t(\log L_t)^3}\nonumber\\
    \le & \frac{(4+e)\log(t-1) + e}{\frac{(p_{\min} - \eta)(t-1)}{2(M-1)} \left[\log\left(\frac{(p_{\min} - \eta)(t-1)}{2(M-1)}\right)\right]^3}\nonumber\\
    \le & \frac{c_3}{c_2(t-1)\left[\log \left(c_2(t-1)\right)\right]^2}
\end{align}
where the second inequality is obtained by dividing the numerator and the denominator by $\log L_t$, the third inequality is by the definition of $L_t$, and the last inequality is obtained by dividing the numerator and the denominator by $\log \left(c_2(t-1)\right)$, where $c_2\coloneqq  \frac{p_{\min} - \eta}{2(M-1)}$ and $c_3\coloneqq \frac{4+e}{\log 2}\log\frac{2}{c_2} + \frac{e}{\log 2}$. For the first term in~\eqref{equ:prob-nt1aj}, we have
\begin{align}\label{equ:nta1-third-ineq}
    &\prob\left(\tilde{\mu}_{\tau_t}^{1}(a_j)\ge \tilde{\mu}_{\tau_t}^{1}(a_1), \tilde{\mu}_{\tau_t}^{1}(a_1) \ge \mu(a_1), \tau_t \ge L_t, N_{\tau_t}(a_j) \ge \left\lceil{\frac{(p_{\min} - \eta)(t-1)}{2(M-1)}}\right\rceil - 1 \right) \nonumber\\
    \le & \prob\left(\tilde{\mu}_{\tau_t}^{1}(a_j)\ge \mu(a_1), N_{\tau_t}(a_j) \ge \left\lceil{\frac{(p_{\min} - \eta)(t-1)}{2(M-1)}}\right\rceil - 1\right)
\end{align}
Consider the event $\{\tilde{\mu}_{\tau_t}^{1}(a_j)\ge \mu(a_1), N_{\tau_t}(a_j) \ge \left\lceil{\frac{(p_{\min} - \eta)(t-1)}{2(M-1)}}\right\rceil - 1\}$. Then we have
\begin{align}
    \tilde{\mu}_{\tau_t}^{1}(a_j) = & \bar{\mu}_{\tau_t}(a_j) + \sqrt{\frac{\log \tau_t + 4 \log \left(\log \tau_t\right)}{2N_{\tau_t}(a_j)}}\nonumber\\
    \ge & \mu(a_1) = \mu(a_j) + (\mu(a_1) - \mu(a_j)) \ge \mu(a_j) + (\mu(a_1) - \mu(a_2))
\end{align}
which implies
\begin{align}
    \bar{\mu}_{\tau_t}(a_j) - \mu(a_j) \ge & (\mu(a_1) - \mu(a_2)) - \sqrt{\frac{\log \tau_t + 4 \log \left(\log \tau_t\right)}{2N_{\tau_t}(a_j)}}\nonumber\\
    \ge & (\mu(a_1) - \mu(a_2)) - \sqrt{\frac{\log t + 4 \log \left(\log t \right)}{\frac{(p_{\min} - \eta)(t-1)}{(M-1)} - 2}}\nonumber\\
    \ge & \gamma
\end{align}
where the second inequality is by $\tau_t<t$ and $N_{\tau_t}(a_j) \ge \left\lceil{\frac{(p_{\min} - \eta)(t-1)}{2(M-1)}}\right\rceil - 1$, and the last inequality is by $t\ge T_1$ and the definition of $T_1$.
Hence we have
\begin{align}\label{equ:nta1-fourth-ineq}
    & \prob\left(\tilde{\mu}_{\tau_t}^{1}(a_j)\ge \mu(a_1), N_{\tau_t}(a_j) \ge \left\lceil{\frac{(p_{\min} - \eta)(t-1)}{2(M-1)}}\right\rceil - 1\right)\nonumber\\
    \le &\prob\left(\bar{\mu}_{\tau_t}(a_j) - \mu(a_j)\ge \gamma, N_{\tau_t}(a_j) \ge \left\lceil{\frac{(p_{\min} - \eta)(t-1)}{2(M-1)}}\right\rceil - 1 \right)\nonumber\\
    \le & \sum_{i=2}^M \prob\left(\bar{\mu}_{\tau_t}(a_i) - \mu(a_i)\ge \gamma, N_{\tau_t}(a_i) \ge \left\lceil{\frac{(p_{\min} - \eta)(t-1)}{2(M-1)}}\right\rceil - 1 \right)\nonumber\\
    \le & \sum_{i=2}^M \sum_{n=\left\lceil{\frac{(p_{\min} - \eta)(t-1)}{2(M-1)}}\right\rceil - 1}^{t - 1} \prob\left( \frac{1}{n}\sum_{s=1}^n R_s(a_i)  - \mu(a_i)\ge \gamma\right)\nonumber\\
    \le & (M-1) \sum_{n=\left\lceil{\frac{(p_{\min} - \eta)(t-1)}{2(M-1)}}\right\rceil - 1}^{t - 1} \exp(-2n\gamma^2)\nonumber\\
    \le & \frac{M-1}{2\gamma^2\exp\left(2\gamma^2 c_2(t-1) - 4\gamma^2 \right)}.
\end{align}
In the derivation above, the second inequality is by the union bound over all possible $j$. The third inequality is by the union bound over all possible number of pulls of arm $a_i$, where $\{R_s(a_i)\}_{s=1}^n$ are $n$ i.i.d. Bernoulli rewards of pulling arm $a_i$. The fourth inequality uses Hoeffding inequality, and the last inequality is by integration.

From~\eqref{equ:nta1-first-ineq},~\eqref{equ:prob-nt1aj},~\eqref{equ:nta1-second-ineq},~\eqref{equ:nta1-third-ineq}, and~\eqref{equ:nta1-fourth-ineq}, it follows that
\begin{align}
    & \prob\left(N_{t}(a_1)\le \frac{(p_{\min}-\eta)(t-1)}{2}\right) \nonumber\\
    \le & \frac{M-1}{2\gamma^2\exp\left(2\gamma^2 c_2(t-1) - 4\gamma^2 \right)} + \frac{c_3}{c_2(t-1)\left[\log \left(c_2(t-1)\right)\right]^2} + \exp\left(-\frac{\eta^2(t-1)}{2}\right)
\end{align}
for any $t\ge T_1$.

\begin{lemma}\label{lemma:sd-cb-prob-tilde-mu}
    Let all the assumptions in Lemma~\ref{lemma:sd-cb-1} hold. Consider the ULCB algorithm. Let $L_t\coloneqq \frac{(p_{\min} - \eta)(t-1)}{2(M-1)} + (M-1)$, $\delta_{L_t}\coloneqq  c_1^2\log L_t + c c_1^2 \log(\log L_t)$. Then for any $t\ge T_1$, 
    \begin{align}
        \prob\left(\tilde{\mu}_{\tau_t}^{1}(a_1) < \mu(a_1), \tau_t \ge L_t\right) \le e \left\lceil \delta_{L_t} \log (t-2) \right\rceil \exp(-\delta_{L_t})
    \end{align}
\end{lemma}

\begin{proof}
This lemma can be proved by a minor modification of the proof of Theorem~10 in~\cite{garivier2011kl}. The main difference is that $\tau_t$ is a random variable with a lower bound $L_t$. We present the whole proof for completeness. We first define several notations and construct a martingale. Let $n$ be any time step. Let $\{X_i\}_{i=1}^{n}$ be the i.i.d. Bernoulli rewards generated by pulling arm $a_1$. Let $\{\mathcal{F}'_{i}\}$ be an increasing sequence of $\sigma$-algebra such that
\begin{align}
    \mathcal{F}'_0\coloneqq &\sigma(S_1,A_1)\nonumber\\
    \mathcal{F}'_i\coloneqq &\sigma(S_1,A_1,X_1,...,S_i, A_i, X_i, S_{i+1}, A_{i+1}),~i\ge 1
\end{align}
Note that $X_i$ is independent of $\mathcal{F}'_{i-1}$ and $X_i$ is measurable with respect to $\mathcal{F}'_{i}$. By the definition of $\tau_n$, the event $\{\tau_n - 1\ge i\}=\{\tau_n \le i\}^{\mathrm{c}}$ is measurable with respect to $\mathcal{F}'_{i-1}$ for any $i\in\{1,...,n-2\}$. Let
\begin{align}
    V_{n}\coloneqq & \sum_{i=1}^{n-2}\epsilon_i X_i,~n\ge 2\\
    U_{n}\coloneqq & \sum_{i=1}^{n-2}\epsilon_i,~n\ge 2
\end{align}
where $\epsilon_i\coloneqq \mathbbm{1}\{A_i=a_1, i\le \tau_{n} - 1\}$, which is measurable with respect to $\mathcal{F}'_{i-1}$. Hence, $V_{n}$ and $U_{n+1}$ are measurable with respect to $\mathcal{F}'_{n-2}$. For any $\lambda\in\mathbb{R}$, let $\phi(\lambda)\coloneqq \log\expt\left[\exp(\lambda X_1)\right]$.
For any $n\ge 0$, define $W_{n}^{\lambda}$ by
\begin{align}\label{equ:wn-lambda}
    W_{n}^{\lambda} \coloneqq  \exp\left(\lambda V_{n+2} - U_{n+2}\phi(\lambda)\right)
\end{align}
For any $n\ge 1$, we have
\begin{align}\label{equ:expt-exp-lambda-v-diff}
    & \expt\left[\exp\left(\lambda(V_{n+2}-V_{n+1})\right)\vert \mathcal{F}'_{n-1}\right]\nonumber\\
    = & \expt\left[\exp\left(\lambda\epsilon_{n}X_{n}\right)\vert \mathcal{F}'_{n-1} \right]\nonumber\\
    = & \expt\left[\left(\exp\left(\lambda X_{n}\right)\right)^{\epsilon_n} \vert \mathcal{F}'_{n-1} \right]\nonumber\\
    = & \left(\expt\left[\exp\left(\lambda X_{n}\right) \vert \mathcal{F}'_{n-1} \right]\right)^{\epsilon_n}\nonumber\\
    = & \exp\left(\epsilon_n \log\expt\left[\exp(\lambda X_n)\vert \mathcal{F}'_{n-1} \right]\right)\nonumber\\
    = & \exp\left(\epsilon_n \phi(\lambda) \right)\nonumber\\
    = & \exp\left(\left(U_{n+2}-U_{n+1}\right) \phi(\lambda) \right)
\end{align}
where the third equality is due to the fact that $\epsilon_n\in\{0, 1\}$ and $\epsilon_n$ is measurable with respect to $\mathcal{F}'_{n-1}$, and the fifth equality is due to the fact that $X_n$ is independent of $\mathcal{F}'_{n-1}$. Hence, by~\eqref{equ:expt-exp-lambda-v-diff} and the fact that $V_{n+1}$ and $U_{n+2}$ are measurable with respect to $\mathcal{F}'_{n-1}$ we have
\begin{align}
    \expt\left[\exp\left(\lambda V_{n+2}-U_{n+2}\phi(\lambda)\right)\vert \mathcal{F}'_{n-1}\right]
    = \exp\left(\lambda V_{n+1}-U_{n+1}\phi(\lambda) \right),
\end{align}
i.e., $\expt\left[W_{n}^{\lambda}\vert \mathcal{F}'_{n-1} \right] = W_{n-1}^{\lambda}$, which implies that $W_{n}^{\lambda}$ is a martingale with respect to the filtration $\{\mathcal{F}'_n\}$. Hence we have for any $n$,
\begin{align}\label{equ:martingale}
    \expt\left[W_{n}^{\lambda}\right] =  \expt\left[W_{0}^{\lambda}\right] = 1
\end{align}

Next we will use this conclusion to bound $\prob\left(\tilde{\mu}_{\tau_t}^{1}(a_1) < \mu(a_1), \tau_t \ge L_t\right)$. Note that $N_{\tau_t}(a_1)=U_{t}$ and $\bar{\mu}_{\tau_t}(a_1)=V_{t}/U_{t}$ by definition. By Line~\ref{line:alg1-state1-mu-tilde} of Algorithm~\ref{alg:1}, $\tilde{\mu}_{\tau_t}^{1}(a_1)$ can also be written as
\begin{align}\label{equ:mu-tilde-1-taut}
    \tilde{\mu}_{\tau_t}^{1}(a_1) = \max\left\{p: 2 U_t \left(p-\bar{\mu}_{\tau_t}(a_1)\right)^2\le c_1^2\log \tau_t + c c_1^2 \log(\log \tau_t) \right\}
\end{align}
Since $t\ge T_1$, we have $L_t\ge M+1 \ge 3$ by the definition of $L_t$. Hence, by the definition of $\delta_{L_t}$, we have $\delta_{L_t} > 0$. If $\delta_{L_t}\le 1$, then
\begin{align}
    \prob\left(\tilde{\mu}_{\tau_t}^{1}(a_1) < \mu(a_1), \tau_t \ge L_t\right)\le 1 \le \exp\left(1-\delta_{L_t} \right) \lceil \delta_{L_t} \log(t-2) \rceil
\end{align}
for $t\ge T_1$. Then we only need to consider the case where $\delta_{L_t}>1$. We use the same ``peeling trick'' as in~\cite{garivier2011kl}: we divide $\{1,2,...,t-2\}$ of possible values for $U_t$ into slices $\{t_{n-1}+1,...,t_{n-1}\}$ of geometrically increasing size, and treat the slices individually. Let $\beta_t\coloneqq  \frac{1}{\delta_{L_t}-1}$. Let $t_0\coloneqq 0$ and for $n\in\mathbb{N}$, $t_n\coloneqq \lfloor (1+\beta_t)^n \rfloor$. Let $D_t$ be the first integer such that $t_{D_t}\ge t-2$. Hence, $D_t=\left\lceil \frac{\log(t-2)}{\log(1+\beta_t)} \right\rceil$. Define $E_n\coloneqq \{t_{n-1}<U_t\le t_n\}\cap \{\tilde{\mu}_{\tau_t}^{1}(a_1) < \mu(a_1), \tau_t \ge L_t\}$. Then by union bound, we have:
\begin{align}\label{equ:prob-tilde-mu-union}
    \prob\left(\tilde{\mu}_{\tau_t}^{1}(a_1) < \mu(a_1), \tau_t \ge L_t\right) = \prob\left(\bigcup_{n=1}^{D_t}E_n\right)\le \sum_{n=1}^{D_t} \prob\left(E_n\right)
\end{align}
Note that 
\begin{align}\label{equ:event-inclusion}
    &\left\{\tilde{\mu}_{\tau_t}^{1}(a_1) < \mu(a_1), \tau_t \ge L_t\right\} \nonumber\\
    \subseteq & \left\{\bar{\mu}_{\tau_t}(a_1) < \mu(a_1),  2 U_t (\mu(a_1) - \bar{\mu}_{\tau_t}(a_1))^2 > c_1^2\log \tau_t + c c_1^2 \log(\log \tau_t), \tau_t \ge L_t\right\}\nonumber\\
    \subseteq & \left\{\bar{\mu}_{\tau_t}(a_1) < \mu(a_1),  2 U_t (\mu(a_1) - \bar{\mu}_{\tau_t}(a_1))^2 > \delta_{L_t}\right\}
\end{align}
by~\eqref{equ:mu-tilde-1-taut}, $\tau_t \ge L_t$, and the definition of $\delta_{L_t}$. Let $m_t$ be the smallest integer such that $\frac{\delta_{L_t}}{m_t + 1}\le 2\mu(a_1)^2$. If $U_t\le m_t$ and $\bar{\mu}_{\tau_t}(a_1) < \mu(a_1)$, then
\begin{align}
    2 U_t (\mu(a_1) - \bar{\mu}_{\tau_t}(a_1))^2 \le 2 m_t (\mu(a_1) - \bar{\mu}_{\tau_t}(a_1))^2
    \le 2 m_t \mu(a_1)^2 < \delta_{L_t}
\end{align}
which implies that
\begin{align}\label{equ:event-empty}
    \left\{ U_t\le m_t, \bar{\mu}_{\tau_t}(a_1) < \mu(a_1), 2 U_t (\mu(a_1) - \bar{\mu}_{\tau_t}(a_1))^2 > \delta_{L_t} \right\} = \emptyset
\end{align}
Therefore, it follows from~\eqref{equ:event-inclusion} and~\eqref{equ:event-empty} that
\begin{align}\label{equ:event-ut-le-st}
    \left\{ U_t\le m_t, \tilde{\mu}_{\tau_t}^{1}(a_1) < \mu(a_1), \tau_t \ge L_t \right\} = \emptyset
\end{align}
Hence, $E_n=\emptyset$ for all $n$ such that $t_n\le m_t$. For $n$ such that $t_n > m_t$, let $\tilde{t}_{n-1} \coloneqq  \max\{t_{n-1}, m_t\}$. Then we have
\begin{align}\label{equ:event-en-subset1}
    E_n \subseteq & \{\tilde{t}_{n-1} < U_t \le t_n\} \cap \{\tilde{\mu}_{\tau_t}^{1}(a_1) < \mu(a_1), \tau_t \ge L_t\}\nonumber\\
    \subseteq & \{\tilde{t}_{n-1} < U_t \le t_n\} \cap \{\bar{\mu}_{\tau_t}(a_1) < \mu(a_1),  2 U_t (\mu(a_1) - \bar{\mu}_{\tau_t}(a_1))^2 > \delta_{L_t}\}
\end{align}
where the second relation follows from~\eqref{equ:event-inclusion}.
Define $z_t$ such that $0 \le z_t < \mu(a_1)$ and $2(\mu(a_1)-z_t)^2 = \frac{\delta_{L_t}}{(1+\beta_t)^n}$. Note that if $E_n$ occurs, then the definition of $z_t$ is valid since
\begin{align}
    \frac{\delta_{L_t}}{(1+\beta_t)^n}\le \frac{\delta_{L_t}}{U_t} \le \frac{\delta_{L_t}}{m_t + 1} \le 2\mu(a_1)^2
\end{align}
where the first inequality follows from $U_t\le t_n \le (1+\beta_t)^n$, the second inequality follows from $U_t \ge m_t + 1$, and the third inequality is by the definition of $m_t$.
For $U_t > \tilde{t}_{n-1}$, we have
\begin{align}\label{equ:event-en-subset2}
    E_n \cap \left\{U_t > \tilde{t}_{n-1}\right\}\subseteq & E_n \cap\left\{U_t > t_{n-1}\right\} \subseteq E_n \cap\left\{U_t > \lfloor(1+\beta_t)^{n-1}\rfloor\right\} \nonumber\\
    \subseteq & E_n \cap\left\{U_t \ge (1+\beta_t)^{n-1}\right\} \nonumber\\
    \subseteq & E_n \cap\left\{2(\mu(a_1)-z_t)^2 = \frac{\delta_{L_t}}{(1+\beta_t)^n}\ge \frac{\delta_{L_t}}{(1+\beta_t)U_t}\right\}
\end{align}
where the first relation is by definition of $\tilde{t}_{n-1}$, the second relation is by the definition of $t_{n-1}$, and the last relation uses the definition of $z_t$.
For $U_t \le t_n$, we have
\begin{align}\label{equ:event-en-subset3}
    E_n \cap \left\{U_t \le t_n\right\} \subseteq & E_n \cap\left\{U_t \le \lfloor(1+\beta_t)^{n}\rfloor\right\} \subseteq E_n \cap\left\{U_t \le (1+\beta_t)^{n}\right\} \nonumber\\
    \subseteq & E_n \cap \left\{U_t \le (1+\beta_t)^{n}\right\} \cap \left\{2 U_t (\mu(a_1) - \bar{\mu}_{\tau_t}(a_1))^2 > \delta_{L_t}\right\} \nonumber\\
    \subseteq & E_n \cap \left\{2 (\mu(a_1) - \bar{\mu}_{\tau_t}(a_1))^2 > \frac{\delta_{L_t}}{U_t} \ge \frac{\delta_{L_t}}{(1+\beta_t)^n } = 2(\mu(a_1)-z_t)^2 \right\}
\end{align}
where the first relation is by the definition of $t_n$, the third relation follows from~\eqref{equ:event-en-subset1}, and the last relation uses the definition of $z_t$.
Hence, from~\eqref{equ:event-en-subset1}~\eqref{equ:event-en-subset2}, and~\eqref{equ:event-en-subset3}, it follows that
\begin{align}\label{equ:event-en-subset4}
    E_n \subseteq & \left\{ \bar{\mu}_{\tau_t}(a_1) < \mu(a_1), 2 (\mu(a_1) - \bar{\mu}_{\tau_t}(a_1))^2 > 2 (\mu(a_1) - z_t)^2 \ge \frac{\delta_{L_t}}{(1+\beta_t)U_t} \right\}\nonumber\\
    \subseteq & \left\{ \bar{\mu}_{\tau_t}(a_1) < z_t, 2 (\mu(a_1) - z_t)^2 \ge \frac{\delta_{L_t}}{(1+\beta_t)U_t} \right\}
\end{align}
Define
\begin{align}
    \lambda_t \coloneqq  \log\left(z_t(1-\mu(a_1))\right) - \log\left(\mu(a_1)(1-z_t)\right) \le 0.
\end{align}
By Lemma~9 in~\cite{garivier2011kl}, we have
\begin{align}\label{equ:phi-lambda-le-log}
    \phi(\lambda_t)\le \log \left(1-\mu(a_1) + \mu(a_1) \exp(\lambda_t) \right)
\end{align}
Then it follows from~\eqref{equ:event-en-subset4} and~\eqref{equ:phi-lambda-le-log} that
\begin{align}\label{equ:event-en-subset5}
    E_n \subseteq & \left\{ \lambda_t \bar{\mu}_{\tau_t}(a_1) - \phi(\lambda_t) \ge \lambda_t z_t - \log \left(1-\mu(a_1) + \mu(a_1) \exp(\lambda_t) \right), 2 (\mu(a_1) - z_t)^2 \ge \frac{\delta_{L_t}}{(1+\beta_t)U_t} \right\}\nonumber\\
    = & \left\{ \lambda_t \bar{\mu}_{\tau_t}(a_1) - \phi(\lambda_t) \ge \mathrm{kl}(z_t, \mu(a_1)), 2 (\mu(a_1) - z_t)^2 \ge \frac{\delta_{L_t}}{(1+\beta_t)U_t} \right\}\nonumber\\
    \subseteq & \left\{ \lambda_t \bar{\mu}_{\tau_t}(a_1) - \phi(\lambda_t) \ge 2(\mu(a_1)-z_t)^2, 2 (\mu(a_1) - z_t)^2 \ge \frac{\delta_{L_t}}{(1+\beta_t)U_t} \right\}\nonumber\\
    \subseteq & \left\{ \lambda_t \bar{\mu}_{\tau_t}(a_1) - \phi(\lambda_t) \ge \frac{\delta_{L_t}}{(1+\beta_t)U_t} \right\}
\end{align}
where the second relation is by the definition of $\lambda_t$ and $\mathrm{kl}(\cdot,\cdot)$, and the third relation uses Pinsker's inequality such that $2(\mu(a_1)-z_t)^2\le \mathrm{kl}(z_t,\mu(a_1))$. By the relation that $\bar{\mu}_{\tau_t}(a_1)=V_{t}/U_{t}$, the definition of $W_n^{\lambda}$ in~\eqref{equ:wn-lambda}, and~\eqref{equ:event-en-subset5}, we have
\begin{align}
    E_n \subseteq \left\{ \log \left(W_{t-2}^{\lambda_t}\right) \ge \frac{\delta_{L_t}}{(1+\beta_t)} \right\} = \left\{ W_{t-2}^{\lambda_t} \ge \exp\left(\frac{\delta_{L_t}}{(1+\beta_t)}\right) \right\}
\end{align}
By Markov's inequality, we have
\begin{align}
    \prob\left(E_n\right) \le \prob\left( W_{t-2}^{\lambda_t} \ge \exp\left(\frac{\delta_{L_t}}{(1+\beta_t)}\right) \right) \le \frac{\expt\left[ W_{t-2}^{\lambda_t} \right]}{\exp\left(\frac{\delta_{L_t}}{(1+\beta_t)}\right)} = \exp\left(-\frac{\delta_{L_t}}{(1+\beta_t)}\right)
\end{align}
where the equality follows from~\eqref{equ:martingale}. Hence, from~\eqref{equ:prob-tilde-mu-union}, it follows that
\begin{align}\label{equ:prob-tilde-mu-le-dt-exp}
    \prob\left(\tilde{\mu}_{\tau_t}^{1}(a_1) < \mu(a_1), \tau_t \ge L_t\right) \le \sum_{n=1}^{D_t} \prob\left(E_n\right) \le D_t \exp\left(-\frac{\delta_{L_t}}{(1+\beta_t)}\right).
\end{align}
Recall that $\beta_t = \frac{1}{\delta_{L_t}-1}$. Then
\begin{align}\label{equ:dt-le-sth}
    D_t = \left\lceil \frac{\log(t-2)}{\log(1+\beta_t)} \right\rceil = \left\lceil \frac{\log(t-2)}{\log(1+\frac{1}{\delta_{L_t}-1})} \right\rceil \le \left\lceil \delta_{L_t} \log (t-2) \right\rceil
\end{align}
where the last inequality follows from the fact that $\log(1+\frac{1}{x-1})\ge \frac{1}{x}$ for any $x>1$. From~\eqref{equ:prob-tilde-mu-le-dt-exp} and~\eqref{equ:dt-le-sth}, it follows that
\begin{align}
    \prob\left(\tilde{\mu}_{\tau_t}^{1}(a_1) < \mu(a_1), \tau_t \ge L_t\right) \le & \left\lceil \delta_{L_t} \log (t-2) \right\rceil \exp\left(-\frac{\delta_{L_t}}{(1+\beta_t)}\right) \nonumber\\
    = & e \left\lceil \delta_{L_t} \log (t-2) \right\rceil \exp(-\delta_{L_t})
\end{align}
\end{proof}

\subsection{Proof of Lemma~\ref{lemma:sd-cb-2}}
\label{app:lemma:sd-cb-2}

By monotone convergence theorem and linearity of expectation, we have
\begin{align}\label{equ:sum-prob-subopt-pulls-state-0}
    & \expt\left[ \sum_{t=1}^{T(K,\pi)} \sum_{i=2}^{M} \mathbbm{1}\{S_t = 0, A_t = a_i\} \left[ V^*(0) - Q^*(0, a_i)\right]\right] \nonumber\\
    = & \expt\left[ \sum_{t=1}^{\infty} \mathbbm{1}\{t\le T(K,\pi)\} \sum_{i=2}^{M} \mathbbm{1}\{S_t = 0, A_t = a_i\} \left[ V^*(0) - Q^*(0, a_i)\right]\right]\nonumber\\
    = & \sum_{t=1}^{\infty} \expt\left[ \mathbbm{1}\{t\le T(K,\pi)\} \sum_{i=2}^{M} \mathbbm{1}\{S_t = 0, A_t = a_i\} \left[ V^*(0) - Q^*(0, a_i)\right]\right]\nonumber\\
    \le & \sum_{t=1}^{\infty} \expt\left[ \sum_{i=2}^{M} \mathbbm{1}\{S_t = 0, A_t = a_i\} \left[ V^*(0) - Q^*(0, a_i)\right]\right]\nonumber\\
    = &  \sum_{t=1}^{\infty} \sum_{i=2}^{M} \left[ V^*(0) - Q^*(0, a_i)\right] \expt\left[\mathbbm{1}\{S_t = 0, A_t = a_i\}\right]\nonumber\\
    = & \sum_{t=1}^{\infty} \sum_{i=2}^{M} \left[ V^*(0) - Q^*(0, a_i)\right] \prob \left(S_t = 0, A_t = a_i\right)\nonumber\\
    = & \sum_{i=2}^{M} \left[ V^*(0) - Q^*(0, a_i)\right] \sum_{t=1}^{\infty} \prob \left(S_t = 0, A_t = a_i\right)
\end{align}
where $\prob \left(S_t = 0, A_t = a_i\right)$ can be bounded by
\begin{align}
    \prob \left(S_t = 0, A_t = a_i\right) \le \prob \left(S_t = 0, A_t = a_i, \mu(a_i)\ge \tilde{\mu}_t^0(a_i)\right) + \prob\left(\mu(a_i) < \tilde{\mu}_t^0(a_i)\right).
\end{align}
Note that by definition of $\tilde{\mu}_t^0(a_i)$ we know
\begin{align}
    \tilde{\mu}_t^0(a_i) = \bar{\mu}_{t}(a_i)-\sqrt{\frac{\log t + 4 \log(\log t)}{2 N_t(a_i)}}
    = \min\left\{p: 2(\bar{\mu}_t(a_i) - p)^2 N_t(a_i) \le \log t + 4 \log(\log t) \right\}.
\end{align}
Hence by Theorem~10 in~\cite{garivier2011kl}, we have
\begin{align}\label{equ:theorem-10-ref-upper-bound}
    \prob\left(\mu(a_i) < \tilde{\mu}_t^0(a_i)\right) \le \frac{e\left \lceil \left[\log t + 4\log (\log t)\right] \log t \right \rceil}{t(\log t)^4} 
    \le  \frac{6e}{t (\log t)^2}
\end{align}
for any $t\ge T_1$. Hence, we have
\begin{align}
    \prob \left(S_t = 0, A_t = a_i\right) \le \prob \left(S_t = 0, A_t = a_i, \mu(a_i)\ge \tilde{\mu}_t^0(a_i)\right) + \frac{6e}{t (\log t)^2}
\end{align}
for any $t\ge T_1$.
Similarly, we have
\begin{align}\label{equ:prob-subopt-pulls-state-0}
    & \prob \left(S_t = 0, A_t = a_i\right)\nonumber\\
    \le & \prob \left(S_t = 0, A_t = a_i, \tilde{\mu}_t^1(a_1)\ge \mu(a_1), \mu(a_i)\ge \tilde{\mu}_t^0(a_i)\right) + \prob\left(\tilde{\mu}_t^1(a_1) < \mu(a_1)\right) + \frac{6e}{t (\log t)^2}\nonumber\\
    \le & \prob \left(S_t = 0, A_t = a_i, \tilde{\mu}_t^1(a_1)\ge \mu(a_1), \mu(a_i)\ge \tilde{\mu}_t^0(a_i)\right) + \frac{12e}{t (\log t)^2}
\end{align}
for any $t\ge T_1$.
Note that
\begin{align}
    \tilde{\mu}_t^1(a_1) = \tilde{\mu}_t^0(a_1) + 2\sqrt{\frac{\log t + 4 \log(\log t)}{2 N_t(a_1)}}
\end{align}
by definition. Hence we have
\begin{align}
    & \prob \left(S_t = 0, A_t = a_i, \tilde{\mu}_t^1(a_1)\ge \mu(a_1), \mu(a_i)\ge \tilde{\mu}_t^0(a_i)\right)\nonumber\\
    = & \prob \left(S_t = 0, A_t = a_i, \tilde{\mu}_t^0(a_1) + 2\sqrt{\frac{\log t + 4 \log(\log t)}{2 N_t(a_1)}}\ge \mu(a_1), \mu(a_i)\ge \tilde{\mu}_t^0(a_i)\right) \nonumber\\
    = & \prob \left(S_t = 0, A_t = a_i, \tilde{\mu}_t^0(a_1) + 2\sqrt{\frac{\log t + 4 \log(\log t)}{2 N_t(a_1)}}\ge \mu(a_i) + (\mu(a_1) - \mu(a_i)), \mu(a_i)\ge \tilde{\mu}_t^0(a_i)\right) \nonumber\\
    \le & \prob \left(S_t = 0, A_t = a_i, \tilde{\mu}_t^0(a_1) + 2\sqrt{\frac{\log t + 4 \log(\log t)}{2 N_t(a_1)}}\ge \tilde{\mu}_t^0(a_i) + (\mu(a_1) - \mu(a_i)) \right) \nonumber\\
    \le & \prob \left(S_t = 0, A_t = a_i, \tilde{\mu}_t^0(a_1) \ge \tilde{\mu}_t^0(a_i) + (\mu(a_1) - \mu(a_2)) - 2\sqrt{\frac{\log t + 4 \log(\log t)}{2 N_t(a_1)}} \right)
\end{align}
Then by Lemma~\ref{lemma:sd-cb-1}, for any $t\ge T_1$, we have
\begin{align}
    & \prob \left(S_t = 0, A_t = a_i, \tilde{\mu}_t^1(a_1)\ge \mu(a_1), \mu(a_i)\ge \tilde{\mu}_t^0(a_i)\right)\nonumber\\
    \le & \prob \left(S_t = 0, A_t = a_i, \tilde{\mu}_t^0(a_1) \ge \tilde{\mu}_t^0(a_i) + (\mu(a_1) - \mu(a_2)) - 2\sqrt{\frac{\log t + 4 \log(\log t)}{2 N_t(a_1)}},\right.\nonumber\\
    &~~~~ \left.N_t(a_1) > \frac{(p_{\mathrm{min}}-\eta)(t-1)}{2} \right) + \prob\left(N_t(a_1) \le \frac{(p_{\mathrm{min}}-\eta)(t-1)}{2}\right)\nonumber\\
    \le & \prob \left(S_t = 0, A_t = a_i, \tilde{\mu}_t^0(a_1) > \tilde{\mu}_t^0(a_i) + (\mu(a_1) - \mu(a_2)) - 2\sqrt{\frac{\log t + 4 \log(\log t)}{(p_{\mathrm{min}}-\eta)(t-1)}} \right)\nonumber\\ 
    & + \frac{M-1}{2\gamma^2\exp\left(2\gamma^2 c_2(t-1) - 4\gamma^2 \right)} + \frac{c_3}{c_2(t-1)\left(\log \left(c_2(t-1)\right)\right)^2} + \exp\left(-\frac{\eta^2(t-1)}{2}\right)
\end{align}
Define $T_2$ such that $T_2 \ge T_1$ and for any $t\ge T_2$, 
\begin{align}
    \mu(a_1) - \mu(a_2) \ge 2\sqrt{\frac{\log t + 4 \log(\log t)}{(p_{\mathrm{min}}-\eta)(t-1)}}
\end{align}
Then for any $t\ge T_2$, we have
\begin{align}\label{equ:prob-state-0-contradic-event}
    & \prob \left(S_t = 0, A_t = a_i, \tilde{\mu}_t^1(a_1)\ge \mu(a_1), \mu(a_i)\ge \tilde{\mu}_t^0(a_i)\right)\nonumber\\
    \le & \prob \left(S_t = 0, A_t = a_i, \tilde{\mu}_t^0(a_1) > \tilde{\mu}_t^0(a_i) \right)\nonumber\\ 
    & + \frac{M-1}{2\gamma^2\exp\left(2\gamma^2 c_2(t-1) - 4\gamma^2 \right)} + \frac{c_3}{c_2(t-1)\left(\log \left(c_2(t-1)\right)\right)^2} + \exp\left(-\frac{\eta^2(t-1)}{2}\right)\nonumber\\
    \le & \frac{M-1}{2\gamma^2\exp\left(2\gamma^2 c_2(t-1) - 4\gamma^2 \right)} + \frac{c_3}{c_2(t-1)\left(\log \left(c_2(t-1)\right)\right)^2} + \exp\left(-\frac{\eta^2(t-1)}{2}\right)
\end{align}
where the last inequality follows from $\prob \left(S_t = 0, A_t = a_i, \tilde{\mu}_t^0(a_1) > \tilde{\mu}_t^0(a_i) \right)=0$ by the ULCB algorithm. Therefore, it follows from~\eqref{equ:prob-subopt-pulls-state-0} and~\eqref{equ:prob-state-0-contradic-event} that
\begin{align}\label{equ:prob-state-0-subopt-pulls-ub-constant}
    &\sum_{t=T_2}^{\infty} \prob \left(S_t = 0, A_t = a_i\right)\nonumber\\
    \le & \sum_{t=T_2}^{\infty} \frac{M-1}{2\gamma^2\exp\left(2\gamma^2 c_2(t-1) - 4\gamma^2 \right)} + \frac{c_3}{c_2(t-1)\left(\log \left(c_2(t-1)\right)\right)^2} + \exp\left(-\frac{\eta^2(t-1)}{2}\right) + \frac{12e}{t (\log t)^2}\nonumber\\
    \le & \frac{(M-1)\exp(4\gamma^2 - 2\gamma^2c_2(T_2-2))}{4\gamma^4 c_2} + \frac{c_3}{c_2\log [c_2(T_2-2)]} + \frac{2}{\eta^2}\exp\left(-\frac{\eta^2(T_2-2)}{2}\right) + \frac{12e}{\log(T_2-1)}
\end{align}
Hence, from~\eqref{equ:sum-prob-subopt-pulls-state-0}, it follows that
\begin{align}
    & \expt\left[ \sum_{t=1}^{T(K,\pi)} \sum_{i=2}^{M} \mathbbm{1}\{S_t = 0, A_t = a_i\} \left[ V^*(0) - Q^*(0, a_i)\right]\right] \nonumber\\
    \le & \sum_{i=2}^{M} \left[ V^*(0) - Q^*(0, a_i)\right] \left[ T_2 + \sum_{t=T_2}^{\infty} \prob \left(S_t = 0, A_t = a_i\right)\right]\nonumber\\
    \le & c_4 \sum_{i=2}^{M} \left[ V^*(0) - Q^*(0, a_i)\right]
\end{align}
where the last inequality follows from~\eqref{equ:prob-state-0-subopt-pulls-ub-constant} and
\begin{align}
    c_4 \coloneqq  & T_2 + \frac{(M-1)\exp(4\gamma^2 - 2\gamma^2c_2(T_2-2))}{4\gamma^4 c_2} + \frac{c_3}{c_2\log [c_2(T_2-2)]} + \frac{2}{\eta^2}\exp\left(-\frac{\eta^2(T_2-2)}{2}\right)\nonumber\\
    & + \frac{12e}{\log(T_2-1)}
\end{align}

\subsection{Proof of Lemma~\ref{lemma:sd-cb-3}}
\label{app:lemma:sd-cb-3}

By monotone convergence theorem and linearity of expectation, we have
\begin{align}\label{equ:proof-lemma-sd-cb-3-first}
    & \expt\left[ \sum_{t=1}^{T(K,\pi)} \sum_{i=2}^{M} \mathbbm{1}\{S_t = 1, A_t = a_i\} \left[ V^*(1) - Q^*(1, a_i)\right]\right] \nonumber\\
    = & \expt\left[ \sum_{t=1}^{\infty} \mathbbm{1}\{t\le T(K,\pi)\} \sum_{i=2}^{M} \mathbbm{1}\{S_t = 1, A_t = a_i\} \left[ V^*(1) - Q^*(1, a_i)\right]\right]\nonumber\\
    = & \expt\left[ \sum_{i=2}^{M} \sum_{t=1}^{\infty} \mathbbm{1}\{S_t = 1, A_t = a_i, t\le T(K,\pi)\} \left[ V^*(1) - Q^*(1, a_i)\right]\right]\nonumber\\
    = & \sum_{i=2}^{M} \left[ V^*(1) - Q^*(1, a_i)\right] \expt\left[ \sum_{t=1}^{\infty} \mathbbm{1}\{S_t = 1, A_t = a_i, t\le T(K,\pi)\} \right]\nonumber\\
    \le & \sum_{i=2}^{M} \left[ V^*(1) - Q^*(1, a_i)\right] \left( M + \expt\left[ \sum_{t=M + 1}^{\infty} \mathbbm{1}\{S_t = 1, A_t = a_i, t\le T(K,\pi)\} \right]\right).
\end{align}
Let $\epsilon>0$. Let $B(t)\coloneqq  \frac{1+\epsilon}{2(\mu(a_1) - \mu(a_i))^2}\left[\log t + 4\log (\log t) \right]$. Then we have
\begin{align}\label{equ:proof-lemma-sd-cb-3-second}
    & \expt\left[ \sum_{t=M+1}^{\infty} \mathbbm{1}\{S_t = 1, A_t = a_i, t\le T(K,\pi)\} \right] \nonumber\\
    = & \expt\left[ B(T(K,\pi)) + \sum_{t=M+1}^{\infty} \mathbbm{1}\{S_t = 1, A_t = a_i, t\le T(K,\pi), N^1_t(a_i)\ge B(T(K,\pi))\} \right]\nonumber\\
    \le & \expt\left[ B(T(K,\pi)) + \sum_{t=M+1}^{\infty} \mathbbm{1}\{S_t = 1, A_t = a_i, t\le T(K,\pi), N_t(a_i)\ge B(T(K,\pi))\} \right]\nonumber\\
    \le & \expt\left[ B(T(K,\pi)) + \sum_{t=M+1}^{\infty} \mathbbm{1}\{S_t = 1, A_t = a_i, N_t(a_i)\ge B(t)\} \right] \nonumber\\
    = & \expt[B(T(K,\pi))] + \sum_{t=M+1}^{\infty} \prob\left(S_t = 1, A_t = a_i, N_t(a_i)\ge B(t)\right)
\end{align}
where the first inequality is due to the fact that $N_t(a_i)\ge N^1_t(a_i)$, the second inequality is due to the fact that $B(t)$ is increasing in $t$, and the last equality is by monotone convergence theorem. For the second term, we have
\begin{align}\label{equ:proof-lemma-sd-cb-3-summation-prob}
    & \sum_{t=M+1}^{\infty}\prob\left(S_t = 1, A_t = a_i, N_t(a_i)\ge B(t)\right)\nonumber\\
    \le & \sum_{t=M+1}^{\infty}\prob\left(\tilde{\mu}^1_t(a_1) < \mu(a_1)\right) + \sum_{t=M+1}^{\infty}\prob\left(S_t = 1, A_t = a_i, N_t(a_i)\ge B(t), \tilde{\mu}^1_t(a_1) \ge \mu(a_1)\right)\nonumber\\
    \le & \sum_{t=M+1}^{\infty} \frac{6e}{t(\log t)^2} + \sum_{t=M+1}^{\infty}\prob\left(S_t = 1, A_t = a_i, N_t(a_i)\ge B(t), \tilde{\mu}^1_t(a_1) \ge \mu(a_1)\right)\nonumber\\
    \le & \sum_{t=M+1}^{\infty} \frac{6e}{t(\log t)^2} + \sum_{t=M+1}^{\infty}\prob\left(S_t = 1, A_t = a_i, N_t(a_i)\ge B(t), \tilde{\mu}^1_t(a_i)\ge \tilde{\mu}^1_t(a_1) \ge \mu(a_1)\right)\nonumber\\
    \le & \frac{6e}{\log M} + \sum_{t=M+1}^{\infty}\prob\left(A_t = a_i, N_t(a_i)\ge B(t), \tilde{\mu}^1_t(a_i) \ge \mu(a_1)\right)
\end{align}
where the second inequality is by Theorem~10 in~\cite{garivier2011kl}, similar to~\eqref{equ:theorem-10-ref-upper-bound} in the proof of Lemma~\ref{lemma:sd-cb-2}. The third inequality holds since the ULCB algorithm pulls arm $a_i$ in state $1$ if and only if $\tilde{\mu}^1_t(a_i)\ge \tilde{\mu}^1_t(a_1)$ when $t\ge M+1$. Consider that the event $\{A_t = a_i, N_t(a_i)\ge B(t), \tilde{\mu}^1_t(a_i) \ge \mu(a_1)\}$ holds. Then we have
\begin{align}\label{equ:mu-a1-proof-lemma-sd-cb-3}
    \mu(a_1) \le \tilde{\mu}^1_t(a_i) = \bar{\mu}_t(a_i) + \sqrt{\frac{\log t + 4\log(\log t)}{2 N_t(a_i)}} \le \bar{\mu}_t(a_i) + \frac{\mu(a_1)-\mu(a_i)}{\sqrt{1+\epsilon}}
\end{align}
where the last inequality is by $N_t(a_i)\ge B(t)$ and the definition of $B(t)$. Define
\[r_{\epsilon}(a_i)\in(\mu(a_i), \mu(a_1))\]
such that
\begin{align}\label{equ:rai-proof-lemma-sd-cb-3}
    \mu(a_1) - r_{\epsilon}(a_i) = \frac{\mu(a_1)-\mu(a_i)}{\sqrt{1+\epsilon}}.
\end{align}
Then it follows from~\eqref{equ:mu-a1-proof-lemma-sd-cb-3} and~\eqref{equ:rai-proof-lemma-sd-cb-3} that
\begin{align}
    \mu(a_1) \le \bar{\mu}_t(a_i) + \mu(a_1) - r_{\epsilon}(a_i)
\end{align}
which implies that $\bar{\mu}_t(a_i) \ge r_{\epsilon}(a_i)$. Hence, the second term in~\eqref{equ:proof-lemma-sd-cb-3-summation-prob} can be bounded by
\begin{align}\label{equ:proof-lemma-sd-cb-3-sum-prob-2}
    & \sum_{t=M+1}^{\infty}\prob\left(A_t = a_i, N_t(a_i)\ge B(t), \tilde{\mu}^1_t(a_i) \ge \mu(a_1)\right)\nonumber\\
    \le & \sum_{t=M+1}^{\infty}\prob\left(A_t = a_i, \bar{\mu}_t(a_i) \ge r_{\epsilon}(a_i)\right)\nonumber\\
    = & \sum_{t=M+1}^{\infty}\prob\left(A_t = a_i, \bar{\mu}_t(a_i) - \mu(a_i) \ge r_{\epsilon}(a_i) - \mu(a_i)\right)\nonumber\\
    = & \sum_{t=M+1}^{\infty} \sum_{n=1}^{t-1} \prob\left(A_t = a_i, N_t(a_i)=n, \frac{1}{n}\sum_{s=1}^n R_s(a_i) - \mu(a_i) \ge r_{\epsilon}(a_i) - \mu(a_i)\right)\nonumber\\
    \le & \sum_{n=1}^{\infty} \sum_{t=n+1}^{\infty} \prob\left(A_t = a_i, N_t(a_i)=n, \frac{1}{n}\sum_{s=1}^n R_s(a_i) - \mu(a_i) \ge r_{\epsilon}(a_i) - \mu(a_i)\right)\nonumber\\
    \le & \sum_{n=1}^{\infty} \prob\left(\frac{1}{n}\sum_{s=1}^n R_s(a_i) - \mu(a_i) \ge r_{\epsilon}(a_i) - \mu(a_i)\right)
\end{align}
where the second equality is by law of total probability, where $\{R_s(a_i)\}_{s=1}^n$ are i.i.d. Bernoulli rewards of pulling arm $a_i$, and the last inequality is due to the fact that $\{A_t=a_i, N_t(a_i)=n\}_{t=n+1}^{\infty}$ are mutually exclusive and the countable additivity of probability measure. Then by Hoeffding inequality and~\eqref{equ:proof-lemma-sd-cb-3-sum-prob-2}, we have
\begin{align}\label{equ:proof-lemma-sd-cb-3-concentration}
    & \sum_{t=M+1}^{\infty}\prob\left(A_t = a_i, N_t(a_i)\ge B(t), \tilde{\mu}^1_t(a_i) \ge \mu(a_1)\right)\nonumber\\
    \le & \sum_{n=1}^{\infty} \exp\left(-2n \left(r_{\epsilon}(a_i) - \mu(a_i)\right)^2\right)\nonumber\\
    \le & \frac{1}{2\left(r_{\epsilon}(a_i) - \mu(a_i)\right)^2}
\end{align}
Therefore, from~\eqref{equ:proof-lemma-sd-cb-3-first},~\eqref{equ:proof-lemma-sd-cb-3-second},~\eqref{equ:proof-lemma-sd-cb-3-summation-prob}, and~\eqref{equ:proof-lemma-sd-cb-3-concentration}, it follows that
\begin{align}\label{equ:proof-lemma-sd-cb-3-final-1}
    & \expt\left[ \sum_{t=1}^{T(K,\pi)} \sum_{i=2}^{M} \mathbbm{1}\{S_t = 1, A_t = a_i\} \left[ V^*(1) - Q^*(1, a_i)\right]\right] \nonumber\\
    \le & \sum_{i=2}^{M} \left[ V^*(1) - Q^*(1, a_i)\right] \left[ M + \expt[B(T(K,\pi))] + \frac{6e}{\log M} + \frac{1}{2\left(r_{\epsilon}(a_i) - \mu(a_i)\right)^2} \right]
\end{align}
where
\begin{align}\label{equ:proof-lemma-sd-cb-3-final-2}
    \expt[B(T(K,\pi))] = & \expt\left[\frac{1+\epsilon}{2(\mu(a_1) - \mu(a_i))^2}\left[\log T(K,\pi) + 4\log (\log T(K,\pi)) \right]\right]\nonumber\\
    \le & \frac{1+\epsilon}{2(\mu(a_1) - \mu(a_i))^2}\biggl[\log \expt\left[T(K,\pi)\right] + 4\log\left(\log\expt\left[T(K,\pi)\right]\right)\biggr]\nonumber\\
    \le & \frac{1+\epsilon}{2(\mu(a_1) - \mu(a_i))^2}\biggl[\log c_5K + 4\log\left(\log (c_5 K) \right)\biggr]
\end{align}
where the first inequality is by Jensen's inequality, and the second inequality is by Lemma~\ref{lemma:sd-cb-bound-TKpi} (which is presented after this proof). Then it follows from~\eqref{equ:proof-lemma-sd-cb-3-final-1} and~\eqref{equ:proof-lemma-sd-cb-3-final-2} that
\begin{align}
    & \expt\left[ \sum_{t=1}^{T(K,\pi)} \sum_{i=2}^{M} \mathbbm{1}\{S_t = 1, A_t = a_i\} \left[ V^*(1) - Q^*(1, a_i)\right]\right] \nonumber\\
    \le & \sum_{i\neq 1} \frac{1+\epsilon}{2(\mu(a_1)-\mu(a_i))^2} \left(V^*(1) - Q^*(1,a_i)\right) \log K + o(\log K)
\end{align}

\begin{lemma}\label{lemma:sd-cb-bound-TKpi}
Let Assumption~\ref{assum:1} hold. Then for any policy $\pi\in\Pi$, we have
\begin{align}
    \expt[T(K,\pi)] = \sum_{k=1}^K \expt\left[I_k(\pi,S_{k,1},\phi_{k,1})\right] \le \sum_{k=1}^K \expt\left[I_k(\pi^*, S_{k,1})\right] \le c_5 K
\end{align}
\end{lemma}
\begin{proof}
Define a genie-aided (model-based) policy $\pi_*'$ that maximizes the expected number of steps in one episode, i.e., for any $s\in\{0,1\}$,
$$
\pi_*' \coloneqq  \operatorname{argmax}_{\pi} \expt\left[I(\pi,s)\right]
$$
where $\pi$ is taken over all policies, $I(\pi,s)$ is the number of steps in one episode given initial state $s$. We omit the subscript $k$ in $I(\pi,s)$ since the distribution does not depend on the episode number $k$. Then we have
\begin{align}
    \label{equ:lemma-sd-cb-bound-TKpi-1}
    \expt\left[I(\pi_*',1)|A_{k,1}=a\right] =& 1 + \mu_{a} (1-q(1,1)) \expt\left[I(\pi_*',1)\right] + (1-\mu_{a})(1-q(1,0)) \expt\left[I(\pi_*',0)\right]\\
    \label{equ:lemma-sd-cb-bound-TKpi-2}
    \expt\left[I(\pi_*',0)|A_{k,1}=a\right] =& 1 + \mu_{a}(1-q(0,1)) \expt\left[I(\pi_*',1)\right] + (1-\mu_{a})(1-q(0,0))\expt\left[I(\pi_*',0)\right]
\end{align}
Then
\begin{align}
    &\expt\left[I(\pi_*',1)|A_{k,1}=a\right] - \expt\left[I(\pi_*',0)|A_{k,1}=a\right]\nonumber\\
    =& \mu_{a} (q(0,1)-q(1,1)) \expt\left[I(\pi_*',1)\right] +(1-\mu_{a}) (q(0,0)-q(1,0)) \expt\left[I(\pi_*',0)\right] \ge 0
\end{align}
where the inequality follows from Assumption~\ref{assum:1}. Hence, we have
\begin{align}\label{equ:lemma-sd-cb-bound-TKpi-3}
    & \expt\left[I(\pi_*',1)\right] - \expt\left[I(\pi_*',0)\right]\nonumber\\
    = & \max_{a} \expt\left[I(\pi_*',1)|A_{k,1}=a\right] - \max_{a}\expt\left[I(\pi_*',0)|A_{k,1}=a\right]\nonumber\\
    = & \max_{a} \expt\left[I(\pi_*',1)|A_{k,1}=a\right] - \expt\left[I(\pi_*',0)|A_{k,1}=a'\right]\nonumber\\
    \ge & \expt\left[I(\pi_*',1)|A_{k,1}=a'\right] - \expt\left[I(\pi_*',0)|A_{k,1}=a'\right]\nonumber\\
    \ge & 0
\end{align}
where $a'\coloneqq \mathop{\mathrm{argmax}}_a \expt\left[I(\pi_*',0)|A_{k,1}=a\right]$. Hence, from~\eqref{equ:lemma-sd-cb-bound-TKpi-1} and~\eqref{equ:lemma-sd-cb-bound-TKpi-2}, it follows that
\begin{equation}
    \begin{aligned}
    & \expt\left[I(\pi_*',1)|A_{k,1}=a_1\right] - \expt\left[I(\pi_*',1)|A_{k,1}=a_i\right]\\
    = & (\mu_{a_1}-\mu_{a_i})\biggl[(1-q(1,1))\expt\left[I(\pi_*',1)\right]-(1-q(1,0))\expt\left[I(\pi_*',0)\right]\biggr] \\
    \ge & 0,
    \end{aligned}
\end{equation}
where the inequality follows from~\eqref{equ:lemma-sd-cb-bound-TKpi-3} and Assumption~~\ref{assum:1}. Similarly,
\begin{equation}
    \begin{aligned}
    & \expt\left[I(\pi_*',0)|A_{k,1}=a_1\right] - \expt\left[I(\pi_*',0)|A_{k,1}=a_i\right]\\
    = & (\mu_{a_1}-\mu_{a_i}) \biggl[(1-q(0,1))\expt\left[I(\pi_*',1)\right]-(1-q(0,0))\expt\left[I(\pi_*',0)\right]\biggr]\\
    \ge & 0.
    \end{aligned}
\end{equation}
Therefore, the policy $\pi_*'$ is always pulling $a_1$, i.e., $\pi_*'=\pi^*$, which implies that
\begin{align}
    \expt[T(K,\pi)] = \sum_{k=1}^K \expt\left[I_k(\pi,S_{k,1},\phi_{k,1})\right] \le \sum_{k=1}^K \expt\left[I_k(\pi^*, S_{k,1})\right] \le \sum_{k=1}^K \expt\left[I(\pi^*, 1)\right] \le c_5 K
\end{align}
\end{proof}

\subsection{Proof of Theorem~\ref{theorem:2}: upper bound for KL-ULCB}
\label{app:theorem2}

The proof idea is similar to that of Theorem~\ref{theorem:1}. The proof is based on the regret decomposition in~\eqref{equ:regret-decomp}. We will first show that the number of optimal pulls scales linearly with $t$ with high probability. Then, we will bound the expected regrets induced in state $0$ and state $1$ respectively. We first prove a lemma similar to Lemma~\ref{lemma:sd-cb-1} as follows:
\begin{lemma}\label{lemma:sd-kl-cb-1}
Let all the assumptions in Theorem~\ref{theorem:2} hold. Consider the KL-ULCB algorithm with $c_0=c_1=1$, and $c=4$. Let $p_{\min}\coloneqq \mu(a_M)\min\left\{1-q(0,1), 1-q(1,1)\right\}$. Let $\eta\in(0, p_{\min})$ be a constant. Let $\gamma' >0$ be a constant and $r_{\gamma'}\in(\mu(a_2), \mu(a_1))$ such that $\mathrm{kl}(r_{\gamma'}, \mu(a_1))=\frac{\mathrm{kl}(\mu(a_2), \mu(a_1))}{1+\gamma'}$. Define $T_3$ such that for any $t\ge T_3$, 
    $$\frac{(p_{\min}-\eta)(t-1)}{2(M-1)}\ge 2,~\mathrm{and}~ \frac{\log t + 4 \log \left(\log t \right)}{\frac{(p_{\min} - \eta)(t-1)}{2(M-1)} - 1} \le \frac{\mathrm{kl}(\mu(a_2), \mu(a_1))}{1+\gamma'}.$$
    Then for any $t\ge T_3$
    \begin{align}
        & \prob\left(N_{t}(a_1)\le \frac{(p_{\min}-\eta)(t-1)}{2}\right) \le \prob\left(N_{t}^1(a_1)\le \frac{(p_{\min}-\eta)(t-1)}{2}\right) \le \exp\left(-\frac{\eta^2(t-1)}{2}\right)\nonumber\\
        & + \frac{M-1}{\mathrm{kl}\left(r_{\gamma'}, \mu(a_2)\right) \exp\left(\mathrm{kl}\left(r_{\gamma'}, \mu(a_2)\right) \left[\frac{(p_{\min} - \eta)(t-1)}{2(M-1)} - 2\right]\right)} + \frac{c_3}{c_2(t-1)\left[\log \left(c_2(t-1)\right)\right]^2}
    \end{align}
    where $c_2\coloneqq  \frac{p_{\min} - \eta}{2(M-1)}$ and $c_3\coloneqq \frac{4+e}{\log 2}\log\frac{2}{c_2} + \frac{e}{\log 2}$ are constants.
\end{lemma}
\begin{proof}
Let $t\ge T_3$.
Using the same argument as in the proof of Lemma~\ref{lemma:sd-cb-1}, we have
\begin{align}\label{equ:sd-kl-cb-prob-nta1}
    &\prob\left(N_{t}(a_1)\le \frac{(p_{\min}-\eta)(t-1)}{2}\right) \le \prob\left(N_{t}^{1}(a_1)\le \frac{(p_{\min}-\eta)(t-1)}{2}\right) \nonumber \\
    \le & \prob\left(\tilde{\mu}_{\tau_t}^{1}(a_j) \ge \mu(a_1), N_{\tau_t}(a_j) \ge \left\lceil{\frac{(p_{\min} - \eta)(t-1)}{2(M-1)}}\right\rceil - 1 \right) \nonumber\\
    & + \prob\left(\tilde{\mu}_{\tau_t}^{1}(a_1) < \mu(a_1), \tau_t \ge L_t\right) + \exp\left(-\frac{\eta^2(t-1)}{2}\right)
\end{align}
where $N_{t}^{1}(a)$ denotes the number of times arm $a\in\mathcal{A}$ was pulled in state $1$ before time step $t$, $j\in\operatorname{\mathrm{argmax}}_{i\in\{2,...,M\}} N_{t}^{1}(a_i)$, $\tau_t<t$ denotes the time step when $a_j$ is pulled in state $1$ for the $\left\lceil{\frac{(p_{\min} - \eta)(t-1)}{2(M-1)}}\right\rceil$-th time, and $L_t\coloneqq \frac{(p_{\min} - \eta)(t-1)}{2(M-1)} + (M-1)$. The second term in~\eqref{equ:sd-kl-cb-prob-nta1}, $\prob\left(\tilde{\mu}_{\tau_t}^{1}(a_1) < \mu(a_1), \tau_t \ge L_t\right)$, can be bounded by
\begin{align}\label{equ:sd-kl-cb-prob-tilde-mu}
    \prob\left(\tilde{\mu}_{\tau_t}^{1}(a_1) < \mu(a_1), \tau_t \ge L_t\right) \le e \left\lceil \delta_{L_t} \log (t-2) \right\rceil \exp(-\delta_{L_t})
\end{align}
where $\delta_{L_t}\coloneqq  \log L_t + 4\log(\log L_t)$. The proof of~\eqref{equ:sd-kl-cb-prob-tilde-mu} is omitted since it can be proved the same way as Lemma~\ref{lemma:sd-cb-prob-tilde-mu} by just replacing the Euclidean distance in the proof of Lemma~\ref{lemma:sd-cb-prob-tilde-mu} with KL divergence. Hence,
we have
\begin{align}\label{equ:sd-kl-cb-prob-ub-second-term}
    & \prob\left(\tilde{\mu}_{\tau_t}^{1}(a_1) < \mu(a_1), \tau_t \ge L_t\right) \le \frac{c_3}{c_2(t-1)\left[\log \left(c_2(t-1)\right)\right]^2}
\end{align}
where $c_2\coloneqq  \frac{p_{\min} - \eta}{2(M-1)}$ and $c_3\coloneqq \frac{4+e}{\log 2}\log\frac{2}{c_2} + \frac{e}{\log 2}$. Consider that the event
\[\left\{\tilde{\mu}_{\tau_t}^{1}(a_j)\ge \mu(a_1), N_{\tau_t}(a_j) \ge \left\lceil{\frac{(p_{\min} - \eta)(t-1)}{2(M-1)}}\right\rceil - 1\right\}\]
holds. Define $\mathrm{kl}^+(x,y)\coloneqq \mathrm{kl}(x,y)\mathbbm{1}\{x<y\}$. Then we have
\begin{align}
    \mathrm{kl}^+(\bar{\mu}_{\tau_t}(a_j),\mu(a_1))\le &  \mathrm{kl}^+(\bar{\mu}_{\tau_t}(a_j),\tilde{\mu}_{\tau_t}^{1}(a_j))=\frac{\log \tau_t + 4 \log \left(\log \tau_t \right)}{N_{\tau_t}(a_j)} \le \frac{\log t + 4 \log \left(\log t \right)}{\frac{(p_{\min} - \eta)(t-1)}{2(M-1)} - 1}\nonumber\\
    \le & \frac{\mathrm{kl}(\mu(a_2), \mu(a_1))}{1+\gamma'}
\end{align}
where the last inequality is by $t\ge T_3$ and the definition of $T_3$.
Hence, for the first term in~\eqref{equ:sd-kl-cb-prob-nta1}, we have
\begin{align}\label{equ:sd-kl-cb-prob-tilde-mu-1-taut}
    & \prob\left(\tilde{\mu}_{\tau_t}^{1}(a_j) \ge \mu(a_1), N_{\tau_t}(a_j) \ge \left\lceil{\frac{(p_{\min} - \eta)(t-1)}{2(M-1)}}\right\rceil - 1 \right)\nonumber\\
    \le & \prob\left(\mathrm{kl}^+(\bar{\mu}_{\tau_t}(a_j),\mu(a_1))\le \frac{\mathrm{kl}(\mu(a_2), \mu(a_1))}{1+\gamma'}, N_{\tau_t}(a_j) \ge \left\lceil{\frac{(p_{\min} - \eta)(t-1)}{2(M-1)}}\right\rceil - 1 \right)\nonumber\\
    \le & \sum_{i=2}^{M} \prob\left(\mathrm{kl}^+(\bar{\mu}_{\tau_t}(a_i),\mu(a_1))\le \frac{\mathrm{kl}(\mu(a_2), \mu(a_1))}{1+\gamma'}, N_{\tau_t}(a_i) \ge \left\lceil{\frac{(p_{\min} - \eta)(t-1)}{2(M-1)}}\right\rceil - 1 \right)\nonumber\\
    \le & \sum_{i=2}^{M} \sum_{n=\left\lceil{\frac{(p_{\min} - \eta)(t-1)}{2(M-1)}}\right\rceil - 1}^{t-1} \prob\left(\mathrm{kl}^+\left(\frac{1}{n}\sum_{s=1}^n R_s(a_i),\mu(a_1)\right)\le \frac{\mathrm{kl}(\mu(a_2), \mu(a_1))}{1+\gamma'} \right)
\end{align}
where the second inequality is by the union bound over all possible $j$, and the last inequality is by the union bound over all possible number of pulls of arm $a_i$, where $\{R_s(a_i)\}_{s=1}^n$ are $n$ i.i.d. Bernoulli rewards of pulling arm $a_i$. Consider that the event $\{\mathrm{kl}^+(\frac{1}{n}\sum_{s=1}^n R_s(a_i),\mu(a_1))\le \frac{\mathrm{kl}(\mu(a_2), \mu(a_1))}{1+\gamma'}\}$ holds. By the definition of $r_{\gamma'}$, we have $\mathrm{kl}^+(\frac{1}{n}\sum_{s=1}^n R_s(a_i),\mu(a_1))\le \mathrm{kl}(r_{\gamma'}, \mu(a_1))$, which implies that $\frac{1}{n}\sum_{s=1}^n R_s(a_i) \ge r_{\gamma'}$. Hence, we have
\begin{align}\label{equ:sd-kl-cb-kl-ineq}
    \mathrm{kl}\left(\frac{1}{n}\sum_{s=1}^n R_s(a_i), \mu(a_i)\right) \ge \mathrm{kl}\left(r_{\gamma'}, \mu(a_i)\right)\ge \mathrm{kl}\left(r_{\gamma'}, \mu(a_2)\right)
\end{align}
since $\mu(a_i)\le \mu(a_2) < r_{\gamma'}\le \frac{1}{n}\sum_{s=1}^n R_s(a_i)$. Hence, from~\eqref{equ:sd-kl-cb-prob-tilde-mu-1-taut} and~\eqref{equ:sd-kl-cb-kl-ineq}, it follows that
\begin{align}\label{equ:sd-kl-cb-prob-ub-first-term}
    & \prob\left(\tilde{\mu}_{\tau_t}^{1}(a_j) \ge \mu(a_1), N_{\tau_t}(a_j) \ge \left\lceil{\frac{(p_{\min} - \eta)(t-1)}{2(M-1)}}\right\rceil - 1 \right)\nonumber\\
    \le & \sum_{i=2}^{M} \sum_{n=\left\lceil{\frac{(p_{\min} - \eta)(t-1)}{2(M-1)}}\right\rceil - 1}^{t-1} \prob\left(\mathrm{kl}\left(\frac{1}{n}\sum_{s=1}^n R_s(a_i), \mu(a_i)\right) \ge \mathrm{kl}\left(r_{\gamma'}, \mu(a_2)\right), \frac{1}{n}\sum_{s=1}^n R_s(a_i) > \mu(a_i) \right)\nonumber\\
    \le & (M-1) \sum_{n=\left\lceil{\frac{(p_{\min} - \eta)(t-1)}{2(M-1)}}\right\rceil - 1}^{t-1} \exp\left(-n \mathrm{kl}\left(r_{\gamma'}, \mu(a_2)\right) \right)\nonumber\\
    \le & \frac{M-1}{\mathrm{kl}\left(r_{\gamma'}, \mu(a_2)\right) \exp\left(\mathrm{kl}\left(r_{\gamma'}, \mu(a_2)\right) \left[\frac{(p_{\min} - \eta)(t-1)}{2(M-1)} - 2\right]\right)}
\end{align}
where the second inequality uses the concentration inequality for KL divergence in~\cite{mardia2020concentration}, and the last inequality is by integration. From~\eqref{equ:sd-kl-cb-prob-nta1},~\eqref{equ:sd-kl-cb-prob-ub-second-term}, and~\eqref{equ:sd-kl-cb-prob-ub-first-term}, Lemma~\ref{lemma:sd-kl-cb-1} is proved.
\end{proof}

Lemma~\ref{lemma:sd-kl-cb-1} shows that when $t$ is large enough, the number of optimal pulls scales linearly with $t$ with high probability. This plays an important role in the proof of Lemma~\ref{lemma:sd-kl-cb-2}, which bounds the expected regret induced by pulling suboptimal arms in state $0$ by a constant.

\begin{lemma}\label{lemma:sd-kl-cb-2}
    Let all the assumptions in Theorem~\ref{theorem:2} hold. Consider the KL-ULCB algorithm with $c_0=c_1=1$, and $c=4$. The regret induced in state $0$ can be bounded by
    \begin{align}
        \expt\left[ \sum_{t=1}^{T(K,\pi)} \sum_{i=2}^{M} \mathbbm{1}\{S_t = 0, A_t = a_i\} \left[ V^*(0) - Q^*(0, a_i)\right]\right] \le \sum_{i=2}^{M} c_{6,i} \left[ V^*(0) - Q^*(0, a_i)\right]
    \end{align}
    where $c_{6,i}$ are constants which depend only on $M$, $\mu(a_1)$, $\mu(a_2)$, $\mu(a_i)$, $p_{\mathrm{min}}$, $\eta$, and $\gamma'$.
\end{lemma}
\begin{proof}
Using the same argument as in the proof of Lemma~\ref{lemma:sd-cb-2}, we have
\begin{align}\label{equ:sd-kl-cb-lemma-2-beginning}
    & \expt\left[ \sum_{t=1}^{T(K,\pi)} \sum_{i=2}^{M} \mathbbm{1}\{S_t = 0, A_t = a_i\} \left[ V^*(0) - Q^*(0, a_i)\right]\right] \nonumber\\
    \le & \sum_{i=2}^{M} \left[ V^*(0) - Q^*(0, a_i)\right] \sum_{t=1}^{\infty} \prob \left(S_t = 0, A_t = a_i\right)
\end{align}
where $\prob \left(S_t = 0, A_t = a_i\right)$ can be bounded by
\begin{align}
    & \prob \left(S_t = 0, A_t = a_i\right) \le \prob \left(S_t = 0, A_t = a_i, \mu(a_i)\ge \tilde{\mu}_t^0(a_i)\right) + \prob\left(\mu(a_i) < \tilde{\mu}_t^0(a_i)\right)\nonumber\\
    \le & \prob \left(S_t = 0, A_t = a_i, \mu(a_i)\ge \tilde{\mu}_t^0(a_i)\right) + \frac{6e}{t\left(\log t\right)^2}
\end{align}
for any $t\ge T_3$, where the inequality is by Theorem~10 in~\cite{garivier2011kl}. Note that $S_t=0, A_t=a_i$ implies that $\tilde{\mu}_t^0(a_i) \ge \tilde{\mu}_t^0(a_1)$ by the KL-ULCB algorithm. Hence, for any $t\ge T_3$, we have
\begin{align}\label{equ:sd-kl-cb-lemma-2-prob-state-0}
    & \prob \left(S_t = 0, A_t = a_i\right) \nonumber\\
    \le & \prob \left(\mu(a_i)\ge \tilde{\mu}_t^0(a_i), \tilde{\mu}_t^0(a_i) \ge \tilde{\mu}_t^0(a_1)\right) + \frac{6e}{t\left(\log t\right)^2}\nonumber\\
    \le & \prob \left(\mu(a_i) \ge \tilde{\mu}_t^0(a_1)\right) + \frac{6e}{t\left(\log t\right)^2}\nonumber\\
    \le & \prob \left(\mu(a_i) \ge \tilde{\mu}_t^0(a_1), N_{t}(a_1) > \frac{(p_{\min}-\eta)(t-1)}{2}\right)  + \prob\left(N_{t}(a_1)\le \frac{(p_{\min}-\eta)(t-1)}{2}\right) + \frac{6e}{t\left(\log t\right)^2}\nonumber\\
    \le & \prob \left(\mu(a_i) \ge \tilde{\mu}_t^0(a_1), N_{t}(a_1) > \frac{(p_{\min}-\eta)(t-1)}{2}\right) + \frac{6e}{t\left(\log t\right)^2} + \exp\left(-\frac{\eta^2(t-1)}{2}\right)\nonumber\\
    & + \frac{M-1}{\mathrm{kl}\left(r_{\gamma'}, \mu(a_2)\right) \exp\left(\mathrm{kl}\left(r_{\gamma'}, \mu(a_2)\right) \left[\frac{(p_{\min} - \eta)(t-1)}{2(M-1)} - 2\right]\right)} + \frac{c_3}{c_2(t-1)\left[\log \left(c_2(t-1)\right)\right]^2}
\end{align}
where the last inequality uses Lemma~\ref{lemma:sd-kl-cb-1}. Consider that the event
\[
    \left\{\mu(a_i) \ge \tilde{\mu}_t^0(a_1), N_{t}(a_1) > \frac{(p_{\min}-\eta)(t-1)}{2}\right\}
\]
holds. Define $\mathrm{kl}^-(x,y)\coloneqq \mathrm{kl}(x,y)\mathbbm{1}\{x>y\}$. Hence we have
\begin{align}
    \mathrm{kl}^-(\bar{\mu}_t(a_1),\mu(a_i)) \le \mathrm{kl}(\bar{\mu}_t(a_1), \tilde{\mu}_t^0(a_1))
    \le \frac{\log t + 4\log (\log t)}{N_t(a_1)} \le \frac{\log t + 4\log (\log t)}{\frac{(p_{\min}-\eta)(t-1)}{2}}.
\end{align}
Define $T_4$ such that $T_4\ge T_3$ and for any $t\ge T_4$,
\begin{align}
    \frac{\log t + 4\log (\log t)}{\frac{(p_{\min}-\eta)(t-1)}{2}} \le \frac{\mathrm{kl}(\mu(a_1), \mu(a_2))}{1+\gamma'}.
\end{align}
Then for any $t\ge T_4$, we have
\begin{align}
    \mathrm{kl}^-(\bar{\mu}_t(a_1),\mu(a_i)) \le \frac{\mathrm{kl}(\mu(a_1), \mu(a_2))}{1+\gamma'}\le \frac{\mathrm{kl}(\mu(a_1), \mu(a_i))}{1+\gamma'}.
\end{align}
Define $r'_{\gamma'}(a_i)\in(\mu(a_i), \mu(a_1))$ such that $\mathrm{kl}(r'_{\gamma'}(a_i), \mu(a_i))=\frac{\mathrm{kl}(\mu(a_1), \mu(a_i))}{1+\gamma'}$. Then we have
\begin{align}
    \mathrm{kl}^-(\bar{\mu}_t(a_1),\mu(a_i)) \le \mathrm{kl}(r'_{\gamma'}(a_i), \mu(a_i))
\end{align}
which implies that $\bar{\mu}_t(a_1) \le r'_{\gamma'}(a_i)$. Therefore, we have
\begin{align}
    \mathrm{kl}(\bar{\mu}_t(a_1),\mu(a_1)) \ge \mathrm{kl}(r'_{\gamma'}(a_i), \mu(a_1))
\end{align}
since $\bar{\mu}_t(a_1) \le r'_{\gamma'}(a_i) < \mu(a_1)$. Hence, the first term in~\eqref{equ:sd-kl-cb-lemma-2-prob-state-0} can be bounded by
\begin{align}\label{equ:sd-kl-cb-lemma-2-concentration}
    & \prob \left(\mu(a_i) \ge \tilde{\mu}_t^0(a_1), N_{t}(a_1) > \frac{(p_{\min}-\eta)(t-1)}{2}\right)\nonumber\\
    \le & \prob\left( \mathrm{kl}(\bar{\mu}_t(a_1),\mu(a_1)) \ge \mathrm{kl}(r'_{\gamma'}(a_i), \mu(a_1)), N_{t}(a_1) > \frac{(p_{\min}-\eta)(t-1)}{2}, \bar{\mu}_t(a_1) < \mu(a_1) \right)\nonumber\\
    \le & \sum_{n=\left\lceil\frac{(p_{\min}-\eta)(t-1)}{2}\right\rceil}^{t-1} \prob\left( \mathrm{kl}\left(\frac{1}{n}\sum_{s=1}^{n}R_s(a_1),\mu(a_1)\right) \ge \mathrm{kl}\left(r'_{\gamma'}(a_i), \mu(a_1)\right), \frac{1}{n}\sum_{s=1}^{n}R_s(a_1) < \mu(a_1) \right)\nonumber\\
    \le & \sum_{n=\left\lceil\frac{(p_{\min}-\eta)(t-1)}{2}\right\rceil}^{t-1} \exp\left( -n \mathrm{kl}\left(r'_{\gamma'}(a_i), \mu(a_1)\right) \right)\nonumber\\
    \le & \frac{1}{\mathrm{kl}\left(r'_{\gamma'}(a_i), \mu(a_1)\right)\exp\left(\mathrm{kl}\left(r'_{\gamma'}(a_i), \mu(a_1)\right) \left[ \frac{(p_{\min}-\eta)(t-1)}{2} - 1 \right] \right)}
\end{align}
where the second inequality is by the union bound over all possible number of pulls of arm $a_1$, where $\{R_s(a_1)\}_{s=1}^n$ are $n$ i.i.d. Bernoulli rewards of pulling arm $a_1$, the third inequality uses the concentration inequality for KL divergence in~\cite{mardia2020concentration}, and the last inequality is by integration. Then by combining~\eqref{equ:sd-kl-cb-lemma-2-prob-state-0} and~\eqref{equ:sd-kl-cb-lemma-2-concentration}, we have
\begin{align}\label{equ:sd-kl-cb-lemma-2-final}
    & \sum_{t=T_4}^{\infty} \prob \left(S_t = 0, A_t = a_i\right)\nonumber\\
    \le & \sum_{t=T_4}^{\infty} \frac{1}{\mathrm{kl}\left(r'_{\gamma'}(a_i), \mu(a_1)\right)\exp\left(\mathrm{kl}\left(r'_{\gamma'}(a_i), \mu(a_1)\right) \left[ \frac{(p_{\min}-\eta)(t-1)}{2} - 1 \right] \right)} + \frac{6e}{t\left(\log t\right)^2} \nonumber\\
    &~~~~ + \exp\left(-\frac{\eta^2(t-1)}{2}\right) + \frac{M-1}{\mathrm{kl}\left(r_{\gamma'}, \mu(a_2)\right) \exp\left(\mathrm{kl}\left(r_{\gamma'}, \mu(a_2)\right) \left[\frac{(p_{\min} - \eta)(t-1)}{2(M-1)} - 2\right]\right)} \nonumber\\
    &~~~~ + \frac{c_3}{c_2(t-1)\left[\log \left(c_2(t-1)\right)\right]^2}\nonumber\\
    \le & \frac{2}{(p_{\min}-\eta)\left[\mathrm{kl}\left(r'_{\gamma'}(a_i), \mu(a_1)\right)\right]^2} + \frac{6e}{\log (T_4 - 1)} + \frac{2}{\eta^2} + \frac{2(M-1)^2}{\left[\mathrm{kl}\left(r_{\gamma'}, \mu(a_2)\right)\right]^2(p_{\min}-\eta)} \nonumber\\
    & + \frac{c_3}{c_2\log [c_2(T_4-2)]}
\end{align}
Combining~\eqref{equ:sd-kl-cb-lemma-2-beginning} and~\eqref{equ:sd-kl-cb-lemma-2-final}, we have
\begin{align}
    & \expt\left[ \sum_{t=1}^{T(K,\pi)} \sum_{i=2}^{M} \mathbbm{1}\{S_t = 0, A_t = a_i\} \left[ V^*(0) - Q^*(0, a_i)\right]\right] \nonumber\\
    \le & \sum_{i=2}^{M} \left[ V^*(0) - Q^*(0, a_i)\right] \left[ T_4 + \sum_{t=T_4}^{\infty} \prob \left(S_t = 0, A_t = a_i\right)\right]\nonumber\\
    \le & \sum_{i=2}^{M} \left[ V^*(0) - Q^*(0, a_i)\right] c_{6,i}
\end{align}
where $c_{6,i}\coloneqq \frac{2}{(p_{\min}-\eta)\left[\mathrm{kl}\left(r'_{\gamma'}(a_i), \mu(a_1)\right)\right]^2} + \frac{6e}{\log (T_4 - 1)} + \frac{2}{\eta^2} + \frac{2(M-1)^2}{\left[\mathrm{kl}\left(r_{\gamma'}, \mu(a_2)\right)\right]^2(p_{\min}-\eta)}+ \frac{c_3}{c_2\log [c_2(T_4-2)]}$.
\end{proof}

Next, we will bound the expected regret induced by pulling suboptimal arms in state $1$. Lemma~\ref{lemma:sd-kl-cb-3} shows the result.
\begin{lemma}\label{lemma:sd-kl-cb-3}
    Let all the assumptions in Theorem~\ref{theorem:2} hold. Consider the KL-ULCB algorithm with $c_0=c_1=1$, and $c=4$. For any $\epsilon>0$, the regret induced in state $1$ can be bounded by
    \begin{align}
        & \expt\left[ \sum_{t=1}^{T(K,\pi)} \sum_{i=2}^{M} \mathbbm{1}\{S_t = 1, A_t = a_i\} \left[ V^*(1) - Q^*(1, a_i)\right]\right] \nonumber\\
        \le & \sum_{i\neq 1} \frac{1+\epsilon}{\mathrm{kl}(\mu(a_i),\mu(a_1))} \left(V^*(1) - Q^*(1,a_i)\right) \log K + o(\log K)
    \end{align}
\end{lemma}
\begin{proof}
This proof is similar to the proof of Lemma~\ref{lemma:sd-cb-3}. Using the same argument as in the proof of Lemma~\ref{lemma:sd-cb-3}, we have
\begin{align}\label{equ:proof-lemma-sd-kl-cb-3-regret-state-1}
    & \expt\left[ \sum_{t=1}^{T(K,\pi)} \sum_{i=2}^{M} \mathbbm{1}\{S_t = 1, A_t = a_i\} \left[ V^*(1) - Q^*(1, a_i)\right]\right] \nonumber\\
    \le & \sum_{i=2}^{M} \left[ V^*(1) - Q^*(1, a_i)\right] \nonumber\\
    &~~~\left( M + \expt[B'(T(K,\pi))] + \frac{6e}{\log M} + \sum_{t=M+1}^{\infty}\prob\left(A_t = a_i, N_t(a_i)\ge B'(t), \tilde{\mu}^1_t(a_i) \ge \mu(a_1)\right) \right)
\end{align}
where $B'(t)\coloneqq  \frac{1+\epsilon}{\mathrm{kl}(\mu(a_i),\mu(a_1))}\left[\log t + 4\log (\log t) \right]$. Consider that the event
\[
    \left\{A_t = a_i, N_t(a_i)\ge B'(t), \tilde{\mu}^1_t(a_i) \ge \mu(a_1)\right\}
\]
holds. Then we have
\begin{align}
    \mathrm{kl}^+\left(\bar{\mu}_t(a_i), \mu(a_1)\right) \le & \mathrm{kl} \left(\bar{\mu}_t(a_i), \tilde{\mu}^1_t(a_i)\right) \le \frac{\log t + 4 \log (\log t)}{N_t(a_i)} \nonumber\\
    \le & \frac{\log t + 4 \log (\log t)}{B'(t)} = \frac{\mathrm{kl}(\mu(a_i),\mu(a_1))}{1+\epsilon}
\end{align}
where $\mathrm{kl}^+(x,y)\coloneqq \mathrm{kl}(x,y)\mathbbm{1}\{x<y\}$. Define $r'_{\epsilon}(a_i)\in\{\mu(a_i),\mu(a_1)\}$ such that
\begin{align}
    \mathrm{kl}(r'_{\epsilon}(a_i),\mu(a_1)) = \frac{\mathrm{kl}(\mu(a_i), \mu(a_1))}{1+\epsilon}.
\end{align}
Then we have
\begin{align}
    \mathrm{kl}^+\left(\bar{\mu}_t(a_i), \mu(a_1)\right) \le \mathrm{kl}(r'_{\epsilon}(a_i),\mu(a_1))
\end{align}
which implies that $\mu(a_i) \le r'_{\epsilon}(a_i) \le \bar{\mu}_t(a_i)$. Hence we have
\begin{align}
    \mathrm{kl}(\bar{\mu}_t(a_i), \mu(a_i)) \ge \mathrm{kl}(r'_{\epsilon}(a_i), \mu(a_i))
\end{align}
Therefore, we have
\begin{align}\label{equ:proof-lemma-sd-kl-cb-3-sum-prob}
    & \sum_{t=M+1}^{\infty}\prob\left(A_t = a_i, N_t(a_i)\ge B'(t), \tilde{\mu}^1_t(a_i) \ge \mu(a_1)\right) \nonumber\\
    \le & \sum_{t=M+1}^{\infty}\prob\left(A_t=a_i, \mathrm{kl}(\bar{\mu}_t(a_i), \mu(a_i)) \ge \mathrm{kl}(r'_{\epsilon}(a_i), \mu(a_i))\right)\nonumber\\
    = & \sum_{t=M+1}^{\infty} \sum_{n=1}^{t-1} \prob\left(A_t=a_i, N_t(a_i)=n, \mathrm{kl}\left(\frac{1}{n}\sum_{s=1}^{n}R_s(a_i), \mu(a_i)\right) \ge \mathrm{kl}\left(r'_{\epsilon}(a_i), \mu(a_i)\right)\right)\nonumber\\
    \le & \sum_{n=1}^{\infty} \sum_{t=n+1}^{\infty} \prob\left(A_t=a_i, N_t(a_i)=n, \mathrm{kl}\left(\frac{1}{n}\sum_{s=1}^{n}R_s(a_i), \mu(a_i)\right) \ge \mathrm{kl}\left(r'_{\epsilon}(a_i), \mu(a_i)\right)\right)\nonumber\\
    \le & \sum_{n=1}^{\infty} \prob\left(\mathrm{kl}\left(\frac{1}{n}\sum_{s=1}^{n}R_s(a_i), \mu(a_i)\right) \ge \mathrm{kl}\left(r'_{\epsilon}(a_i), \mu(a_i)\right)\right)\nonumber\\
    \le & \sum_{n=1}^{\infty} \exp\left(-n\mathrm{kl}\left(r'_{\epsilon}(a_i), \mu(a_i)\right)\right)\nonumber\\
    \le & \frac{1}{\mathrm{kl}\left(r'_{\epsilon}(a_i), \mu(a_i)\right)}
\end{align}
where the first equality is by law of total probability, where $\{R_s(a_i)\}_{s=1}^n$ are i.i.d. Bernoulli rewards of pulling arm $a_i$. The third inequality is due to the fact that $\{A_t=a_i, N_t(a_i)=n\}_{t=n+1}^{\infty}$ are mutually exclusive and the countable additivity of probability measure. The fourth inequality uses the concentration inequality for KL divergence in~\cite{mardia2020concentration}, and the last inequality is by integration. It then follows from~\eqref{equ:proof-lemma-sd-kl-cb-3-regret-state-1} and~\eqref{equ:proof-lemma-sd-kl-cb-3-sum-prob} that
\begin{align}\label{equ:proof-lemma-sd-kl-cb-3-final-1}
    & \expt\left[ \sum_{t=1}^{T(K,\pi)} \sum_{i=2}^{M} \mathbbm{1}\{S_t = 1, A_t = a_i\} \left[ V^*(1) - Q^*(1, a_i)\right]\right] \nonumber\\
    \le & \sum_{i=2}^{M} \left[ V^*(1) - Q^*(1, a_i)\right] \left( M + \expt[B'(T(K,\pi))] + \frac{6e}{\log M}  + \frac{1}{\mathrm{kl}\left(r'_{\epsilon}(a_i), \mu(a_i)\right)} \right)
\end{align}
where
\begin{align}\label{equ:proof-lemma-sd-kl-cb-3-final-2}
    \expt[B'(T(K,\pi))] = & \expt\left[\frac{1+\epsilon}{\mathrm{kl}(\mu(a_i),\mu(a_1))}\left[\log T(K,\pi) + 4\log (\log T(K,\pi)) \right]\right]\nonumber\\
    \le & \frac{1+\epsilon}{\mathrm{kl}(\mu(a_i),\mu(a_1))}\biggl[\log \expt\left[T(K,\pi)\right] + 4\log\left(\log\expt\left[T(K,\pi)\right]\right)\biggr]\nonumber\\
    \le & \frac{1+\epsilon}{\mathrm{kl}(\mu(a_i),\mu(a_1))}\biggl[\log c_5K + 4\log\left(\log (c_5 K) \right)\biggr]
\end{align}
where the first inequality is by Jensen's inequality, and the second inequality is by Lemma~\ref{lemma:sd-cb-bound-TKpi}. It then follows from~\eqref{equ:proof-lemma-sd-kl-cb-3-final-1} and~\eqref{equ:proof-lemma-sd-kl-cb-3-final-2} that
\begin{align}
    & \expt\left[ \sum_{t=1}^{T(K,\pi)} \sum_{i=2}^{M} \mathbbm{1}\{S_t = 1, A_t = a_i\} \left[ V^*(1) - Q^*(1, a_i)\right]\right] \nonumber\\
    \le & \sum_{i\neq 1} \frac{1+\epsilon}{\mathrm{kl}(\mu(a_i),\mu(a_1))} \left(V^*(1) - Q^*(1,a_i)\right) \log K + o(\log K)
\end{align}
\end{proof}

By the regret decomposition result~\eqref{equ:regret-decomp}, Lemma~\ref{lemma:sd-kl-cb-2}, and Lemma~\ref{lemma:sd-kl-cb-3}, we have
\begin{align}
    \limsup_{K\rightarrow \infty} \frac{\expt[\mathrm{Reg}_{\pi}(K)]}{\log K} \le \sum_{i\neq 1} \frac{1+\epsilon}{\mathrm{kl}(\mu(a_i),\mu(a_1))} \left(V^*(1) - Q^*(1,a_i)\right),
\end{align}
i.e., Theorem~\ref{theorem:2} is proved.

\subsection{Proof of Theorem~\ref{theorem:3}: lower bound}
\label{app:theorem3}

From the regret decomposition~\eqref{equ:regret-decomp}, given any consistent policy $\pi\in\Pi_{\mathrm{cons}}$, we have
\begin{align}\label{equ:theo3-regret-decomp-lower-bound}
    &\expt[\mathrm{Reg}_{\pi}(K)]\nonumber\\
    = & \expt\left[ \sum_{t=1}^{T(K,\pi)} \sum_{i=2}^{M} \mathbbm{1}\{S_t = 0, A_t = a_i\} \left( V^*(0) - Q^*(0, a_i)\right)  + \mathbbm{1}\{S_t = 1, A_t = a_i\} \left( V^*(1) - Q^*(1, a_i)\right) \right]\nonumber\\
    \ge & \expt\left[ \sum_{t=1}^{T(K,\pi)} \sum_{i=2}^{M}\mathbbm{1}\{A_t = a_i\} \left( V^*(1) - Q^*(1, a_i)\right) \right]\nonumber\\
    = & \sum_{i=2}^{M} \expt\left[ \sum_{t=1}^{T(K,\pi)} \mathbbm{1}\{A_t = a_i\} \right] \left( V^*(1) - Q^*(1, a_i)\right)
\end{align}
where the first inequality uses the conclusion of Lemma~\ref{lemma:sufficient-condition}, $V^*(0) - Q^*(0, a_i) \ge V^*(1) - Q^*(1, a_i)$. We have to bound the term $\expt\left[ \sum_{t=1}^{T(K,\pi)} \mathbbm{1}\{A_t = a_i\} \right]=\expt\left[ N_{T(K,\pi) + 1}(a_i) \right]$. Let $\epsilon\in(0,1)$. Consider a different system $i$, $i\in\{2,...,M\}$, where the only difference is that the mean value of the reward of pulling arm $a_i$ is $\mu'(a_i)\in(\mu(a_1),1)$ such that
\begin{align}\label{equ:theo3-def-system2}
    \mathrm{kl}(\mu(a_i), \mu'(a_i))\le (1+\epsilon)\mathrm{kl}(\mu(a_i),\mu(a_1))
\end{align}
Define event $C_i$ as follows:
\begin{align}
    C_i\coloneqq \left\{N_{T(K,\pi) + 1}(a_i)\le \frac{1-\epsilon}{\mathrm{kl}(\mu(a_i), \mu'(a_i))}\log K, ~\hat{\mathrm{kl}}_{N_{T(K,\pi) + 1}(a_i)} \le (1-\frac{\epsilon}{2})\log K\right\}
\end{align}
where for any $n$,
\begin{align}\label{equ:theo3-def-kl-hat}
    \hat{\mathrm{kl}}_{n} \coloneqq  \sum_{s=1}^n\log \frac{\mu(a_i)R_s(a_i) + (1-\mu(a_i))(1-R_s(a_i))}{\mu'(a_i)R_s(a_i) + (1-\mu'(a_i))(1-R_s(a_i))}
\end{align}
where $\{R_s(a_i)\}_{s=1}^n$ are i.i.d. Bernoulli rewards of pulling arm $a_i$ in the original system. It can be easily verified that $\expt[\hat{\mathrm{kl}}_{n}]=n\mathrm{kl}(\mu(a_i), \mu'(a_i))$. We first show that the change of measure identity~\eqref{equ:theo-3-change-of-measure-identity} holds. 
\begin{lemma}
\label{lemma:theo3-change-of-measure}
Given a policy $\pi\in\Pi$ and total number of episodes $K$, we have
\begin{align}\label{equ:theo-3-change-of-measure-identity}
    \prob' \left(C_i\right) = \expt\left[ \mathbbm{1}_{C_i} \exp\left(-\hat{\mathrm{kl}}_{N_{T(K,\pi) + 1}(a_i)}\right) \right]
\end{align}
where $\prob'$ is the probability measure in system $i$, and $\expt$ is based on probability measure in the original system.
\end{lemma}
\begin{proof}
For any outcome (sample path) $\omega\in C_i$, let $X(\omega)$ denotes the value of the random variable $X$ on the sample path $\omega$. Then we have
\begin{align}
    \prob\left(\{\omega\}\right) = & \prod_{k=1}^{K} \prob\left(S_{k,1}=S_{k,1}(\omega)\right) \prod_{h=1}^{I_{k}(\omega) - 1} 
    \prob\left(A_{k,h}=A_{k,h}(\omega)\vert S_{k,h}(\omega), \phi_{k,h}(\omega), \pi\right)\nonumber\\
    &
    \prob\left(R_{k,h} = R_{k,h}(\omega)\vert A_{k,h}(\omega)\right)
    (1-q\left(S_{k,h}(\omega), R_{k,h}(\omega))\right)\nonumber\\
    & 
    \prob\left(A_{k,I_{k}(\omega)}=A_{k,I_{k}(\omega)}(\omega)\vert S_{k,I_{k}(\omega)}(\omega), \phi_{k,I_{k}(\omega)}(\omega), \pi\right)\nonumber\\
    & \prob\left(R_{k,I_{k}(\omega)} = R_{k,I_{k}(\omega)}(\omega)\vert A_{k,I_{k}(\omega)}(\omega)\right) q\left(S_{k,I_{k}(\omega)}(\omega), R_{k,I_{k}(\omega)}(\omega)\right)
\end{align}
Let $k(s)$ and $h(s)$ denote the episode number and time step when $a_i$ was pulled for the $s$-th time, respectively.
Hence, we have
\begin{align}
    &\prob'\left(\{\omega\}\right) \nonumber\\
    = & \prob\left(\{\omega\}\right) \prod_{s=1}^{N_{T(K,\pi) + 1}(a_i)(\omega)} \left[\mathbbm{1}\left\{ R_{k(s)(\omega),h(s)(\omega)}(\omega)=1 \right\}\frac{\mu'(a_i)}{\mu(a_i)} + \mathbbm{1}\left\{ R_{k(s)(\omega),h(s)(\omega)}(\omega)=0 \right\}\frac{1-\mu'(a_i)}{1-\mu(a_i)}\right]
\end{align}
It then follows that
\begin{align}
    \prob'\left(C_i\right) = & \sum_{\omega\in C_i} \prob'\left(\{\omega\}\right)\nonumber\\
    = & \sum_{\omega\in C_i} \prob\left(\{\omega\}\right) \prod_{s=1}^{N_{T(K,\pi) + 1}(a_i)(\omega)} \biggl[ \mathbbm{1}\left\{ R_{k(s)(\omega),h(s)(\omega)}(\omega)=1 \right\}\frac{\mu'(a_i)}{\mu(a_i)}\biggr. \nonumber\\
    & ~~~~~~~~~~~~~~~~~~~~~~~~~~~~~~~~~~\biggl.+ \mathbbm{1}\left\{ R_{k(s)(\omega),h(s)(\omega)}(\omega)=0 \right\}\frac{1-\mu'(a_i)}{1-\mu(a_i)}\biggr]\nonumber\\
    = & \expt\left[\mathbbm{1}_{C_i} \prod_{s=1}^{N_{T(K,\pi) + 1}(a_i)} \mathbbm{1}\left\{ R_{k(s),h(s)}=1 \right\}\frac{\mu'(a_i)}{\mu(a_i)} + \mathbbm{1}\left\{ R_{k(s),h(s)}=0 \right\}\frac{1-\mu'(a_i)}{1-\mu(a_i)} \right]\nonumber\\
    = & \expt\left[ \mathbbm{1}_{C_i} \exp\left(-\hat{\mathrm{kl}}_{N_{T(K,\pi) + 1}(a_i)}\right) \right]
\end{align}
where the last equality is by the definition of $\hat{\mathrm{kl}}_n$ in~\eqref{equ:theo3-def-kl-hat}.
\end{proof}

By Lemma~\ref{lemma:theo3-change-of-measure} and $\hat{\mathrm{kl}}_{N_{T(K,\pi) + 1}(a_i)} \le (1-\frac{\epsilon}{2})\log K$ in the definition of $C_i$, we have
\begin{align}
    \prob' \left(C_i\right) = \expt\left[ \mathbbm{1}_{C_i} \exp\left(-\hat{\mathrm{kl}}_{N_{T(K,\pi) + 1}(a_i)}\right) \right] \ge  \prob\left(C_i\right) K^{-\left(1-\frac{\epsilon}{2}\right)}
\end{align}
It follows that
\begin{align}\label{equ:theo3-prob-ci-upper-bound}
    \prob\left(C_i\right) \le & K^{\left(1-\frac{\epsilon}{2}\right)} \prob'\left(C_i\right)\nonumber\\
    \le & K^{\left(1-\frac{\epsilon}{2}\right)} \prob'\left(N_{T(K,\pi) + 1}(a_i)\le \frac{1-\epsilon}{\mathrm{kl}(\mu(a_i), \mu'(a_i))}\log K\right)\nonumber\\
    = & K^{\left(1-\frac{\epsilon}{2}\right)} \prob'\left(\sum_{j\neq i}N_{T(K,\pi) + 1}(a_j) \ge T(K,\pi) - \frac{1-\epsilon}{\mathrm{kl}(\mu(a_i), \mu'(a_i))}\log K\right)\nonumber\\
    \le & K^{\left(1-\frac{\epsilon}{2}\right)} \prob'\left(\sum_{j\neq i}N_{T(K,\pi) + 1}(a_j) \ge K - \frac{1-\epsilon}{\mathrm{kl}(\mu(a_i), \mu'(a_i))}\log K\right)\nonumber\\
    \le & K^{\left(1-\frac{\epsilon}{2}\right)} \frac{\expt'\left[\sum_{j\neq i}N_{T(K,\pi) + 1}(a_j)\right]}{K - \frac{1-\epsilon}{\mathrm{kl}(\mu(a_i), \mu'(a_i))}\log K}
\end{align}
where the second inequality is by the definition of $C_i$, the first equality is due to the fact that $\sum_{j}N_{T(K,\pi) + 1}(a_j) = T(K,\pi)$, the third inequality is due to the fact that $T(K,\pi)\ge K$, and the last inequality is by Markov's inequality, where $\expt'$ is based on probability measure in system $i$. Since $\pi\in\Pi_{\mathrm{cons}}$, by the definition of consistent policies in Definition~\ref{def:consistent-policy}, we have
\begin{align}\label{equ:theo3-little-o-system-i}
    \expt'[\mathrm{Reg}_{\pi}(K)] = o\left(K^{\alpha}\right)
\end{align}
for any $\alpha>0$. Similar to~\eqref{equ:theo3-regret-decomp-lower-bound}, for system $i$, we can obtain
\begin{align}\label{equ:theo3-regret-lower-bound-system-i}
    \expt'[\mathrm{Reg}_{\pi}(K)] \ge & \sum_{j\neq i} \expt'\left[ \sum_{t=1}^{T(K,\pi)} \mathbbm{1}\{A_t = a_j\} \right] \left( V'^*(1) - Q'^*(1, a_j)\right)\nonumber\\
    \ge & \min_{j\neq i} \left( V'^*(1) - Q'^*(1, a_j)\right) \expt'\left[\sum_{j\neq i}N_{T(K,\pi) + 1}(a_j)\right]
\end{align}
where $V'^*$ and $Q'^*$ are the optimal value function and optimal Q function for system $i$, respectively. Since $\mu'(a_i) > \mu(a_j)$ for any $j\neq i$, similar to the original system, it can be verified that $\min_{j\neq i} \left( V'^*(1) - Q'^*(1, a_j)\right)>0$. Hence, from~\eqref{equ:theo3-regret-lower-bound-system-i}, we have
\begin{align}
    \expt'\left[\sum_{j\neq i}N_{T(K,\pi) + 1}(a_j)\right] \le & \frac{\expt'[\mathrm{Reg}_{\pi}(K)]}{\min_{j\neq i} \left( V'^*(1) - Q'^*(1, a_j)\right)} = o\left(K^{\alpha}\right)
\end{align}
for any $\alpha>0$, where the equality is by~\eqref{equ:theo3-little-o-system-i}. It then follows from~\eqref{equ:theo3-prob-ci-upper-bound} that
\begin{align}
    \prob\left(C_i\right) \le K^{\left(1-\frac{\epsilon}{2}\right)} \frac{o\left(K^{\alpha}\right)}{K - \frac{1-\epsilon}{\mathrm{kl}(\mu(a_i), \mu'(a_i))}\log K}
\end{align}
for any $\alpha>0$. Let $\alpha=\frac{\epsilon}{4}$. Then we have
\begin{align}\label{equ:theo3-prob-ci-little-o}
    \prob\left(C_i\right) \le K^{\left(1-\frac{\epsilon}{2}\right)} \frac{o\left(K^{\frac{\epsilon}{4}}\right)}{K - \frac{1-\epsilon}{\mathrm{kl}(\mu(a_i), \mu'(a_i))}\log K} = o\left(K^{-\frac{\epsilon}{4}}\right).
\end{align}
Note that by definition of $C_i$ and the law of total probability, we have 
\begin{align}\label{equ:theo3-prob-ntkai-less-than-logk}
    & \prob\left(N_{T(K,\pi) + 1}(a_i)\le \frac{1-\epsilon}{\mathrm{kl}(\mu(a_i), \mu'(a_i))}\log K\right) \nonumber\\
    = & \prob\left(C_i\right)  + \prob\left(N_{T(K,\pi) + 1}(a_i)\le \frac{1-\epsilon}{\mathrm{kl}(\mu(a_i), \mu'(a_i))}\log K, \hat{\mathrm{kl}}_{N_{T(K,\pi) + 1}(a_i)} > (1-\frac{\epsilon}{2})\log K \right)\nonumber\\
    \le & \prob\left(C_i\right)  + \prob\left(
    \hat{\mathrm{kl}}_{N_{T(K,\pi) + 1}(a_i)} - 
    N_{T(K,\pi) + 1}(a_i) \mathrm{kl}(\mu(a_i), \mu'(a_i)) 
    > \frac{\epsilon}{2} \log K,\right.\nonumber\\
    & ~~~~~~~~~~~~~~\left.N_{T(K,\pi) + 1}(a_i)\le \frac{1-\epsilon}{\mathrm{kl}(\mu(a_i), \mu'(a_i))}\log K
    \right)\nonumber\\
    \le & \prob\left(C_i\right) + \sum_{n=1}^{\frac{1-\epsilon}{\mathrm{kl}(\mu(a_i), \mu'(a_i))}\log K} \prob\left(\hat{\mathrm{kl}}_{n} - 
    n \mathrm{kl}(\mu(a_i), \mu'(a_i)) 
    > \frac{\epsilon}{2} \log K\right)\nonumber\\
    \le & \prob\left(C_i\right) + \sum_{n=1}^{\frac{1-\epsilon}{\mathrm{kl}(\mu(a_i), \mu'(a_i))}\log K} \exp\left(-\frac{\epsilon^2 (\log K)^2}{2n \left\lvert\log \frac{\mu(a_i)}{\mu'(a_i)}- \log \frac{1-\mu(a_i)}{1-\mu'(a_i)} \right\rvert^2}\right)\nonumber\\
    \le & \prob\left(C_i\right) + \frac{1-\epsilon}{\mathrm{kl}(\mu(a_i), \mu'(a_i))}\log K \exp\left(-\frac{\epsilon^2 (\log K)^2}{2 \frac{1-\epsilon}{\mathrm{kl}(\mu(a_i), \mu'(a_i))}\log K \left\lvert\log \frac{\mu(a_i)}{\mu'(a_i)}- \log \frac{1-\mu(a_i)}{1-\mu'(a_i)} \right\rvert^2}\right)\nonumber\\
    = & \prob\left(C_i\right) + \frac{(1-\epsilon)\log K}{\mathrm{kl}(\mu(a_i), \mu'(a_i)) K^{\frac{\epsilon^2 \mathrm{kl}(\mu(a_i), \mu'(a_i)) }{2(1-\epsilon)\left\lvert\log \frac{\mu(a_i)}{\mu'(a_i)}- \log \frac{1-\mu(a_i)}{1-\mu'(a_i)} \right\rvert^2}} }
\end{align}
where the second inequality is by union bound over all possible values of $N_{T(K,\pi) + 1}(a_i)$, and the third inequality is by Hoeffding's inequality. It follows from~\eqref{equ:theo3-prob-ci-little-o} and~\eqref{equ:theo3-prob-ntkai-less-than-logk} that
\begin{align}\label{equ:theo3-0-prob}
    \lim_{K\rightarrow \infty} \prob\left(N_{T(K,\pi) + 1}(a_i)\le \frac{1-\epsilon}{\mathrm{kl}(\mu(a_i), \mu'(a_i))}\log K\right) = 0
\end{align}
which implies that
\begin{align}\label{equ:theo3-1-prob}
    \lim_{K\rightarrow \infty} \prob\left(N_{T(K,\pi) + 1}(a_i) > \frac{1-\epsilon}{\mathrm{kl}(\mu(a_i), \mu'(a_i))}\log K\right) = 1.
\end{align}
By Markov's inequality, we have
\begin{align}
    \prob\left(N_{T(K,\pi) + 1}(a_i)> \frac{1-\epsilon}{\mathrm{kl}(\mu(a_i), \mu'(a_i))}\log K\right)\le \frac{\expt\left[ N_{T(K,\pi) + 1}(a_i) \right]}{\frac{1-\epsilon}{\mathrm{kl}(\mu(a_i), \mu'(a_i))}\log K}.
\end{align}
Therefore, we have
\begin{align}\label{equ:theo3-expt-ntkai}
    \expt\left[ N_{T(K,\pi) + 1}(a_i) \right] \ge \prob\left(N_{T(K,\pi) + 1}(a_i)> \frac{1-\epsilon}{\mathrm{kl}(\mu(a_i), \mu'(a_i))}\log K\right) \frac{1-\epsilon}{\mathrm{kl}(\mu(a_i), \mu'(a_i))}\log K.
\end{align}
Hence, it follows from~\eqref{equ:theo3-regret-decomp-lower-bound} and~\eqref{equ:theo3-expt-ntkai} that
\begin{align}\label{equ:theo3-before-conc}
    & \expt[\mathrm{Reg}_{\pi}(K)]\nonumber\\
    \ge & \sum_{i=2}^{M} \expt\left[N_{T(K,\pi) + 1}(a_i)\right] \left( V^*(1) - Q^*(1, a_i)\right)\nonumber\\
    \ge & \sum_{i=2}^{M} \prob\left(N_{T(K,\pi) + 1}(a_i)> \frac{(1-\epsilon)\log K}{\mathrm{kl}(\mu(a_i), \mu'(a_i))}\right) \frac{(1-\epsilon)\log K}{\mathrm{kl}(\mu(a_i), \mu'(a_i))} \left( V^*(1) - Q^*(1, a_i)\right)\nonumber\\
    \ge & \sum_{i=2}^{M} \prob\left(N_{T(K,\pi) + 1}(a_i)> \frac{(1-\epsilon)\log K}{\mathrm{kl}(\mu(a_i), \mu'(a_i))}\right) \frac{(1-\epsilon)\log K}{(1+\epsilon)\mathrm{kl}(\mu(a_i), \mu(a_1))} \left( V^*(1) - Q^*(1, a_i)\right),
\end{align}
where the last inequality follows from~\eqref{equ:theo3-def-system2}.
From~\eqref{equ:theo3-1-prob} and~\eqref{equ:theo3-before-conc}, we have
\begin{align}
    \liminf_{K\rightarrow \infty} \frac{\expt[\mathrm{Reg}_{\pi}(K)]}{\log K} \ge \sum_{i\neq 1} \frac{1-\epsilon}{(1+\epsilon)\mathrm{kl}(\mu(a_i),\mu(a_1))} \left(V^*(1) - Q^*(1,a_i)\right).
\end{align}

\section{Extension to the general-state setting}
\label{app:extension-cont}

In this section, we present the details of the state transition, the details of the proofs, and the simulation results in the general-state setting.

\subsection{State transition}
\label{app:cont-ext-model}
The transition probabilities $P(s'|s,a)$ while pulling arm $a$ are shown in Table~\ref{tab:tran-prob-cont}, where $x\in[0,1]$.

\begin{table}[!htbp]
\small
\caption{Transition probabilities $P(s'|s,a)$}
\label{tab:tran-prob-cont}
\scriptsize
\centering
\begin{tabular}{|c|c|c|c|c|}
\hline
\multicolumn{2}{|c|}{ \multirow{2}*{$P(s'|s,a)$} }& \multicolumn{3}{c|}{Next state $s'$}\\
\cline{3-5}
\multicolumn{2}{|c|}{}&$(1-\theta)x$&$(1-\theta)x+\theta$&g\\
\hline
\multirow{2}*{State $s$}&$x$&$(1-\mu(a))[1-q((1-\theta)x)]$&$\mu(a)[1-q((1-\theta)x+\theta)]$&$(1-\mu(a))q((1-\theta)x)+\mu(a)q((1-\theta)x+\theta)$\\
\cline{2-5}
&g&0&0&1\\
\hline
\end{tabular}
\end{table}

\subsection{Proof of Lemma~\ref{lemma:cont-1}}
\label{app:proof-lemma-optimal-policy-general}

If the model is known, this problem can be viewed as a SSP problem~\cite{bertsekas1991analysis}. Since $q(s)>0$ for any $s\in[0,1]$, all policies are proper. Hence, by the results in~\cite{bertsekas1991analysis}, there exists a stationary optimal policy. Therefore, it is enough to consider only stationary policies for $\pi^*$.
Define the optimal value function $V^*$ and optimal Q function $Q^*$ the same way as~\eqref{equ:v*-def} and~\eqref{equ:q*-def}.
Then for $s\neq g$ and $a\in\mathcal{A}$, we have the Bellman equation as follows:
\begin{align}
    V^{*}(s) =& \max_{a} Q^{*}(s, a)\\
    Q^*(s,a) = & \mu(a) + (1-\mu(a))\left[1-q((1-\theta)s)\right]V^*((1-\theta)s) \nonumber\\
    & + \mu(a)\left[1-q((1-\theta)s+\theta)\right]V^*((1-\theta)s+\theta)
\end{align}
Hence, for any $s_1,s_2\in[0,1]$ such that $s_1\ge s_2$ and any $a\in\mathcal{A}$, we have
\begin{align}
    & Q^*(s_1,a)-Q^*(s_2,a)\nonumber\\
    = & (1-\mu(a)) \biggl\{\left[1-q((1-\theta)s_1)\right]V^*((1-\theta)s_1)-\left[1-q((1-\theta)s_2)\right]V^*((1-\theta)s_2)\biggr\}\nonumber\\
    & + \mu(a) \biggl\{\left[1-q((1-\theta)s_1 + \theta)\right]V^*((1-\theta)s_1) + \theta) -\left[1-q((1-\theta)s_2 + \theta)\right]V^*((1-\theta)s_2 + \theta)\biggr\}\nonumber\\
    \ge & (1-\mu(a))\left[1-q((1-\theta)s_1)\right]\left[V^*((1-\theta)s_1) - V^*((1-\theta)s_2)\right]\nonumber\\
    & + \mu(a)\left[1-q((1-\theta)s_1+ \theta)\right]\left[V^*((1-\theta)s_1+ \theta) - V^*((1-\theta)s_2+ \theta)\right]
\end{align}
where the inequality is by Assumption~\ref{assum:2}. Therefore, we have
\begin{align}\label{equ:cont-v*-diff}
    V^*(s_1)- V^*(s_2) 
    = & \max_{a} Q^*(s_1,a) - \max_{a} Q^*(s_2,a)\nonumber \\
    = & \max_{a} Q^*(s_1,a) -  Q^*(s_2,a')\nonumber\\
    \ge & Q^*(s_1,a') -  Q^*(s_2,a')\nonumber\\
    \ge & (1-\mu(a))\left[1-q((1-\theta)s_1)\right]\left[V^*((1-\theta)s_1) - V^*((1-\theta)s_2)\right]\nonumber\\
    & + \mu(a)\left[1-q((1-\theta)s_1+ \theta)\right]\left[V^*((1-\theta)s_1+ \theta) - V^*((1-\theta)s_2+ \theta)\right]
\end{align}
where $a'\coloneqq \mathop{\mathrm{argmax}}_a Q^*(s_2,a)$. Note that $1-q(s)<1$ for any $s\in[0,1]$ by Assumption~\ref{assum:2}.
Hence, by iteratively applying~\eqref{equ:cont-v*-diff}, we can obtain
\begin{align}
    V^*(s_1) - V^*(s_2) \ge 0
\end{align}
which means that $V^*(s)$ is non-decreasing in $s$.
Hence, for any $s\in[0,1]$ and any $i\in\{2,...,M\}$, we have
\begin{align}\label{equ:cont-q*-diff}
    & Q^*(s,a_1)-Q^*(s,a_i)\nonumber\\
    = & (\mu(a_1)-\mu(a_i)) + (\mu(a_1)-\mu(a_i))\biggl\{[1-q((1-\theta)s + \theta)]V^*((1-\theta)s + \theta) \nonumber\\
    & ~~~~~~~~~~~~~~~~~~~~~~~~~~~~~~~~~~~~~~~~~~~ - [1-q((1-\theta)s)]V^*((1-\theta)s)\biggr\}\nonumber\\
    \ge & 0
\end{align}
where the inequality is by Assumption~\ref{assum:2} and the monotonicity of $V^*$. Therefore, the genie-aided optimal policy is always pulling Arm $a_1$.

\subsection{Some examples of the abandonment probability functions}
\label{app:cont-state-gap-function}

We present some examples of the abandonment probability functions $q(\cdot)$ that satisfy $V^*(s_1) - Q^*(s_1, a) \le V^*(s_2) - Q^*(s_2, a)$ for any $a\in\mathcal{A}, s_1,s_2\in[0,1], s_1\ge s_2$.

For any $a\in\mathcal{A}, s_1,s_2\in[0,1], s_1\ge s_2$, we have
\begin{align}\label{equ:app-gap-diff}
    & \left[ V^*(s_1) - Q^*(s_1, a) \right] - \left[ V^*(s_2) - Q^*(s_2, a) \right]\nonumber\\
    = & \left[ Q^*(s_1,a_1) - Q^*(s_1, a) \right] - \left[ Q^*(s_2,a_1) - Q^*(s_2, a) \right]\nonumber\\
    = & (\mu(a_1)-\mu(a))
    \biggl\{
    [1-q((1-\theta)s_1 + \theta)]V^*((1-\theta)s_1 + \theta) - [1-q((1-\theta)s_1)]V^*((1-\theta)s_1)\nonumber\\
    & ~~~~~~~~~~~~~~~~~ -  [1-q((1-\theta)s_2 + \theta)]V^*((1-\theta)s_2 + \theta) +            [1-q((1-\theta)s_2)]V^*((1-\theta)s_2)
    \biggr\}
\end{align}
where the first quality is by Lemma~\ref{lemma:cont-1}, and the second equality is by~\eqref{equ:cont-q*-diff}. From~\eqref{equ:app-gap-diff}, we know that the sign of $\left[ V^*(s_1) - Q^*(s_1, a) \right] - \left[ V^*(s_2) - Q^*(s_2, a) \right]$ depends only on the abandonment probability function $q(\cdot)$ and $V^*(\cdot)$. By Lemma~\ref{lemma:cont-1}, the Bellman equation of $V^*(s)$ is as follows:
\begin{align}
    V^*(s) =& \mu(a_1) + (1-\mu(a_1))\left[1-q((1-\theta)s)\right]V^*((1-\theta)s) \nonumber\\
    & + \mu(a_1) \left[1-q((1-\theta)s + \theta)\right]V^*((1-\theta)s+\theta)
\end{align}
Therefore, if we know the abandonment function $q(\cdot)$ and $\mu(a_1)$, we can numerically calculate $V^*(\cdot)$ and then determine the sign of $\left[ V^*(s_1) - Q^*(s_1, a) \right] - \left[ V^*(s_2) - Q^*(s_2, a) \right]$ by~\eqref{equ:cont-q*-diff}. For example, let $\mu(a_1)=0.9$, $\theta=0.5$, and
\begin{align}\label{equ:app-ex-abandon-func}
    q(s)=1-\frac{\log(c_6s+1)}{\log(c_6+1)}
\end{align}
where $c_6$ is a constant.
\begin{figure}[htbp]
    \centering
    \subfigure[$c_6=5$]
    {
        \includegraphics[width=0.31\textwidth]{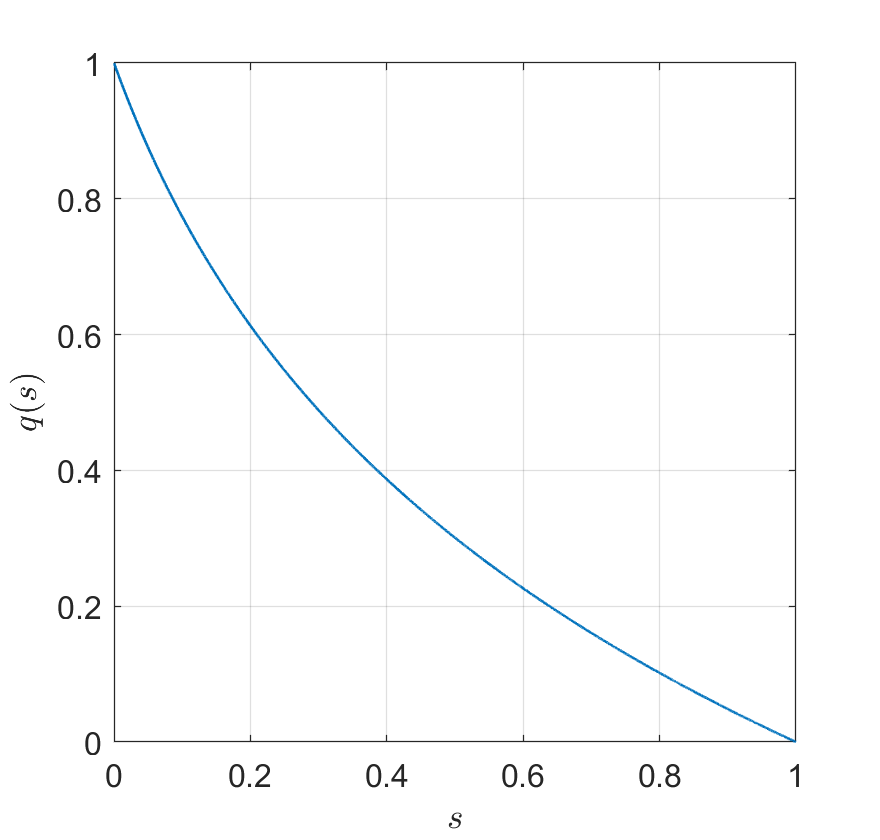}
    }
    \hfill
    \subfigure[$c_6=50$]
    {
        \includegraphics[width=0.31\textwidth]{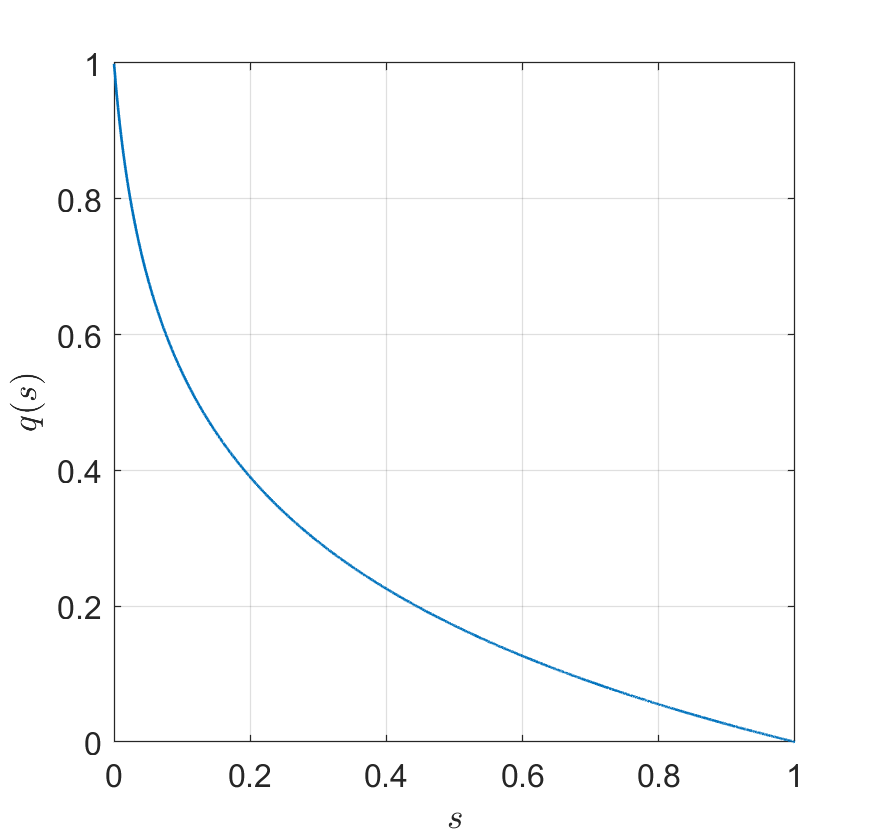}
    }
    \hfill
    \subfigure[$c_6=1000$]
    {
        \includegraphics[width=0.31\textwidth]{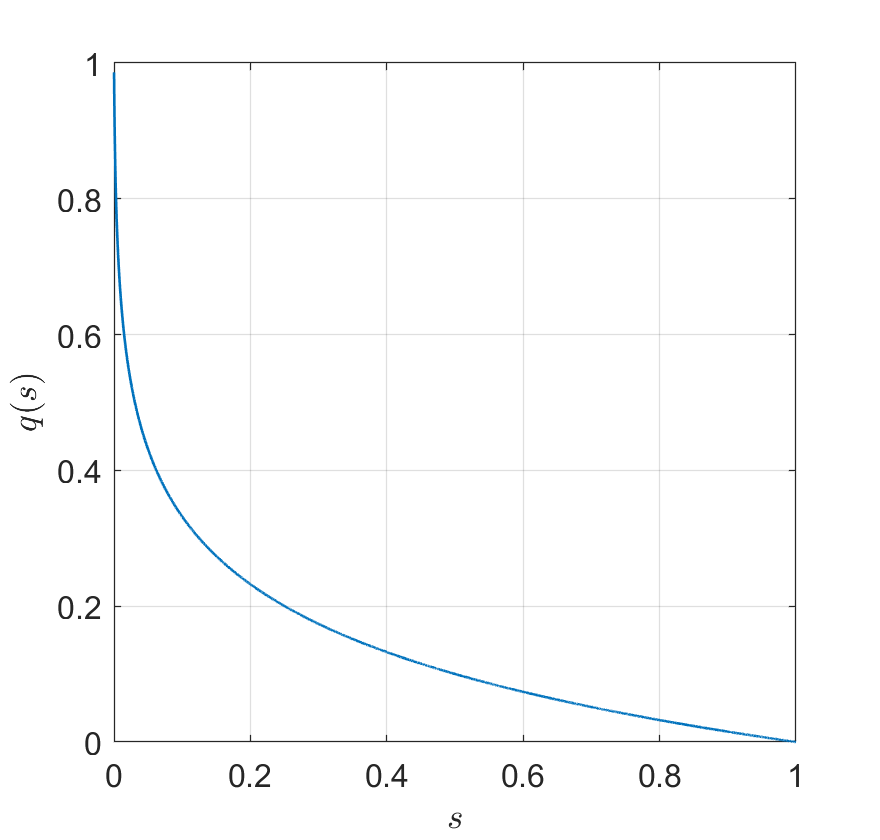}
    }
    \caption{Examples of abandonment probability functions.}
    \label{fig:app-ex-abandon-prob-func}
\end{figure}
Fig.~\ref{fig:app-ex-abandon-prob-func} shows the abandonment probability functions~\eqref{equ:app-ex-abandon-func} for $c_6=5$, $c_6=50$, and $c_6=1000$.
We numerically check that for $c_6=5$, $c_6=50$, and $c_6=1000$,
\begin{align}\label{equ:app-gap-diff-inequ}
    \left[ V^*(s_1) - Q^*(s_1, a) \right] \le \left[ V^*(s_2) - Q^*(s_2, a) \right]
\end{align}
for any $a\in\mathcal{A}, s_1,s_2\in[0,1], s_1\ge s_2$. We conjecture that~\eqref{equ:app-gap-diff-inequ} holds for any $c_6\ge 5$.

\subsection{Details of CONT-ULCB and CONT-KL-ULCB algorithms}
\label{app:cont-ext-alg}
The CONT-ULCB algorithm is shown in Algorithm~\ref{alg:3}. The CONT-KL-ULCB algorithm uses KL divergence in the indices, i.e., for any $s\in[0,1]$ and any $a\in\mathcal{A}$,
\begin{align}
    \tilde{\mu}_t^s(a) = \begin{cases}
    \min\left\{p:\mathrm{kl}(\bar{\mu}_t(a),p)N_t(a)\le (1 - 2s) \log t + c \log(\log t) \right\}, & s\le \frac{1}{2}\nonumber\\
    \max\left\{p:\mathrm{kl}(\bar{\mu}_t(a),p)N_t(a)\le (2s - 1) \log t + c \log(\log t) \right\}, & s > \frac{1}{2}.
    \end{cases}
\end{align}

\begin{algorithm}[ht]
\caption{CONT-ULCB Algorithm}\label{alg:3}
\begin{algorithmic}[1]
    \State {\textbf{Input:}} $N_1(a)\gets 0$, $\bar{\mu}_1(a) \gets 0$ for all $a\in\mathcal{A}$, $t\gets 1$, $c_0$, $c_1$, $c$.
    \State {$h\gets 1$, $S_{k,1} \gets$ initial state of episode $k$, $S_{k,1}\in\{0,1\}$}
    \For{episode $k=1,...,K$}
        \While{$S_{k,h} \neq g$}
            \If{there exists Arm $a'$ such that $N_{t}(a')=0$}
            \State {play Arm $A_{k,h} =a'$ and observe $R_{k,h}$}
            \Else
            \State {Let $\tilde{\mu}_t^{S_{k,h}}(a)=\bar{\mu}_{t}(a)+(2S_{k,h} - 1)\sqrt{\frac{\log t + c \log(\log t)}{2 N_t(a)}}$ for all $a\in\mathcal{A}$}
            \State {Take the action $A_{k,h}\in \operatorname{argmax}_{a}\tilde{\mu}^{S_{k,h}}_t(a)$}
            \State {Observe $R_{k,h}$}
            \EndIf
            \If{abandonment occurs}
                \State {$S_{k,h+1} = g$}
            \Else
                \State {$S_{k,h+1} = (1-\theta)S_{k,h} + \theta R_{k,h}$}
            \EndIf
            \State{Define $(S_{t}, A_{t}, S'_{t}, R_{t}) \coloneqq  (S_{k,h}, A_{k,h}, S_{k,h+1}, R_{k,h})$}
            \State {Update: $N_{t+1}(A_t) = N_{t}(A_t) + 1$ and $N_{t+1}(a) = N_{t}(a)~\forall a\neq A_t$}
            \State {Update: $\bar{\mu}_{t+1}(A_t) = \frac{\bar{\mu}_{t}(A_t)N_{t}(A_t)+R_t}{N_{t+1}(A_t)}$ and $\bar{\mu}_{t+1}(a) = \bar{\mu}_{t}(a)~\forall a\neq A_t$}
            \State {$t\gets t+1$, $h\gets h+1$}
        \EndWhile
    \EndFor
 \end{algorithmic}
\end{algorithm}

\subsection{Simulation results}
\label{app:cont-state-simu}

Consider the MAB-A problem of the general-state setting. Let the abandonment probability function $q(\cdot)$ be
\begin{align}
    q(s)=1-\frac{\log(c_6 s+1)}{\log(c_6 + 1)}
\end{align}
for any $s\in[0,1]$, where $c_6$ is a constant. Let the forgetting factor $\theta=0.5$ in the simulation.
We present the simulation results for our proposed DISC-ULCB, CONT-ULCB, DISC-KL-ULCB, and CONT-KL-ULCB algorithms. Let $n=4$ for the discretization of DISC-ULCB and DISC-KL-ULCB. We simulated $2\times 10^4$ episodes with $10^7$ independent runs. Simulation results are shown in Figure~\ref{fig:simu-cont}, Figure~\ref{fig:simu-cont-2}, and Figure~\ref{fig:simu-cont-3} for different sets of arms and different abandonment probabilities (different $c_6$).
\begin{remark}
For Figure~\ref{fig:simu-cont-comp}, the $95\%$ confidence bounds are at most $\pm 4.43$. For Figure~\ref{fig:simu-cont-2-comp}, the $95\%$ confidence bounds are at most $\pm 3.11$. For Figure~\ref{fig:simu-cont-3-comp}, the $95\%$ confidence bounds are at most $\pm 0.12$.
\end{remark}
From Figure~\ref{fig:simu-cont-comp}, Figure~\ref{fig:simu-cont-2-comp}, and Figure~\ref{fig:simu-cont-3-comp}, we can see in all the three different settings that both DISC-ULCB and CONT-ULCB algorithms outperform the traditional UCB in terms of average cumulative regret, and that both DISC-KL-ULCB and CONT-KL-ULCB algorithms outperform the traditional KL-UCB. Moreover, CONT-ULCB and CONT-KL-ULCB perform slightly better than DISC-ULCB and DISC-KL-ULCB, respecively. These figures also show that the slopes of DISC-ULCB and DISC-KL-ULCB converges to the slopes of their corresponding upper bounds. Also, in Figure~\ref{fig:simu-cont-ub-lb}, Figure~\ref{fig:simu-cont-2-ub-lb}, and Figure~\ref{fig:simu-cont-3-ub-lb} where the Y-axis is the average cumulative regret divided by $\log K$, the curves of DISC-ULCB and DISC-KL-ULCB go towards their corresponding asymptotic upper bounds, which confirms our theoretical results.

\begin{figure}[ht]
    \centering
    \subfigure[Comparison among algorithms]
    {
        \includegraphics[width=0.45\textwidth]{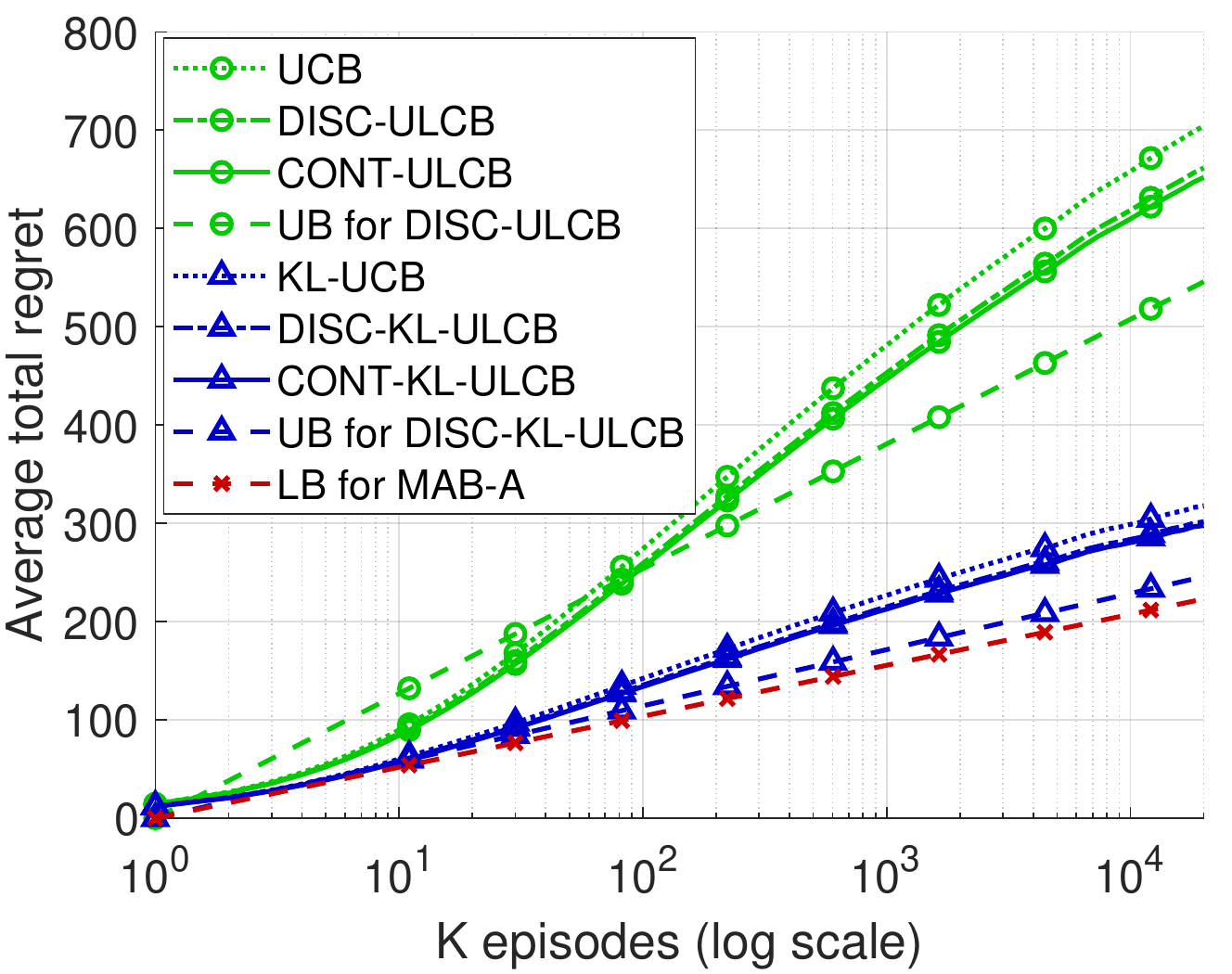}
        \label{fig:simu-cont-comp}
    }
    \hfill
    \subfigure[Upper bound (UB) and lower bound (LB)]
    {
        \includegraphics[width=0.45\textwidth]{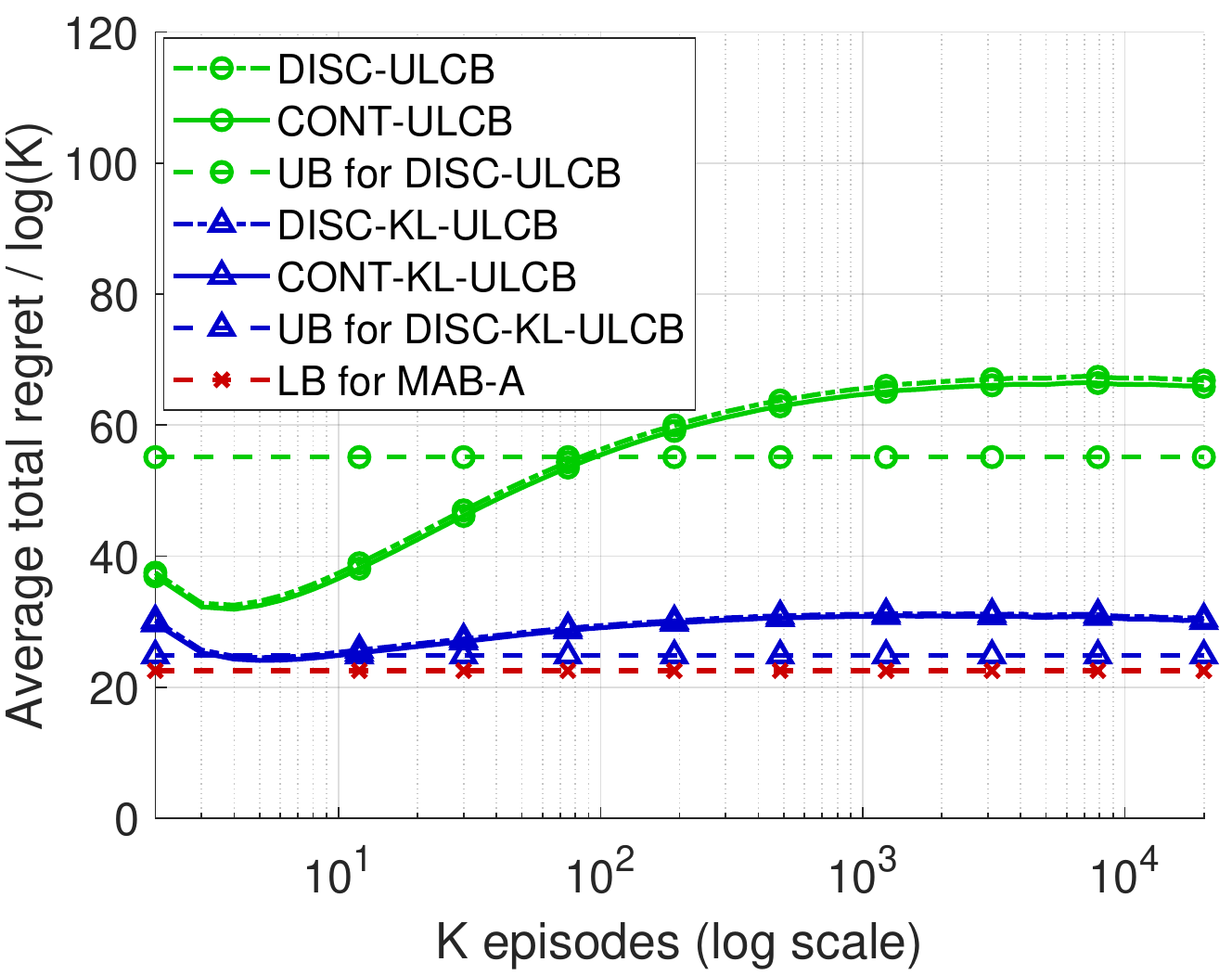}
        \label{fig:simu-cont-ub-lb}
    }
    \caption{Simulation results for the general-state setting, $M=2$, $\mu(a_1)=0.9$, $\mu(a_2)=0.8$, $c_6=1000$.}
    \label{fig:simu-cont}
\end{figure}

\begin{figure}[ht]
    \centering
    \subfigure[Comparison among algorithms]
    {
        \includegraphics[width=0.45\textwidth]{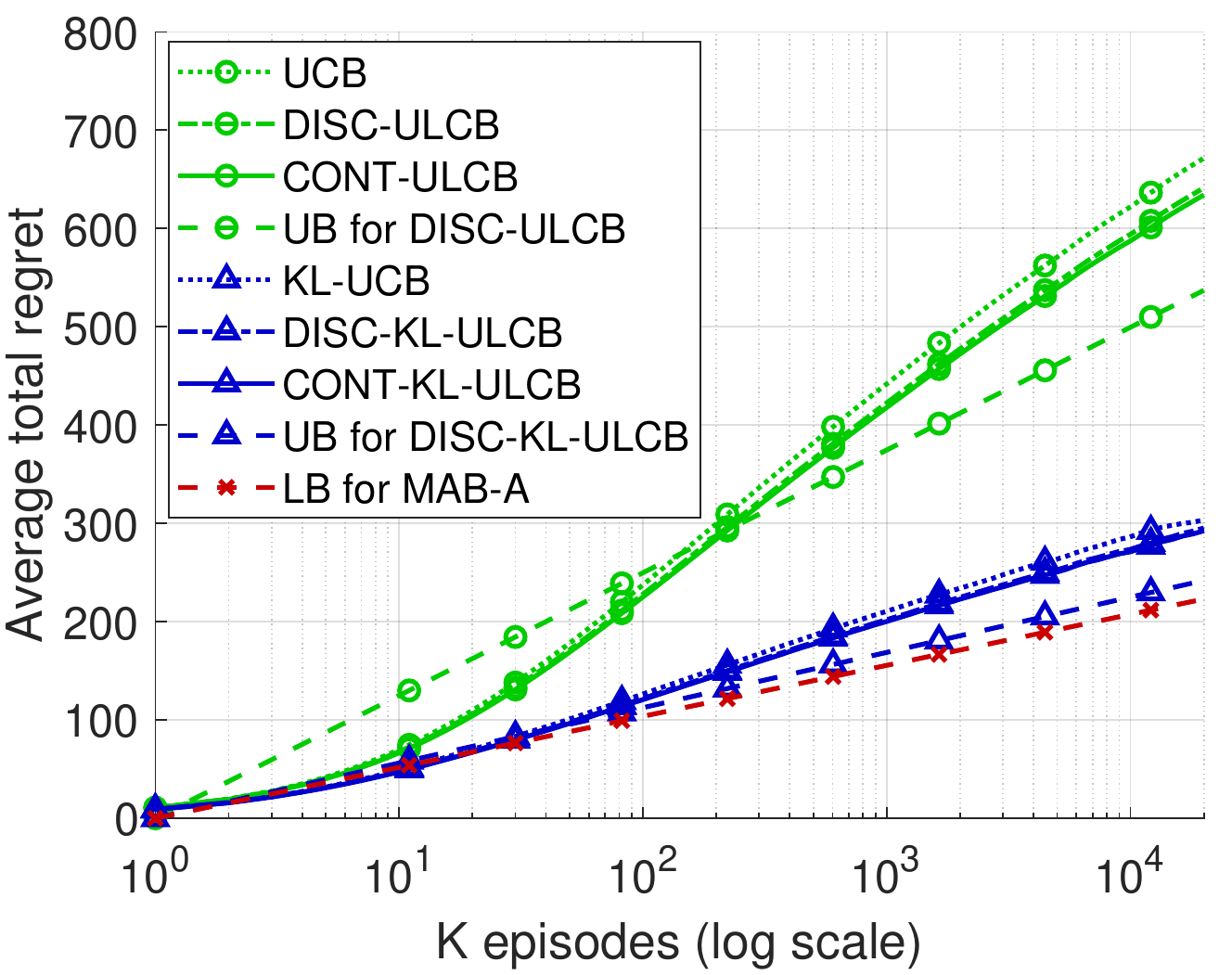}
        \label{fig:simu-cont-2-comp}
    }
    \hfill
    \subfigure[Upper bound (UB) and lower bound (LB)]
    {
        \includegraphics[width=0.45\textwidth]{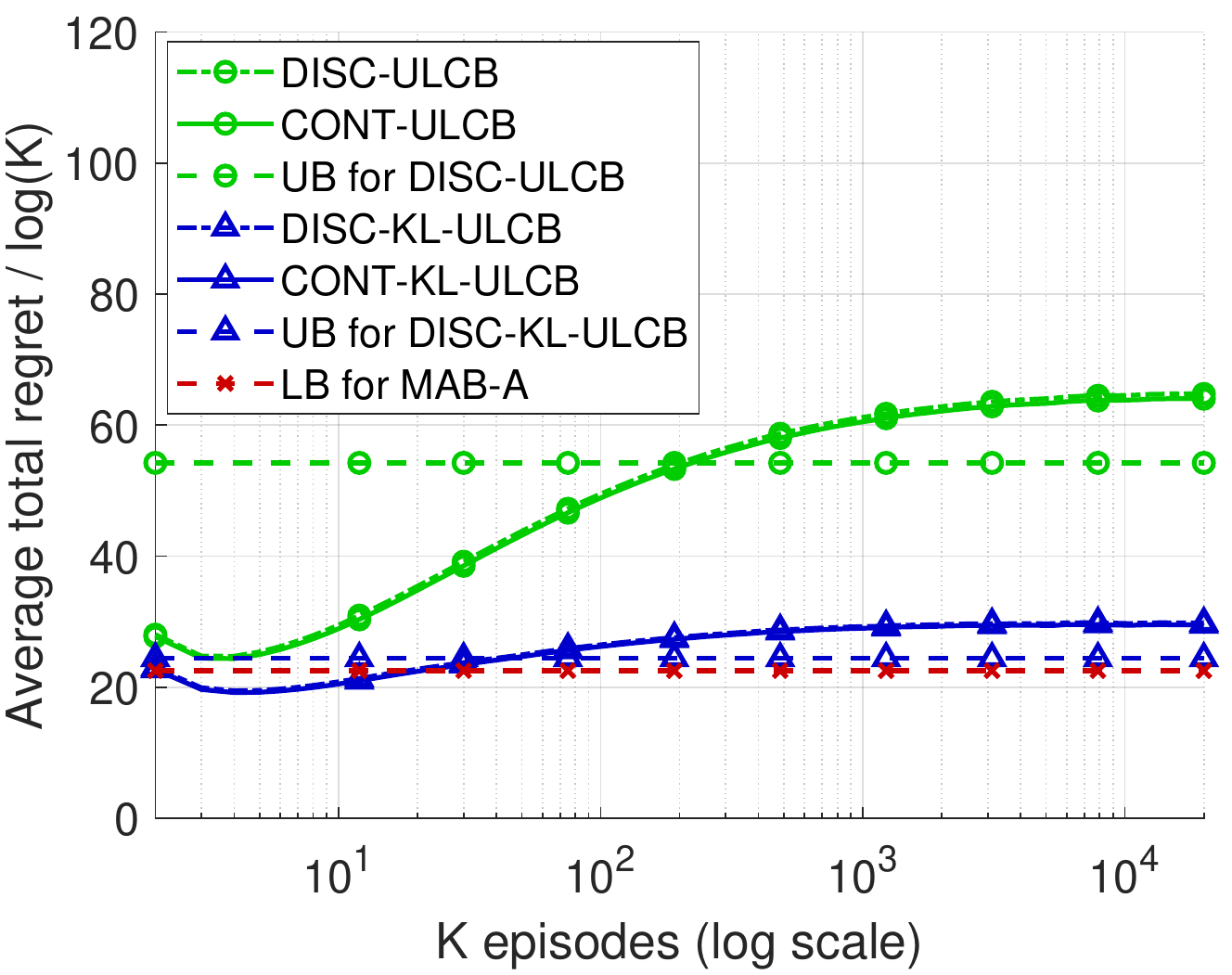}
        \label{fig:simu-cont-2-ub-lb}
    }
    \caption{Simulation results for the general-state setting, $M=2$, $\mu(a_1)=0.9$, $\mu(a_2)=0.8$, $c_6=100$.}
    \label{fig:simu-cont-2}
\end{figure}

\begin{figure}[ht]
    \centering
    \subfigure[Comparison among algorithms]
    {
        \includegraphics[width=0.45\textwidth]{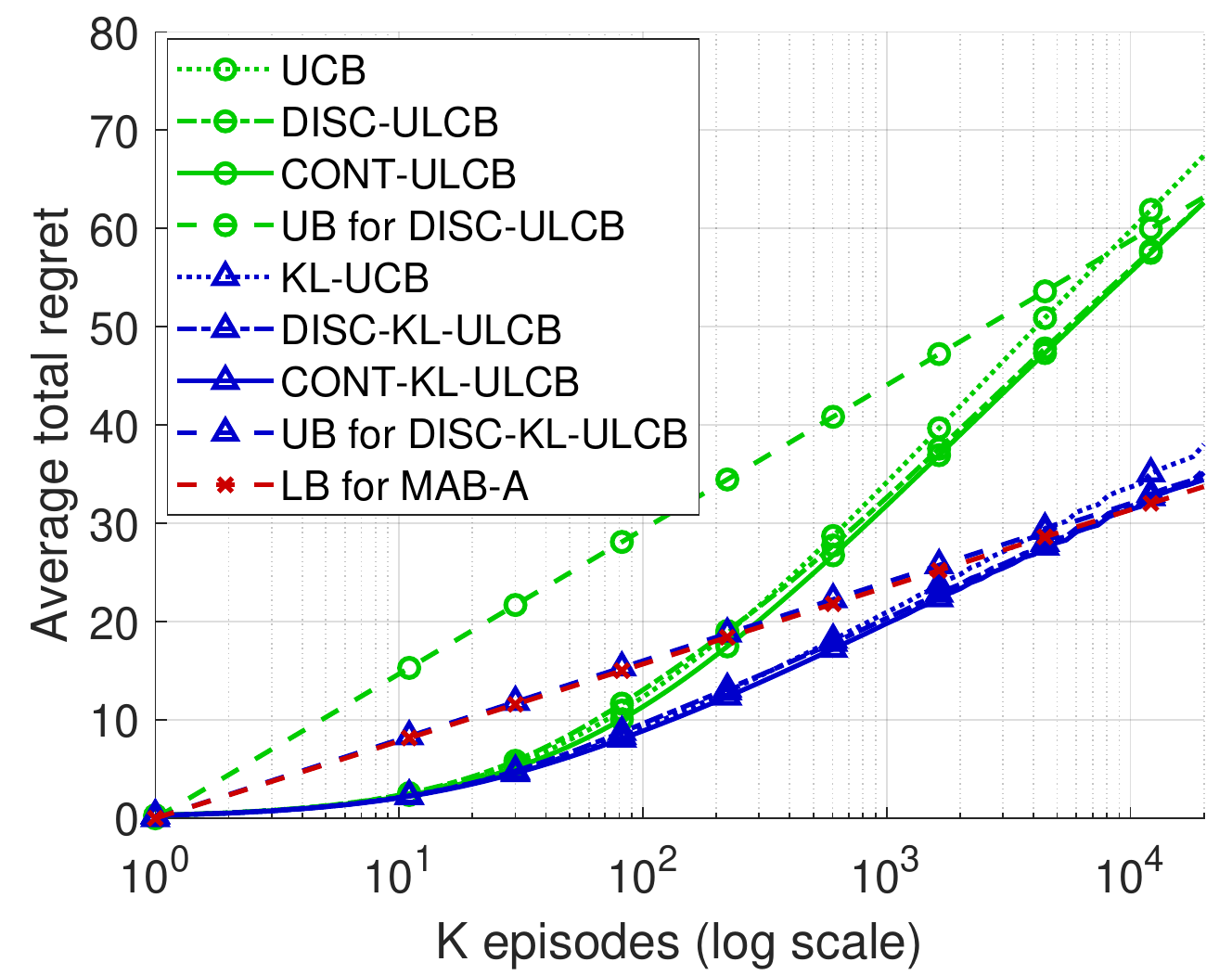}
        \label{fig:simu-cont-3-comp}
    }
    \hfill
    \subfigure[Upper bound (UB) and lower bound (LB)]
    {
        \includegraphics[width=0.45\textwidth]{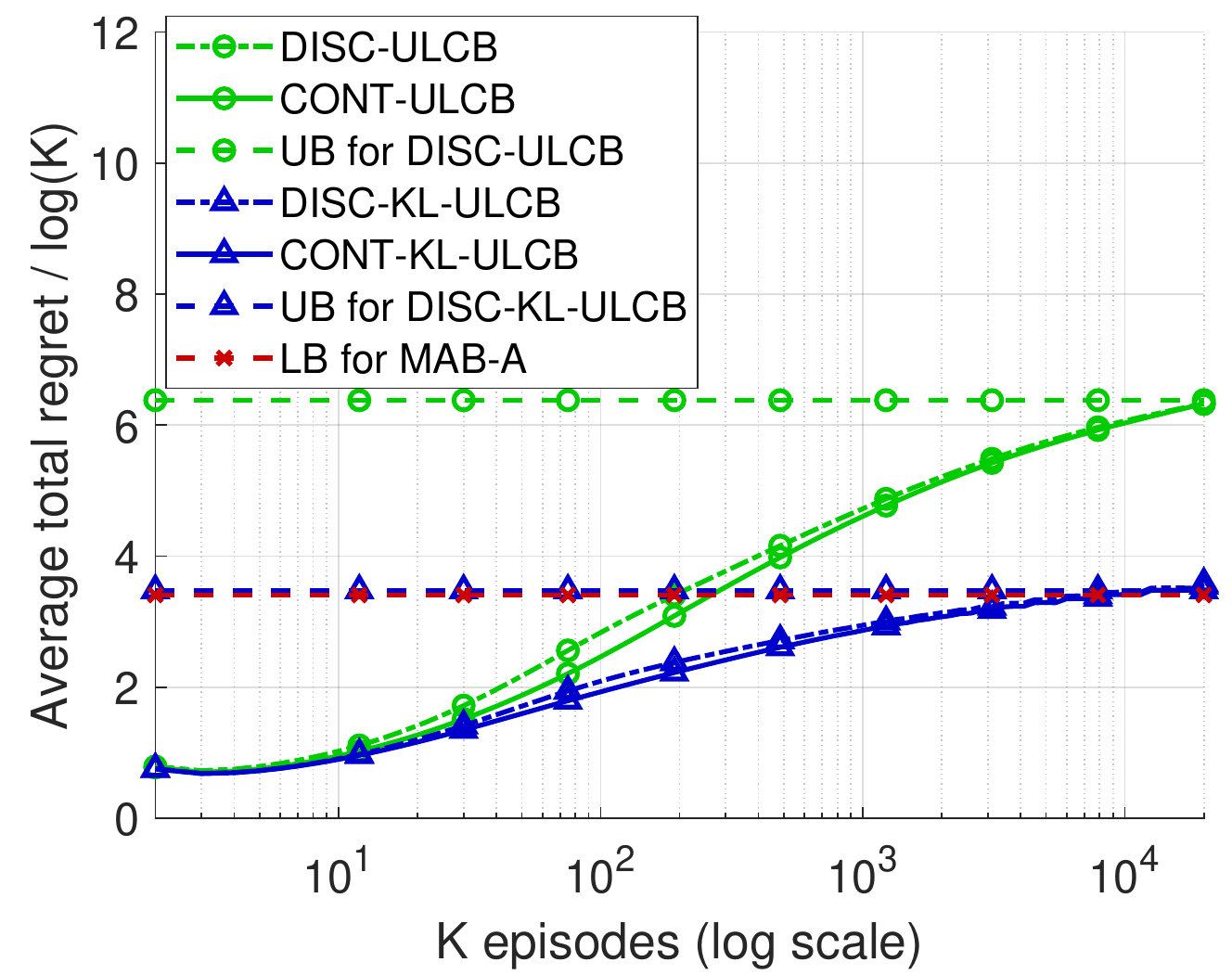}
        \label{fig:simu-cont-3-ub-lb}
    }
    \caption{Simulation results for the general-state setting, $M=2$, $\mu(a_1)=0.2$, $\mu(a_2)=0.1$, $c_6=1000$.}
    \label{fig:simu-cont-3}
\end{figure}

\section{Simulation parameters and additional simulations}
\label{app:add-simulations}

In the simulation of Q-learning with $\epsilon$-greedy, we set $\epsilon=0.1$. The learning rate is set to be $\frac{1}{N(s,a)}$, where $N(s,a)$ is the number of times the state-action pair $(s,a)$ has been visited. For Q-learning with UCB, we set the episode length parameter $H$ to be the maximum expected episode length in MAB-A, which is the expected episode length under the policy of always pulling the optimal arm $a_1$. The number of episodes $K$ is set to be $1000$. The constant $c$ in the bonus term is set to be $4$.

\begin{figure}[ht]
    \centering
    \subfigure[Comparison among algorithms]
    {
        \includegraphics[width=0.48\textwidth]{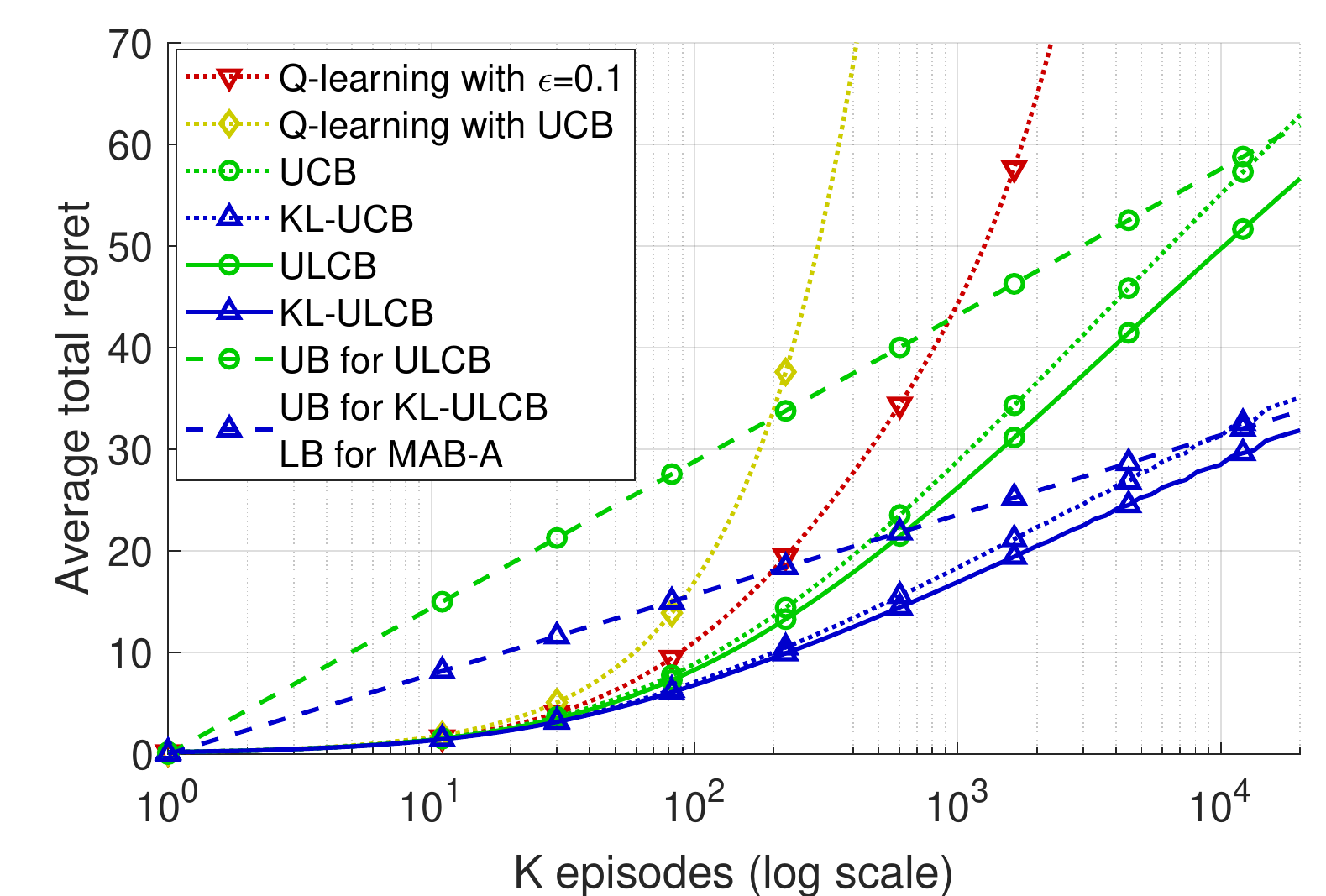}
        \label{fig:simu-0-1-comp-3}
    }
    \hfill
    \subfigure[Upper bound (UB) and lower bound (LB)]
    {
        \includegraphics[width=0.48\textwidth]{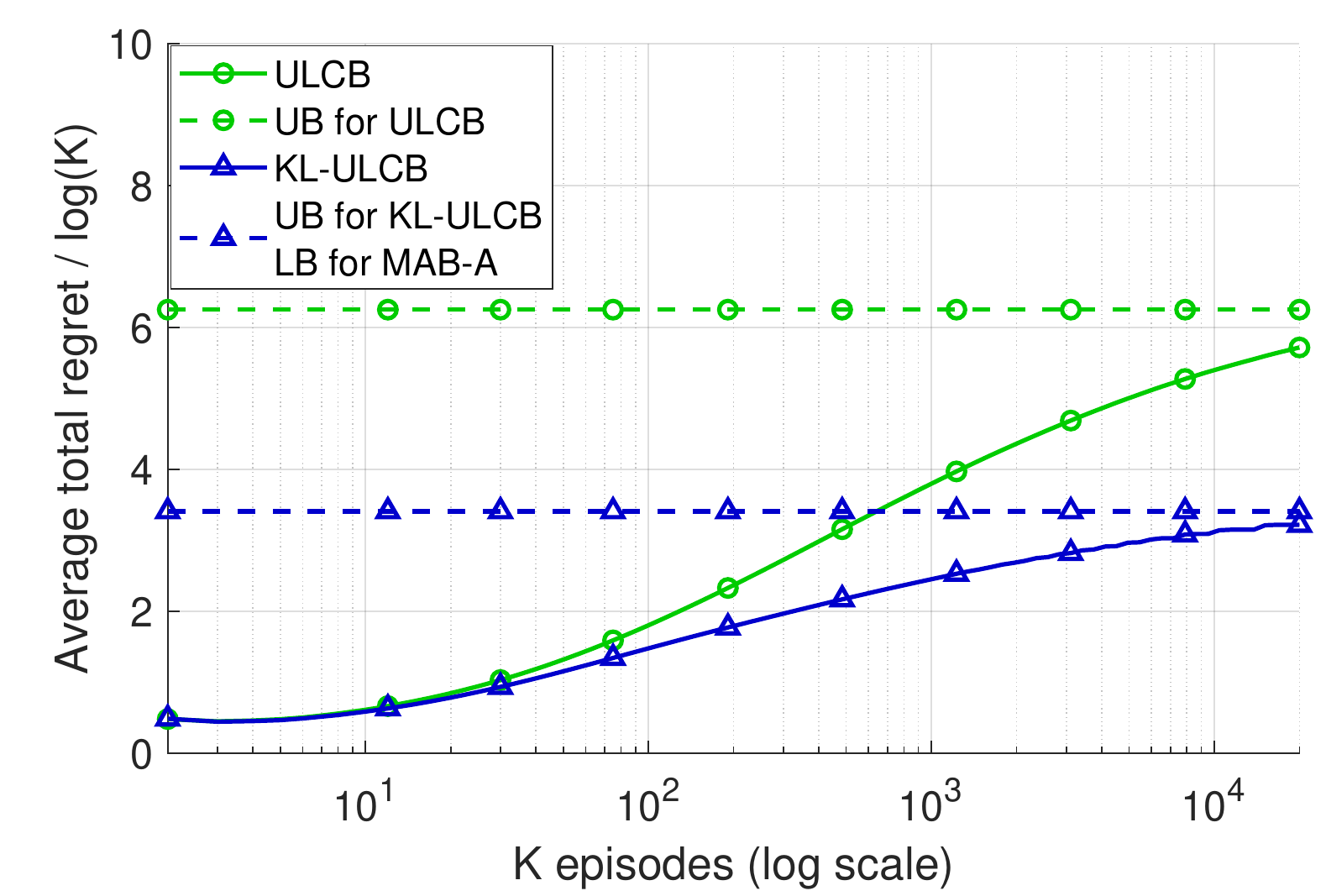}
        \label{fig:simu-0-1-ub-lb-3}
    }
    \caption{Simulation results, $M=2$, $\mu(a_1)=0.2$, $\mu(a_2)=0.1$, $q(0,0)=1$, $q(1,0)=q(0,1)=q(1,1)=0$.}
    \label{fig:simu-0-1-3}
\end{figure}

\begin{figure}[ht]
    \centering
    \subfigure[Comparison among algorithms]
    {
        \includegraphics[width=0.48\textwidth]{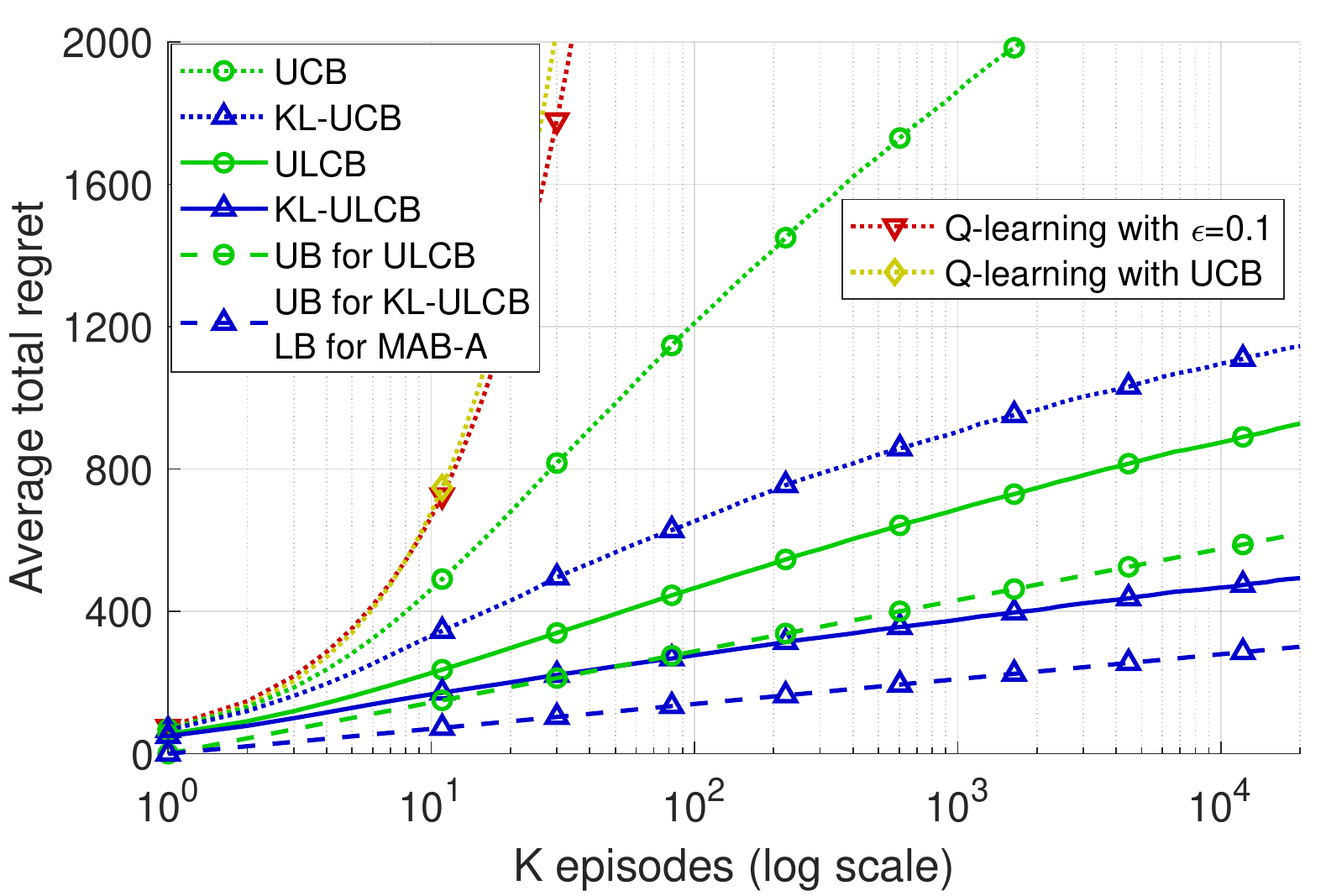}
        \label{fig:simu-0-1-comp-4}
    }
    \hfill
    \subfigure[Upper bound (UB) and lower bound (LB)]
    {
        \includegraphics[width=0.48\textwidth]{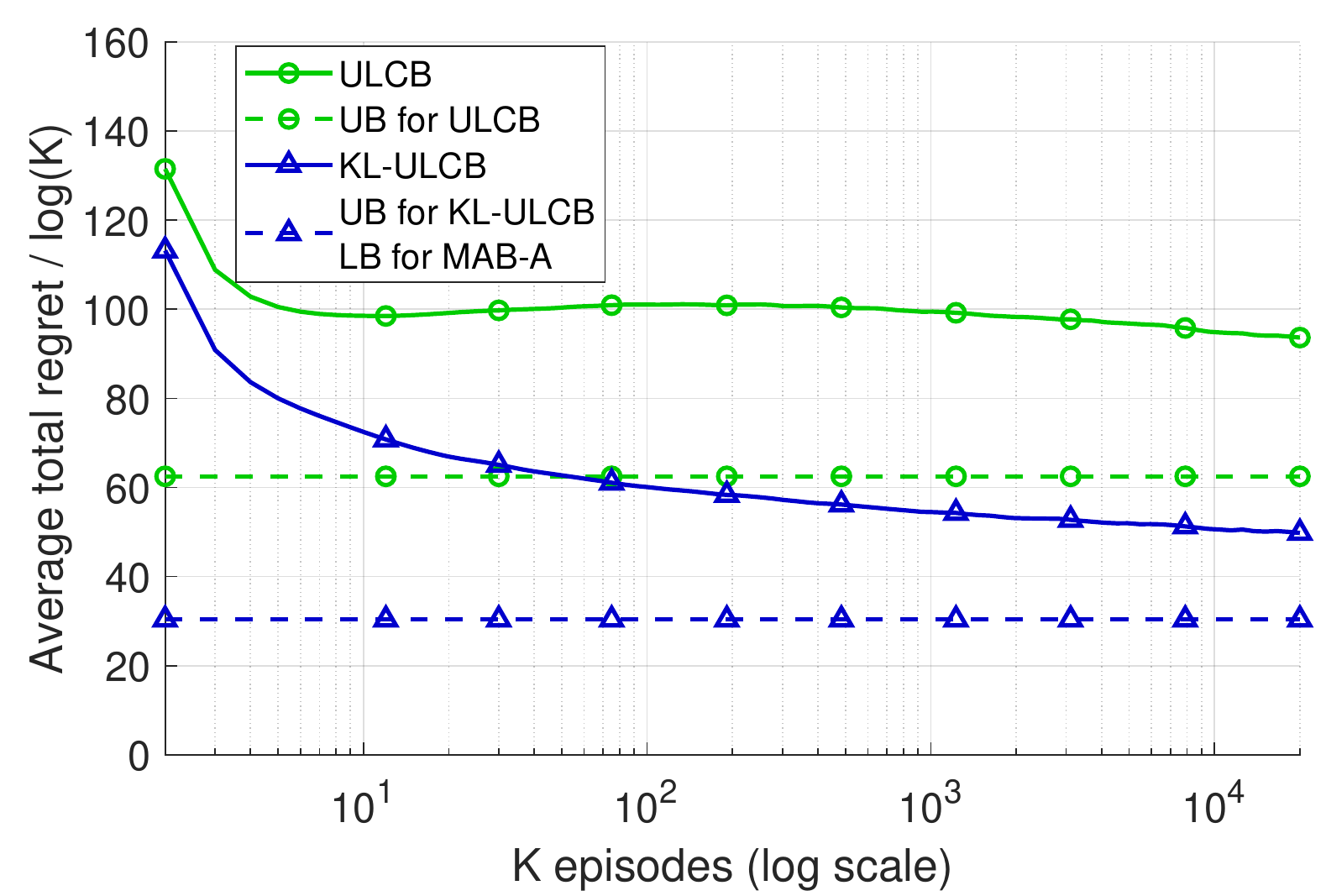}
        \label{fig:simu-0-1-ub-lb-4}
    }
    \caption{Simulation results, $M=3$, $\mu(a_1)=0.9$, $\mu(a_2)=0.8$, $\mu(a_3)=0.5$, $q(0,0)=1$, $q(1,0)=q(0,1)=q(1,1)=0$.}
    \label{fig:simu-0-1-4}
\end{figure}

We did additional simulations for different sets of arms and different abandonment probabilities, as shown in Figure~\ref{fig:simu-0-1-3} and Figure~\ref{fig:simu-0-1-4}. The other settings are the same as those in Section~\ref{sec:simu}. The same conclusion holds, i.e., our proposed ULCB and KL-ULCB outperform the traditional UCB and KL-UCB, respectively, 
our algorithms have order-wise lower regrets than Q-learning with $\epsilon$-greedy and Q-learning with UCB, and the simulation results are also consistent with our theoretical results.

\begin{figure}[ht]
    \centering
    \subfigure[Fix $c_0=-1$]
    {
        \includegraphics[width=0.43\textwidth]{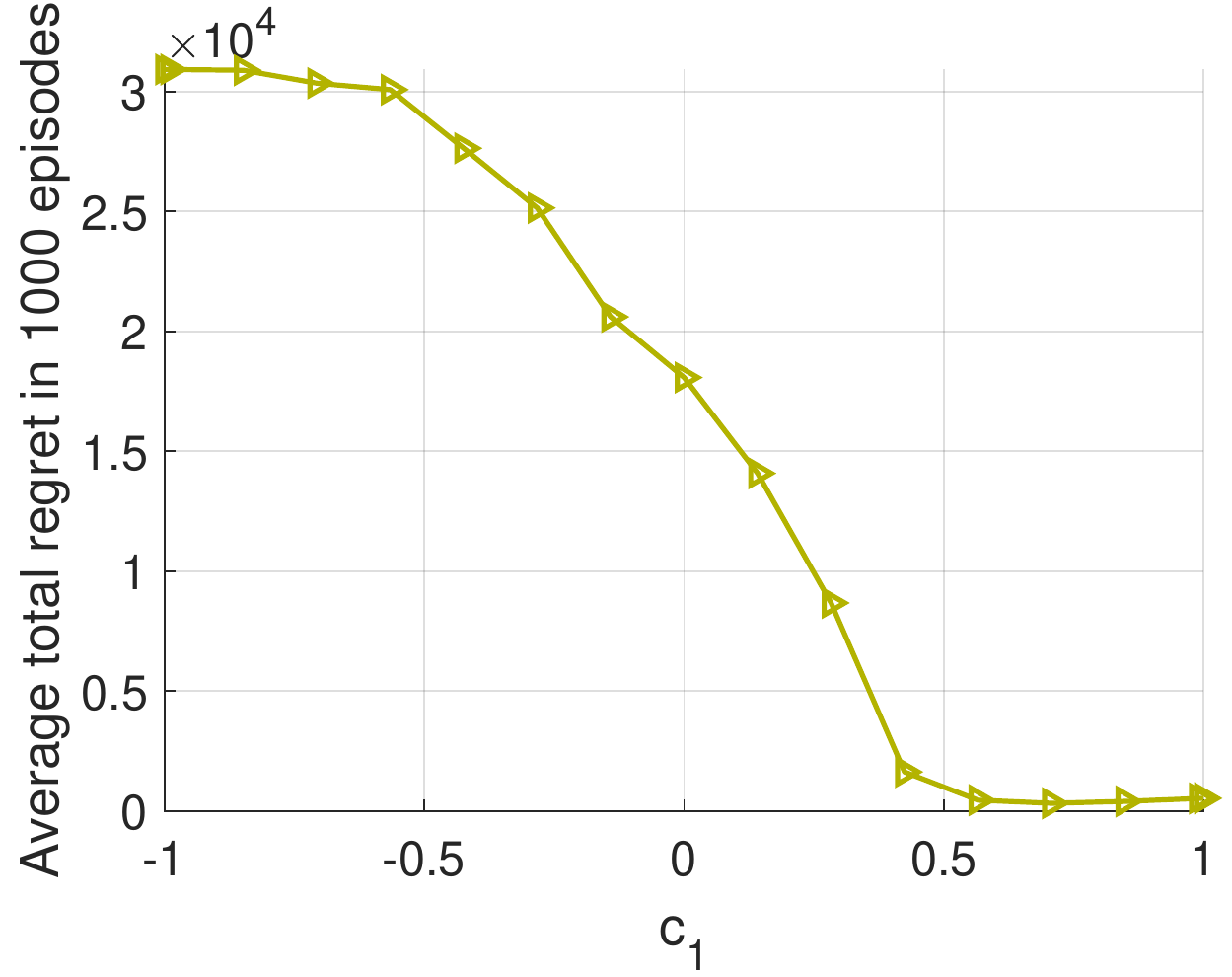}
        \label{fig:simu-diff-coef-fix-c0}
    }
    \hfill
    \subfigure[Fix $c_1=1$]
    {
        \includegraphics[width=0.43\textwidth]{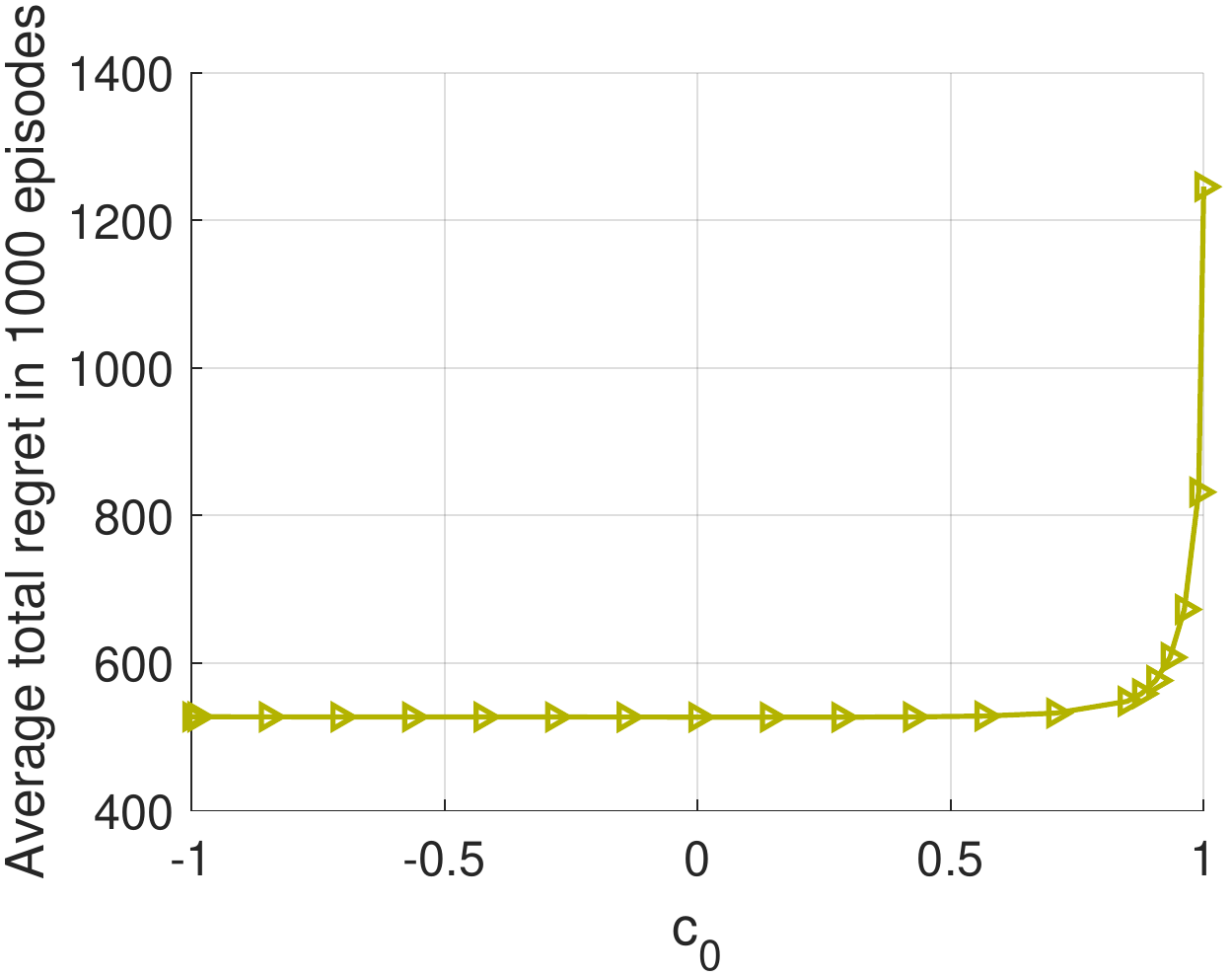}
        \label{fig:simu-diff-coef-fix-c1}
    }
    \caption{ULCB: different coefficients $c_0$ and $c_1$, $q(0,0)=1$, $q(1,1)=q(1,0)=q(0,1)=0$.}
    \label{fig:simu-diff-coef}
\end{figure}

In order to further understand the exploration and the exploitation in MAB-A problem, we run ULCB algorithms with different $c_0$ and $c_1$, which are the state-dependent exploration-exploitation coefficients in the indices for state $0$ and state $1$, respectively. Larger coefficient means more exploration. The results are shown in Figure~\ref{fig:simu-diff-coef}. 
In Figure~\ref{fig:simu-diff-coef-fix-c0}, $c_0$ is fixed, and as $c_1$ increases, the cumulative regret decreases. In Figure~\ref{fig:simu-diff-coef-fix-c1}, $c_1$ is fixed, and as $c_0$ increases, the cumulative regret increases. These changes of regret are consistent with our theoretical results which suggest more exploration in state $1$.

\section{Results for the opposite case}
\label{app:add-cases}

Note that you can see in~\eqref{equ:vqdiff} in Appendix~\ref{app:lemma:sufficient-condition} that the inequality $V^*(1)-Q^*(1,a)\ge V^*(0)-Q^*(0,a)$ or $V^*(1)-Q^*(1,a)\le V^*(0)-Q^*(0,a)$ does not depend on $a$. Hence, at least one of the two cases holds. In this section, we consider the case where $V^*(1)-Q^*(1,a)\ge V^*(0)-Q^*(0,a)$ for any $a\neq a_1$.  Similar to our main results, we also have corresponding algorithms and results. 

\subsection{Upper bound for ULCB}
\label{app:theorem1-case2}

\begin{theorem}\label{theorem:5}
Let Assumption~\ref{assum:1} hold. Assume that $\mu(a_1)>\mu(a_2)$, $\mu(a_M) > 0$, $q(0,1)<1$, $q(1,1)<1$, and $V^*(1)-Q^*(1,a)\ge V^*(0)-Q^*(0,a)$ for any $a\neq a_1$. Then using the ULCB algorithm with $c_0=1$, $c_1=-1$, and $c=4$, we have
\begin{align}
    \limsup_{K\rightarrow \infty} \frac{\expt[\mathrm{Reg}_{\pi}(K)]}{\log K} \le \sum_{i\neq 1} \frac{V^*(0) - Q^*(0,a_i)}{2(\mu(a_1)-\mu(a_i))^2}.
\end{align}
\end{theorem}
The proof is symmetric to that of Theorem~\ref{theorem:1} and hence is omitted.

\subsection{Upper bound for KL-ULCB}
\label{app:theorem2-case2}

\begin{algorithm}[ht]
\caption{KL-ULCB Algorithm for the Opposite Case}\label{alg:2-case2}
\begin{algorithmic}[1]
    \State {\textbf{Input:}} $N_1(a)\gets 0$, $\bar{\mu}_1(a) \gets 0$ for all $a\in\mathcal{A}$, $t\gets 1$, $c_0$, $c_1$, $c$.
    \State {$h\gets 1$, $S_{k,1} \gets$ initial state of episode $k$, $S_{k,1}\in\{0,1\}$}
    \For{episode $k=1,...,K$}
        \While{$S_{k,h} \neq g$}
            \If{there exists Arm $a'$ such that $N_{t}(a')=0$}
            \State {play Arm $A_{k,h} =a'$ and observe $R_{k,h}$}
            \Else
            \If{$S_{k,h}=0$}
            \State {Let $\tilde{\mu}_t^0(a)=\max\left\{p:\mathrm{kl}(\bar{\mu}_t(a),p)N_t(a)\le c_0 \log t + c \log(\log t) \right\}$ for all $a\in\mathcal{A}$}
            \State {Take the action $A_{k,h}\in \operatorname{argmax}_{a}\tilde{\mu}^0_t(a)$}
            \Else
            \State {Let $\tilde{\mu}_t^1(a)=\min\left\{p:\mathrm{kl}(\bar{\mu}_t(a),p)N_t(a)\le c_1 \log t + c \log(\log t) \right\}$ for all $a\in\mathcal{A}$}
            \State {Take the action $A_{k,h}\in \operatorname{argmax}_{a}\tilde{\mu}^1_t(a)$}
            \EndIf
            \State {Observe $R_{k,h}$}
            \EndIf
            \If{abandonment occurs}
                \State {$S_{k,h+1} = g$}
            \Else
                \State {$S_{k,h+1} = R_{k,h}$}
            \EndIf
            \State{Define $(S_{t}, A_{t}, S'_{t}, R_{t}) \coloneqq  (S_{k,h}, A_{k,h}, S_{k,h+1}, R_{k,h})$}
            \State {Update: $N_{t+1}(A_t) = N_{t}(A_t) + 1$ and $N_{t+1}(a) = N_{t}(a)~\forall a\neq A_t$}
            \State {Update: $\bar{\mu}_{t+1}(A_t) = \frac{\bar{\mu}_{t}(A_t)N_{t}(A_t)+R_t}{N_{t+1}(A_t)}$ and $\bar{\mu}_{t+1}(a) = \bar{\mu}_{t}(a)~\forall a\neq A_t$}
            \State {$t\gets t+1$, $h\gets h+1$}
        \EndWhile
    \EndFor
 \end{algorithmic}
\end{algorithm}

\begin{theorem}\label{theorem:6}
Let all the assumptions in Theorem~\ref{theorem:5} hold. Then using Algorithm~\ref{alg:2-case2} with $c_0=c_1=1$ and $c=4$, we have
\begin{align}
    \limsup_{K\rightarrow \infty} \frac{\expt[\mathrm{Reg}_{\pi}(K)]}{\log K} \le \sum_{i\neq 1} \frac{V^*(0) - Q^*(0,a_i)}{\mathrm{kl}(\mu(a_i),\mu(a_1))}.
\end{align}
\end{theorem}
The proof is symmetric to that of Theorem~\ref{theorem:2} and hence is omitted.

\subsection{Lower bound}
\label{app:theorem3-case2}

\begin{theorem}\label{theorem:7}
Let all the assumptions in Theorem~\ref{theorem:5} hold. Let $\pi$ be a consistent policy, i.e., $\pi\in\Pi_{\mathrm{cons}}$. Then for any $\mu(a_1),...,\mu(a_M), q(0,0),q(0,1),q(1,0),q(1,1)$ satisfying the assumptions, we have
\begin{align}
    \liminf_{K\rightarrow \infty} \frac{\expt[\mathrm{Reg}_{\pi}(K)]}{\log K} \ge \sum_{i\neq 1} \frac{V^*(0) - Q^*(0,a_i)}{\mathrm{kl}(\mu(a_i),\mu(a_1))}.
\end{align}
\end{theorem}
The proof is symmetric to that of Theorem~\ref{theorem:3} and hence is omitted. From Theorem~\ref{theorem:6} and Theorem~\ref{theorem:7}, the upper bound matches the lower bound asymptotically for the case where $V^*(1)-Q^*(1,a)\ge V^*(0)-Q^*(0,a)$ for any $a\neq a_1$.

\vfill

\end{document}